%% file: paper.tex
\documentclass[twoside]{article}

\usepackage{float}
\usepackage[accepted]{aistats2026}
\usepackage[round]{natbib}

\usepackage{hyperref}
\input{macros/macros}
\input{macros/bib_macros}

\usepackage{stfloats} 
\begin{document}

%

%

\twocolumn[

\aistatstitle{Active Learning for Stochastic Contextual Linear Bandits}

\aistatsauthor{ Emma Brunskill \And Ishani Karmarkar \And  Zhaoqi Li }

\aistatsaddress{ Stanford University \And  Stanford University \And Stanford University } ]

\begin{abstract}
    A key goal in stochastic contextual linear bandits is to efficiently learn a near-optimal policy. Prior algorithms for this problem learn a policy by strategically sampling actions but naively (passively) sampling contexts from the underlying context distribution. However, in many practical scenarios---including online content recommendation, survey research, and clinical trials---practitioners can actively sample or recruit contexts based on prior knowledge of the context distribution. Despite this potential for \emph{active learning}, the role of strategic context sampling in stochastic contextual linear bandits is underexplored. We propose an algorithm that learns a near-optimal policy by strategically sampling rewards of context-action pairs. We prove \emph{instance-dependent} theoretical guarantees demonstrating that our active context sampling strategy can improve over the minimax rate by up to a factor of $\smash{\sqrt{d}}$, where $d$ is the linear dimension. We show empirically that our algorithm reduces the number of samples needed to learn a near-optimal policy, in tasks such as warfarin dose prediction and joke recommendation.
\end{abstract}
\vspace{-.75em}
\input{main_body/introduction}
\input{main_body/main_proof_sketch}

\input{main_body/comparison}

\input{main_body/experiments}

\input{main_body/discussion}

\subsubsection*{Acknowledgements}
Ishani Karmarkar was funded in part by NSF Grant CCF-1955039 and a PayPal research award. 
Emma Brunskill and Zhaoqi Li appreciate the support of a Stanford HAI research grant award. 
We thank Konwoo Kim for helpful discussions in early stages of this project and Yingxi Li for helpful conversations. 






\bibliographystyle{plainnat}
\bibliography{ref}


\clearpage
\appendix
\thispagestyle{empty}

\onecolumn
\aistatstitle{Active Learning for Stochastic Contextual Linear Bandits}














\input{appendix/organization}
\input{appendix/assumptions}
\input{appendix/compare_settings}
\input{appendix/additional_theory}

\input{appendix/additional_experiments}

\input{appendix/omitted_proofs}

\end{document}

%% file: macros/macros.tex
\usepackage{amsmath,amssymb,amsthm}
\usepackage{graphicx}
\usepackage{color}
\usepackage{bbm}
\usepackage{booktabs}
\usepackage{breakcites}
\usepackage{bm}
\usepackage{thm-restate}
\usepackage{color-edits}
\usepackage{float}
\usepackage{subcaption}
\usepackage{enumitem}
\usepackage{dsfont}
\usepackage{xspace}
\usepackage{placeins}
\usepackage[algo2e,linesnumbered,ruled,vlined]{algorithm2e}
\usepackage{algorithmic}
\usepackage{subcaption}
\usepackage{array}
\usepackage[dvipsnames]{xcolor}
\usepackage{cleveref}
\usepackage{makecell}

\theoremstyle{plain}
\newtheorem{theorem}{Theorem}[section]
\newtheorem{lemma}[theorem]{Lemma}

\newtheorem{corollary}[theorem]{Corollary}

\theoremstyle{definition}
\newtheorem{definition}[theorem]{Definition}

\theoremstyle{remark}
\newtheorem{remark}[theorem]{Remark}

\newcommand{\paren}[1]{\left(#1 \right )}

\newcommand{\Brac}[1]{\left[#1\right]}

\newcommand{\Set}[1]{\left\{#1\right\}}

\newcommand{\abs}[1]{\left\lvert#1\right\rvert}

\newcommand{\norm}[1]{\left\lVert#1\right\rVert}

\newcommand{\defeq}{:=}

\newcommand{\normInline}[1]{\lVert#1\rVert}
\newcommand{\iid}{$\text{i.i.d.}$ }

\newcommand{\Z}{{\mathbb Z}}

\newcommand{\R}{\mathbb R}

\newcommand{\Esymb}{\mathbb{E}}
\newcommand{\Psymb}{\mathbb{P}}

\DeclareMathOperator*{\E}{\Esymb}

\DeclareMathOperator*{\ProbOp}{\Psymb}

\newcommand{\prob}[1]{\ProbOp\Set{#1}}

\newcommand{\cA}{\mathcal A}
\newcommand{\cB}{\mathcal B}
\newcommand{\cC}{\mathcal C}

\newcommand{\cF}{\mathcal F}

\newcommand{\cN}{\mathcal N}

\newcommand{\cS}{\mathcal S}

\newcommand{\cX}{\mathcal X}

\DeclareMathOperator*{\argmin}{argmin} 
\DeclareMathOperator*{\argmax}{argmax}

\allowdisplaybreaks

\SetKwInput{KwInput}{Input}
\SetKwInput{KwOutput}{Output}
\SetKwInput{Return}{return}
\SetKwInput{KwDS}{Data Structure}
\usepackage{algorithmic}

\usepackage{booktabs} 
\usepackage{array} 

\newcommand{\thetahat}{\hat{\theta}}
\newcommand{\thetastar}{{{\theta}^\star}}

\newcommand{\tr}{{\mathrm{tr}}}

\newcommand{\calS}{\mathcal{S}}

\usepackage{enumitem}

\newcommand{\betamind}{\beta}

\newcommand{\SigmaOf}[1]{\Sigma_{#1}}
\newcommand{\SigmOHatM}[1]{\hat{\Sigma}_{#1, T}}
\newcommand{\SigmaOS}[1]{{\Sigma}_{#1}}

\newcommand{\SigmaOInv}[1]{\Sigma^{-1}_{#1}}
\newcommand{\SigmaOHatMInv}[1]{\hat{\Sigma}^{-1}_{#1, T}}
\newcommand{\SigmaOSInv}[1]{{\Sigma}^{-1}_{#1}}

\newcommand{\SigmaLambdaS}[2]{{\Sigma}_{#1}}

\newcommand{\SigmaLambdaSInv}[2]{{\Sigma}^{-1}_{#1}}

\SetCommentSty{mycommfont}
\SetKwFor{Repeat}{repeat}{}{}


%% file: macros/bib_macros.tex
  \usepackage{nth}
  \usepackage{intcalc}

  \newcommand{\cAAAI}[1]{AAAI\ Conference\ on\ Artificial (AAAI)}

%% file: main_body/introduction.tex
\section{INTRODUCTION}\label{sec:intro}

In many applications, algorithm designers seek to leverage reward feedback to develop contextualized decision policies. For example, in healthcare, clinical trial outcomes may be used to design medication dosages adapted to patients’ demographics and health conditions~\citep{varatharajah2022contextual, wang2025optimal}. Similarly, in recommendation systems, user interaction data can inform content recommendations tailored to specific users' preferences~\citep{bouneffouf2012contextual, tang2014ensemble}. Personalized decision-making might also aid in AI alignment to enhance language models for use in domains like education, law, or medicine~\citep{liu2024sample, chen2024online}.

\emph{Contextual bandits} offer a natural framework to formalize such decision-making problems. In a contextual bandit, we have a finite context set $\mathcal{X}$ and action set $\mathcal{A}$. Contexts are drawn from a distribution $p \in \Delta^{\mathcal{X}}$. (As a shorthand, for any set $S$ of size $k$, $\Delta^S \defeq \Delta^k$, where $\Delta^k$ $k$-dimensional simplex.) Each context-action pair yields a (possibly stochastic) reward $r(x,a)$. In applications, contexts correspond, for example, to patients or users, while rewards might reflect medical outcomes, user engagement, or preference alignment.

In the exploration setting, or experiment design setting \citep{zanette2021design, krishnamurthy2023proportional, deshmukh2018simple, li2022instance}, the goal is to design an \emph{exploration algorithm} which observes sampled rewards $\{r(x_t,a_t)\}_{t=1}^T$ of $T$ context-action pairs $\{(x_t, a_t)\}_{t=1}^T$ and uses these to learn a policy $\hat{\pi}$. We measure the quality of $\hat{\pi}$ by its \emph{(simple) regret},
\begin{align}\label{eq:simple-regret-general}
R(\hat{\pi}) = 
{
\max_{\pi \in \Pi}  \mathbb{E}[r(x, \pi(x))] 
}
- {{\E[r(x, \hat{\pi}(x))]}},
\end{align}
where $\Pi \defeq \{\pi: \mathcal{X} \to \mathcal{A}\}$ and expectation is taken with respect to $x \sim p$ and the stochasticity in $r$. To design efficient exploration algorithms with theoretical regret bounds, prior works often consider the setting where the rewards are generated by a noisy $d$-dimensional linear model. This is called the \emph{stochastic contextual linear bandits} (SCLBs) setting \citep{zanette2021design, krishnamurthy2023proportional, abbasi2011improved,deshmukh2018simple, ruan2021linear, lattimore2020bandit}. These works learn near-optimal policies for SCLBs by sampling rewards of context-action pairs, where contexts are sampled \iid from $p$ (what we call \emph{passive context sampling}) and actions are sampled strategically.

\begin{figure*}[t]
  \centering
  \begin{subfigure}[t]{0.9\textwidth}
    \centering
    \includegraphics[width=.9\textwidth]{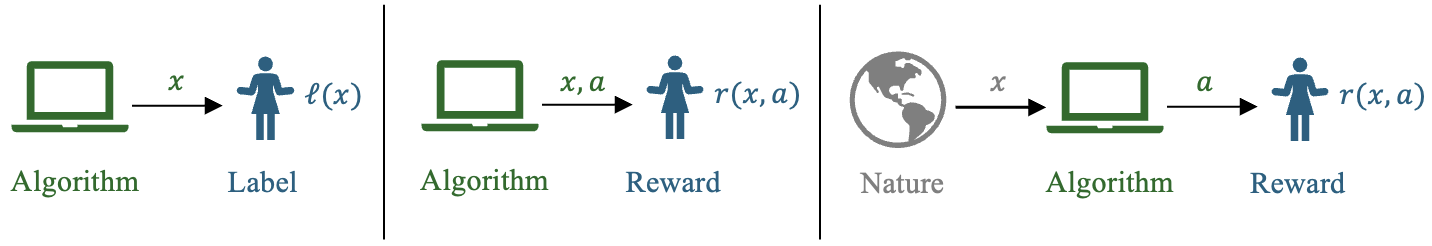}
    \caption{Active context sampling for contextual bandits (center) lies between classic active learning (left), e.g., for regression, and passive context sampling for contextual bandits (right).\\  \vspace{1em}}
    \label{fig:active_bandits_diagram}
  \end{subfigure}
  \begin{subfigure}[b]{0.9\textwidth}
    \centering
    \includegraphics[width=.9\textwidth]{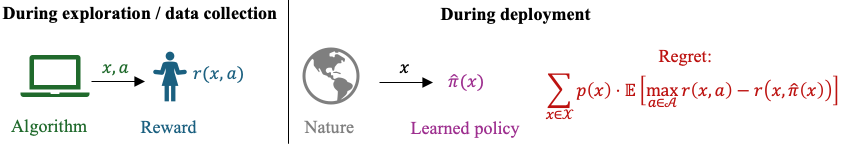}
    \caption{ During exploration, active context sampling allows selecting \emph{both} contexts and actions. However, the goal is to learn a policy that performs well in \emph{deployment}, when contexts will be sampled from $p$ (recall \eqref{eq:simple-regret-general}). }
    \label{fig:active_bandits_train_deploy}
  \end{subfigure}

  \caption{Active context sampling for contextual bandits.}
  \label{fig:active_bandits}
\end{figure*}

However, in several real-world applications, practitioners may already know the context distribution $p$ and can leverage the fact that $p$ is known a-priori to \emph{strategically} select \emph{both} observed contexts \emph{and} observed actions for which to sample rewards. This approach (which we call \emph{active context sampling}) is naturally motivated in the following examples, for instance:
\begin{itemize}[leftmargin=*, nosep]
\item \textbf{Clinical trial design}: Clinicians typically possess prior information about patient demographics and medical conditions. They can deliberately recruit patients (contexts) with specific demographic or medical attributes, such as rare diseases or particular genetic markers, and carefully select experimental treatments (actions) to administer. While traditional trials often aim to represent general populations based on surface-level features (e.g., age, gender), incorporating genomic information or detailed patient-treatment embeddings may provide new opportunities to structure trials more efficiently. This might reduce experimental burden and enhance statistical power \citep{fowler2023clinical, deng2011active, minsker2016active}.
\item \textbf{Consumer marketing}: Marketing professionals commonly have prior knowledge about the demographic profiles of their audiences. With this information, they can intentionally target consumers (contexts) from underrepresented segments—for example, individuals who rarely respond to traditional advertising or those in niche market groups—and craft tailored interventions (actions) for them. By strategically selecting informative consumer-advertisement pairs, marketers might decrease the number of observations required to find effective marketing strategies \citep{labrecque2021addressing}.
\end{itemize}

In such settings, one might hope that by \emph{jointly} optimizing the observed context and action pairs $\{(x_t, a_t)\}_{t=1}^T$ for which the rewards $\{r(x_t, a_t)\}_{t =1}^T$ are observed, an exploration algorithm may be able to better leverage existence of context-action pairs that disproportionately reveal useful information about the reward model. Hence, in this work, we explore the following central question:

\begin{center}
\emph{Can active context sampling reduce the sample complexity needed to learn a high performing policy for stochastic contextual linear bandits?}
\end{center}

Now, one natural concern might be: does allowing the exploration algorithm to choose both contexts and actions reduce the problem to a linear bandit or classic active learning? This is \emph{not} the case, because, {while the algorithm can select contexts and actions during \emph{exploration}, at \emph{deployment}, the environment produces contexts according to the distribution $p$}. So, we require a \emph{context-dependent} policy to achieve low regret, and the problem does not directly reduce to a linear bandit.\footnote{One \emph{could} treat each context as a separate linear bandit; but, this would lead to sample complexities with linear dependence on the \emph{number of contexts}, which in general can be huge relative to the feature dimension. In the SCLB literature, one seeks to \emph{avoid} this linear dependence on the number of contexts by leveraging linear reward structure.} Our techniques do build upon prior literature on active learning for regression, but our objective is different. In active learning for regression, the loss is measured by mean squared error, so continuous convex optimization results \citep{frostig2015competing} apply. In SCLBs, the loss measures the \emph{suboptimality} of the policy, which is a discontinuous loss and requires new insights (see also Appendix~\ref{sec:comparing_settings}.) 

\paragraph{Our contributions.} Motivated by practical applications, we design an exploration algorithm which actively samples contexts \emph{and} actions to learn a near-optimal policy for an SCLB (See Figure~\ref{fig:active_bandits}.)
\begin{itemize}[leftmargin=*, nosep]
    \item We provide a \emph{polynomial-time} active context sampling exploration algorithm for SCLBs and prove an \emph{instance-dependent} regret bound for our algorithm. We prove that our instance-dependent regret bound matches the minimax-optimal rate in the worst-case. 
    \item To demonstrate the power of active context sampling we construct a class of SCLBs where our instance-dependent regret bounds provably improve by a $\smash{\sqrt{d}}$-factor over the state-of-the-art rates obtained by prior polynomial-time algorithms. 
    \item We support our analysis with numerical experiments on warfarin dose prescription and joke recommendations. These show our method requires fewer samples to learn a good policy, compared to baselines.
\end{itemize}

\section{RELATED WORK}\label{sec:setup}

An SCLB is a tuple $\cB = (\cX, \cA, \phi, p, \nu, \thetastar)$. Here, $\cX$ is the context set, $\cA$ is the action set, $\smash{\phi: \cX \times \cA \to \R^d_{\neq 0}}$ is a known $d$-dimensional feature mapping, and $\smash{p \in \Delta^\cX}$ is a known context distribution. We also have a 1-subgaussian noise distribution $\nu$ over $\R$ and \emph{unknown} linear parameter $\thetastar \in \R^d$. Rewards are modeled as a noisy linear function in the feature mapping, i.e. $r(x, a) \sim \phi(x,a)^\top \thetastar + \eta$ where $\eta \sim \nu$. We assume that the problem is normalized so that $\max_{(x, a) \in \cX \times \cA} \normInline{\phi(x, a)} \leq L$ and $\max_{(x, a) \in\cX \times \cA} \E[r(x,a)] \in [0,1]$, where $\normInline{\cdot}$ is the $\ell_2$ norm. To ensure well-posedness, we assume the feature mapping is full-rank, i.e., $\text{span}(\{\phi(x,a)\}) = \R^d$. For SCLBs, note that the regret \eqref{eq:simple-regret-general} simply reduces to
\begin{align}\label{eq:simple-regret-sclb}
    R(\hat{\pi}) = \E_{x \sim p} [\max_{a \in \cA} \phi(x,a)^\top \thetastar - \phi(x,\hat{\pi}(x))^\top \thetastar].
\end{align}
Exploration algorithms for SCLBs are well-studied \citep{zanette2021design, krishnamurthy2023proportional, deshmukh2018simple, ruan2021linear, li2022instance}. All of these prior works design algorithms which use \emph{passive context sampling}. In particular, \citet{zanette2021design} introduced two polynomial-time algorithms, reward-free LinUCB (RFLinUCB) (adapted from the LinUCB algorithm of \citet{abbasi2011improved}) and the Planner-Sampler algorithm. Both algorithms use $T$ reward observations $\{r(x_t, a_t)\}_{t=1}^T$ to learn a policy $\hat{\pi}$ such that ${R(\hat{\pi}) \leq \tilde{O}(\sqrt{d\beta/T})}$ with high probability.\footnote{$\tilde{O}(\cdot)$ hides polylogarithmic factors in $d, 1/\delta$ and $L$.} Here, $\betamind$ is a parameter, which under mild assumptions satisfies $\smash{\sqrt{\beta} = \tilde{O}(\sqrt{d})}$ (defined formally in Lemma~\ref{lemma:simple-regret-bound-general}, \eqref{eq:beta-lambda}). This rate is minimax-optimal (up to polylogarithmic factors) \citep{zanette2021design, chu2011contextual}.\footnote{This rate remains minimax optimal, even in the setting where contexts are sampled actively, as we discuss further in Appendix~\ref{sec:lower-bound-discussion}.}

Inspired by \emph{practical applications} where active context sampling may be beneficial, our work focuses on leveraging the power of active context sampling to develop a \emph{polynomial-time} algorithm with tighter instance-dependent-regret bounds. \citet{li2022instance} also studied instance-dependent regret bounds for SCLBs; however, their algorithm again uses passive context sampling and requires exponential time. Thus, it does not lend itself well to practical applications. We discuss \citet{li2022instance} further in Section~\ref{sec:comparison} and Appendix~\ref{apx:additional_results}. An orthogonal line of work studies algorithms minimizing the \emph{cumulative regret}, for SCLBs; however this is not directly comparable to our exploration setting (see \citep{lattimore2020bandit, abbasi2011improved} and Appendix~\ref{sec:regret-setting} for further discussion of the cumulative regret versus simple regret setting.)

SCLBs are a type of linear Markov Decision Process (MDP) with effective horizon $1$, where contexts are the MDP states. \citet{wagenmaker2022instance, wagenmaker2022beyond} obtain instance-dependent rates for MDPs, but neither considers actively sampling states/contexts, and their algorithms are computationally hard to implement. \citet{wang2021sample} and \citet{gheshlaghi2013minimax} consider MDPs in the generative model setting, where states are sampled actively \emph{but} rewards are known a-priori. The SCLB problem is \emph{trivial} if rewards are known a-priori. 

Our work is related to experiment design and active learning \citep{pukelsheim2006optimal, settles2009active, rodemann2024reciprocal, de2023active}. Most related is work on active learning for regression \citep{chaudhuri2015convergence} and G-optimal experiment design for linear bandits \citep{pukelsheim2006optimal, lattimore2020bandit}. \citet{chaudhuri2015convergence} design active learning methods for maximum likelihood estimator (MLE) problems such as ridge regression, but their approach and analysis \emph{crucially} relies on the differentiability of the MLE objective and classic statistical bounds for MLEs. SCLBs do not satisfy these assumptions: learning a near-optimal policy for an SCLB is not an MLE problem, and the objective \eqref{eq:simple-regret-sclb} is not even differentiable due to the max over $a \in \cA$. To overcome this, we design an algorithm which \emph{combines} techniques from ridge regression and $G$-optimal design with a careful matrix-concentration analysis to prove our theoretical guarantees. 

We are unaware of prior work on active context sampling for SCLBs. There has been work on active learning for other contextual decision making models. \citet{char2019offline, li2022near, kirschner2020distributionally} studied contextual Bayesian optimization, where the reward is modeled as a Gaussian process (instead of a linear reward). Their guarantees are weaker, requiring either super-linear dependence on the context-action space \citep{char2019offline}, or explicit dependence on the dimension $d$ \citep{li2022near}. Lastly, \citet{das2024active} studied active context sampling in a preference framework with a Bradley-Terry-Luce reward model.

\citet{minsker2016active} studied active context sampling in the context of clinical trials. However, they make several  assumptions on the setting which are specifically motivated by filtering in ongoing clinical trials. In particular, they assume binary actions and that the policy given a patient (context) samples uniformly at random over the two actions. Hence, their work does not yield improved regret bounds in our setting. 

%% file: main_body/main_proof_sketch.tex
\section{APPROACH OVERVIEW}\label{sec:approach}

Our goal is to design a way to sample \emph{rewards} of a set of context-action pairs $\{(x_t,a_t)\}_{t=1}^T$ so that the resulting dataset of rewards $\{r(x_t,a_t)\}_{t=1}^T$ allows us to learn a near-optimal contextual policy $\hat{\pi}$. This $\hat{\pi}$ would be used for future deployment, when contexts are sampled from a known distribution $p$ (recall Figure~\ref{fig:active_bandits}, see Figure~\ref{fig:learning-paradigms}.) In this section, we fix $\lambda >0$ to be a regularization parameter, $T > 0$ to be a sample budget, and $\cB = (\cX, \cA, \phi, p, \nu, \thetastar)$ to be any SCLB. In the main body, we assume the context distribution $p$ is known, in order to inform active context sampling. Appendix~\ref{subsec:approxpsection}, relaxes this requirement, allowing $p$ to be replaced by an \emph{approximate} context distribution, $\hat{p}$.

\paragraph{Notation.} We use $[n]$ to denote the set $\{1, ..., n\}$ and $\normInline{\cdot}$ to denote the Euclidean norm (if applied to vectors) or the spectral norm (if applied to matrices). We use $[v]_i$ for the $i$-the entry of $v \in \R^d$. We use $\succeq$, $\succ$ for the Loewner ordering: $A \in \R^{d \times d}$ satisifes $A \succeq 0$ ($A \succ 0$) if and only if $A$ is positive semi-definite (positive definite), respectively (see Lemma~\ref{lemma:lowener}.) For any $A \succ 0$ and $x \in \R^d$, we denote $ {\normInline{x}_{A^{-1}} \defeq (x^\top A^{-1} x)^{1/2}}$. For any set $\cS \subset \cX \times \cA$, we define the covariance matrix $ {\SigmaOS{\cS} \defeq \lambda I + \sum_{(x, a) \in \cS} \phi(x,a) \phi(x,a)^\top}.$ For any $w \in \Delta^{\cX \times \cA}$, let $\SigmaOf{w} \defeq \lambda/T \cdot I + \E_{(x,a) \sim w} \phi(x,a) \phi(x,a)^\top$ denote the $w$-weighted covariance matrix of the feature mapping $\phi$. We also denote the $T$-sample \emph{empirical} covariance matrix as follows
$\SigmOHatM{w} \defeq {\lambda}/{T} \cdot I + {1}/{T}  \sum_{t \in [T]} \phi(x_t, a_t)\phi(x_t,a_t)^\top 
$ where $(x_t,a_t) \sim w$ \iid

We first restate a standard result from the SCLB literature: the next lemma bounds the regret of a learned policy obtained by fitting a ridge regression model to dataset $\cS$ \citep{lattimore2020bandit} and is the basis for prior works which obtain polynomial-time algorithms with minimax-optimal regret for SCLBs \citep{zanette2021design,abbasi2011improved}. 

\begin{restatable}[Ridge regression regret bound]{lemma}{ridgebound}\label{lemma:simple-regret-bound-general} Let $\lambda > 0$ be a regularization parameter, $\mathcal{S} = \{(x_t, a_t)\}_{t \in [T]} \subset \cX \times \cA$, and $\delta \in (0, 1)$ be a failure probability. For each $t \in [T]$, let $ {r_t \sim \phi(x_t,a_t)^\top \theta^\star + \eta_t}$ where each $\eta_t {\sim} \nu$, independently. Let $ {\hat{\theta}}$ solve the \emph{ridge regression} problem, i.e., $\hat{\theta} \defeq \SigmaLambdaS{\cS}{\lambda}^{-1} \sum_{t \in [T]} \phi(x_t, a_t) r_{t}$ and 
$\hat{\pi}\defeq x \mapsto \argmax_{a \in \cA} \phi(x,a)^\top \hat{\theta}$. Moreover, define 
\begin{equation}\label{eq:beta-lambda}
\begin{split}
\sqrt{\betamind}
&\defeq 2 \min\Big(2\sqrt{2}\,\sqrt{\log\!(12\pi^{-2}\delta^{-1}{T^2\,|\cX|\,|\cA|})},
      \\
&\;\; \sqrt{d\,\log\!({2\delta^{-1}(1+TL^2/\lambda)})}
      + \sqrt{\lambda}\,\|\thetastar\|
\Big)
\end{split}
\end{equation}
and the uncertainty measure 
\begin{align}\label{eq:regret_bound}
    \Gamma(\cS) \defeq {\E_{x \sim p} \max_{a \in \cA} \normInline{\phi(x,a)}_{\SigmaOSInv{\cS}}^2}. 
\end{align}
Then, with probability $1-\delta/2$, $R(\hat{\pi}) \leq \sqrt{\betamind  \Gamma(\cS)}.$
\end{restatable}

The uncertainty measure \eqref{eq:regret_bound} quantifies how the design of the dataset $\calS$ influences the regret bound when using ridge regression to learn a policy $\hat{\pi}$. Prior works \citep{zanette2021design, abbasi2011improved} design algorithms (RFLinUCB and Planner-Sampler) which construct a set $\calS$ by \emph{passively} sampling contexts $x_t \sim p$ and strategically sampling actions $a_t \sim \pi_t(x_t)$ (where each $\pi_t \in \Delta^{\cA}$) such that with high probability, 
\begin{align}\label{eq:old-bound}
    \Gamma(\cS) = {\E_{x \sim p} \max_{a \in \cA} \normInline{\phi(x,a)}_{\SigmaOSInv{\cS}}^2} \leq \tilde{O}(d/T). 
\end{align}
\cite{zanette2021design, abbasi2011improved} then essentially instantiate Lemma~\ref{lemma:simple-regret-bound-general} to obtain an overall regret bound of $\Tilde{O}(\sqrt{\betamind d/T})$, which is minimax-optimal \citep{zanette2021design, abbasi2011improved, chu2011contextual}. Interestingly, our next result formalizes a sense in which the uncertainty bound in \eqref{eq:old-bound} is \emph{tight} if one restricts to constructing the dataset $\cS$ by passively sampling $x_t \sim p$ (passive context sampling) \emph{regardless} of how the actions $a_t$ are chosen. 
\begin{restatable}{theorem}{barrier}\label{theorem:barrier} For each $t \in [T]$, let $\pi_t: \cX \to \Delta^{\cA}$ be arbitrary. Let $\cS = \{(x_1, a_1), ..., (x_T, a_T)\} \subset \cX \times \cA$ such that for each $t \in [T]$, $x_t \sim p$ and $a_t \sim \pi_t(x_t)$. Then, as $\lambda \to 0$, $\smash{\E_{\cS}[ \Gamma(\cS) ] \geq d/T}$. 
\end{restatable}

Despite this obstacle, hope remains that active context sampling---strategically sampling the contexts $x_t$ when constructing $\cS$---might allow us to leverage the structure of a \emph{given} SCLB $\cB$ to obtain a more fine-grained, \emph{instance-dependent} analysis than the standard minimax-optimal $\smash{\tilde{O}(\sqrt{\beta d/T})}$ bound. In particular, we aim to design an algorithm that leverages the existence of any disproportionately informative context-action pairs $(x,a) \in \cX \times \cA$ to obtain instance-dependent regret bounds as low as $\smash{\Tilde{O}(\sqrt{\beta/T}}$) in the best-case and at \emph{most} $\smash{\Tilde{O}(\sqrt{\beta d/T}}$) in the worst-case. 

Given Lemma~\ref{lemma:simple-regret-bound-general}, one natural approach towards this goal is to construct $\cS$ so that \eqref{eq:regret_bound} is as small as possible. In particular, we might hope to find the optimal $\cS^\star = \argmin_{\cS \subset \cX\times\cA , |\cS| = T} \Gamma(\cS)$. Sadly, this is NP-hard \citep{civril2009selecting, welch1982algorithmic}. Nonetheless, we can consider a natural \emph{fractional relaxation}, wherein we seek an optimal sampling \emph{distribution} $w^\star$ (with objective value $\cC_\cB$) defined as follows:
\begin{equation}\label{eq:optimal-sampling-distribution}
\begin{split}
w^\star &\defeq 
  \argmin_{w \in \Delta^{\cX \times \cA}}
  \E_{x \sim p} \max_{a \in \cA} 
  \normInline{\phi(x,a)}^2_{\SigmaOInv{w}}, \\
\cC_{\cB} &\defeq 
  \min_{w \in \Delta^{\cX \times \cA}}
  \E_{x \sim p} \max_{a \in \cA} 
  \normInline{\phi(x,a)}^2_{\SigmaOInv{w}}.
\end{split}
\end{equation}

Theorem~\ref{theorem:SDP-formulation-hat} shows that \eqref{eq:optimal-sampling-distribution} can be expressed as a semi-definite program (SDP), enabling a polynomial-time solution! Intuitively, if we could ``fractionally'' sample context-action pairs, then $T \cdot w^\star(x,a)$ is the optimal ``fraction'' of context $(x,a)$ that should be included in $\cS$. A fractional sampling procedure is not implementable, but we can \emph{simulate} it: for each $t \in [T]$, let $(x_t, a_t) \sim w^\star$ \iid For large $T$, we expect:
\begin{equation}\label{eq:our-bound}
\begin{split}
\Gamma(\cS) 
  &= \E_{x \sim p} \max_{a \in \cA} 
     \normInline{\phi(x,a)}_{\SigmaOSInv{\cS}}^2 \\
  &\approx \tfrac{1}{T} \E_{x \sim p} \max_{a \in \cA} 
     \normInline{\phi(x,a)}^2_{\SigmaOInv{w^\star}} 
     = {\cC_{\cB}}/T.
\end{split}
\end{equation}

If \eqref{eq:our-bound} holds, Lemma~\ref{lemma:simple-regret-bound-general} implies a regret of $\tilde{O}(\sqrt{\cC_{\cB} \beta/T})$. The next section formally converts the intuition laid out here into a polynomial-time algorithm (Algorithm~\ref{alg:active-lcb}) achieving a regret of $\smash{\tilde{O}(\sqrt{\cC_{\cB} \beta/T})}$. Further, we show the instance-dependent quantity $\smash{\cC_{\cB}}$ can be as low as $O(1)$ (indicating that our result can improve the regret bounds of Planner-Sampler and RFLinUCB by up to a $\smash{\sqrt{d}}$ factor) and that $\cC_{\cB}$ is \emph{at most} $d$ (ensuring our algorithm is minimax-optimal.) 

\section{MAIN RESULT}\label{sec:theoretical-analysis}

Algorithm~\ref{alg:active-lcb} essentially implements the approach introduced in Section~\ref{sec:approach} (up to some technical details.) Line~\ref{line:smooth} computes a distribution $q$ that \emph{approximates} the \emph{G-optimal distribution} $q^\star$, defined as follows: 
\begin{theorem}[Theorem 21.1 of \citep{lattimore2020bandit}, restated]\label{thm:kiefer-wolofitz} Suppose that $\smash{q^\star \defeq \argmin_{q\in \Delta^{\cX \times \cA}}\max_{(x, a) \in \cX \times \cA} \normInline{\phi(x,a)}_{\SigmaOInv{q}}^2.}$ Then $\smash{\max_{(x, a) \in \cX \times \cA} \normInline{\phi(x,a)}_{\SigmaOInv{q^\star}}^2 = d}$. 
\end{theorem}
There are various polynomial-time algorithms to find $q$ which is a 2-multiplicative approximation to $q^\star$ in the sense of Line~\ref{line:smooth} (e.g., Franke-Wolfe or an SDP solver; see Chapter 21 of \cite{lattimore2020bandit}). 

Second, in Line~\ref{line:approximate}, for given constant $\alpha \in [0, 1]$, the algorithm computes a distribution $w$, which approximates an ``$\alpha$-smoothed'' version of $w^\star$ (denoted $w^*$):
\begin{equation}\label{eq:smooth-opt}
\begin{split}
    &w^* = \argmin_{w \in \Delta^{\cX \times \cA}} 
           \E_{x \sim p} \max_{a \in \cA} 
           \normInline{\phi(x,a)}_{\Sigma_w^{-1}}^2, \\
        &\text{st. } 
           w(x,a) \geq \alpha \cdot {q}(x,a), ~\forall (x,a) \in \cX \times \cA.
\end{split}
\end{equation}
The optimization problems for $w^*$ and $w^\star$ \eqref{eq:optimal-sampling-distribution} differ only in the constraint that $w^*$ dominates $\alpha {q}$. This constraint is for a technical reason---it ensures $w^*$ is \emph{well-conditioned} with respect to the feature mapping so that finite-sample matrix-concentration arguments hold. Fortunately, as the next lemma shows, the objective values attained by $w, w^*,$ and $w^\star$ are all close; so we can replace $w^\star$ with $w$ in the approach  of Section~\ref{sec:approach} at the cost only constants in sample complexity. 

\begin{restatable}{lemma}{noloss}\label{lemma:no-loss} Let $w^\star$, $w^*$, and $w$ be as in \eqref{eq:optimal-sampling-distribution}, \eqref{eq:smooth-opt}, and Line~\ref{line:approximate}. Then, $\E_{x \sim p} \max_{a \in \cA} \normInline{\phi(x,a)}_{\Sigma_{w}^{-1}}^2  \leq 2 \E_{x \sim p} \max_{a \in \cA} \normInline{\phi(x,a)}_{\Sigma_{w^*}^{-1}}^2 \leq 2/(1-\alpha) \cdot  \cC_\cB$.
\end{restatable}

Theorem~\ref{theorem:SDP-formulation-hat} proves that \eqref{eq:smooth-opt} can be expressed as a semi-definite program (SDP). The proof leverages properties of Schur complements (Lemma~\ref{lemma:schur-complement}.) SDPs can be solved to high accuracy in polynomial time (e.g., \cite{jiang2020faster}). Thus, one can compute a $w$ satisfying Line~\ref{line:approximate} in polynomial time. There are also many practical, open-source solvers (e.g., \cite{goulart2024clarabel}). We used MOSEK (see Appendix~\ref{apx:additional-experiments}.)

Finally, the algorithm draws $T$ \iid samples from $w$ to construct $\cS$, performs ridge regression, and outputs a policy using the procedure of Lemma~\ref{lemma:simple-regret-bound-general}. 

We prove Algorithm~\ref{alg:active-lcb} has the following guarantee.

\begin{restatable}{theorem}{mainresult} \label{theorem:main} Let $\cB = (\cX, \cA, \phi, p, \nu, \thetastar)$ be an SCLB, $\lambda > 0$, $\delta \in (0, 1)$, and $T > 0$ be a sample budget. Invoke Algorithm~\ref{alg:active-lcb} with $\alpha \gets 1/2$. Let $\beta$ be as defined in \eqref{eq:beta-lambda}. There exists $T_0 = \tilde{O}(d^2)$ so that whenever $T \geq T_0$, with probability $1-\delta$, Algorithm~\ref{alg:active-lcb} outputs $\hat{\pi}$ with $R(\hat{\pi}) \leq \tilde{O}(\sqrt{{\beta \cC_{\cB}}/{T}})$. This regret bound is at most $\tilde{O}(\sqrt{\beta d/T})$. The algorithm runs in polynomial time. 
\end{restatable}

\textit{Proof sketch.} Given Lemma~\ref{lemma:simple-regret-bound-general}, to prove Theorem~\ref{theorem:main}, it suffices to show that for $\cS$ as constructed in Line~\ref{line:cS}, the following holds with high probability:
\begin{align*}
    \Gamma(\cS) = \E_{x \sim p} \max_{a \in \cA} \normInline{\phi(x,a)}_{\SigmaOSInv{\cS}}^2 \leq 8\cC_{\cB}/T. 
\end{align*}
We prove this bound in three stages. First, by construction of $\cS$ in Line~\ref{line:cS}, $\smash{\SigmaOS{\cS} = T\SigmOHatM{w}}$. Inverting both sides, we have ${T \SigmaOSInv{\cS} = \SigmaOHatMInv{w}}.$ Second, we use matrix concentration guarantees \cite{tropp2012user} to show that for $T$ sufficiently large, with high probability,
\begin{align}\label{eq:invert-get-T}
T \SigmaOSInv{\cS} = \SigmaOHatMInv{w} \preceq 2 \cdot \SigmaOInv{w}. 
\end{align}

\input{main_body/alg}

This concentration argument uses that $q$ is constructed to be well-conditioned to the features, in that $\max_{(x,a) \in \cX \times \cA} \normInline{\phi(x,a)}_{\SigmaOInv{{q}}}^2 \leq 2d $ and that $\Sigma_w \succeq 1/2 \cdot \Sigma_{{q}}$ (since $w$ dominates $q/2$.) Thus, we have that for any $(x, a) \in \cX \times \cA$, 
\begin{align*}
    &T \normInline{\phi(x,a)}_{\SigmaOSInv{\cS}}^2 = T \cdot \phi(x,a)^\top (\SigmaOS{\cS})^{-1} \phi(x,a) \\
    &\leq 2 \cdot \phi(x,a)^\top\SigmaOInv{{w}}\phi(x,a) = 2 \cdot \normInline{\phi(x,a)}_{\SigmaOInv{{w}}}^2. 
\end{align*}
Third, applying expectation and max to the above display, and then invoking Lemma~\ref{lemma:no-loss}, we conclude 
\begin{align*}
    T {\E_{x \sim p} \max_{a \in \cA} \normInline{\phi(x,a)}_{\SigmaOSInv{\cS}}^2} \leq  2 {\E_{x \sim p} \max_{a \in \cA} \normInline{\phi(x,a)}_{\SigmaOInv{{w}}}^2} \leq 8 \cC_\cB. 
\end{align*}

\begin{figure*}[tb]
    \centering
    \includegraphics[width=.6\linewidth]{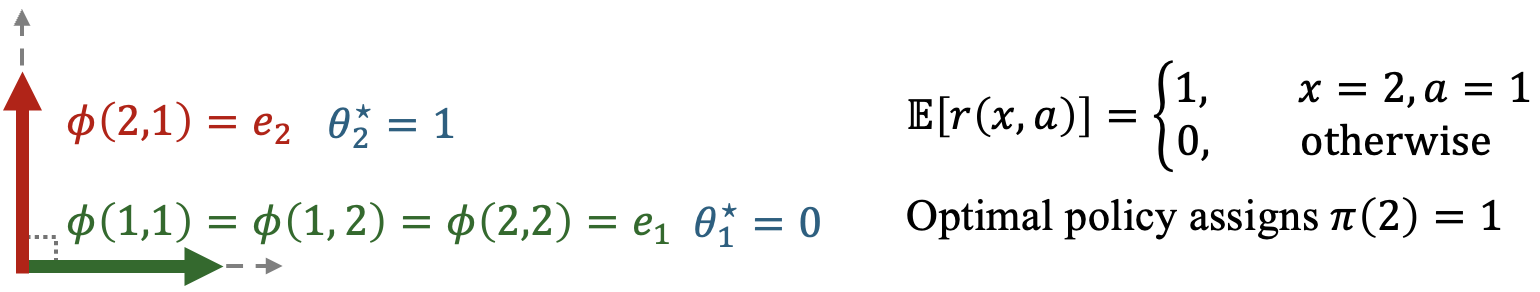}
    \caption{Visualization of $\cB_{2,2}^\star$. In context $x=1$ any action is optimal. Only $a=1$ is optimal in $x=2$. Since $p(1) \gg p(2)$, passive context sampling repeatedly sees $x=1$, and hence require many samples to learn the best policy. Active context sampling upsamples $x=2$ to learn the best policy more efficiently. }
    \label{fig:hard-instance-visualization}
    \vspace{-1em}
\end{figure*}

Thus, $\Gamma(\cS) \leq 8\cC_{\cB}/T$ as desired, and Lemma~\ref{lemma:simple-regret-bound-general} yields the regret bound ${R(\hat{\pi}) \leq \tilde{O}(\sqrt{\beta \cC_{\cB}/T})}$. Finally, to show that $\cC_{\cB} \leq d$, we let $q^\star$ be the G-optimal design as in Theorem~\ref{thm:kiefer-wolofitz}. Then, we see that $\smash{\cC_{\cB} \leq {\E_{x \sim p} \max_{a \in \cA} \normInline{\phi(x,a)}_{\SigmaOf{q^\star}^{-1}}^2}} \leq \smash{\max_{(x, a)\in \cX \times \cA} {\normInline{\phi(x,a)}_{\SigmaOf{q^\star}^{-1}}^2}} = d.~~~\square$

Next, we highlight some advantages of Algorithm~\ref{alg:active-lcb}. 

\begin{itemize}[leftmargin=*, nosep]

\item\textbf{Choice of $\alpha$:} The choice $\alpha = 1/2$ in Theorem~\ref{theorem:main} is purely to enable a finite-sample matrix concentration analysis in the proof. In theory, any constant $\alpha \in (0,1)$ would suffice for the analysis to go through (at the cost of constants in the sample complexity). In practice, we find even $\alpha =0$ works well and avoids computing $q$, so we set $\alpha \gets 0$ in experiments. 

\item \textbf{SDP approximation:} The theoretical analysis only requires a $2$-multiplicative approximation to $w^*$ in Line~\ref{line:approximate}, but in practice, SDPs can be solved to high accuracy. Tighter approximation would lead to tighter constants in our regret bound. Consequently, in experiments, we solve the SDP to convergence. 

\item \textbf{Batching:} For $t \in [T]$, Algorithm~\ref{alg:active-lcb} queries rewards for $\{(x_t, a_t)\}_{t \in [T]}$  drawn \iid from $w$ in Line~\ref{line:cS}. Importantly, the $t$-th pair $(x_t, a_t)$ is \emph{not} dependent on the history $\smash{\{(x_j, a_j)\}_{j \in [t-1]}}$. This ensures the reward observations can be parallelized. Hence, our algorithm falls under the batch-learning paradigm in active learning \cite{gentile2024fast} and non-adaptivity paradigm in SCLBs \cite{zanette2021design}. This is important in settings where each reward requires long observation horizons. For example, in drug trials, it may be infeasible to run trials on one subject at a time to inform selection of the next subject, so parallelizing reward observations is essential.

\item \textbf{Robustness to approximation of $p$:} Theorem~\ref{theorem:main} assumes the context distribution $p$ is known exactly. However, in some applications, the context distribution $p$ is only known approximately, i.e., one replaces $p$ in Algorithm~\ref{alg:active-lcb} with $\hat{p} \approx p$. In applications, $\hat{p}$ might be the \emph{empirical distribution} from a historical dataset of sampled contexts. Building such a $\hat{p}$ \emph{only} requires \emph{context} data (no reward data), which we believe is available in many applications. Theorem~\ref{thm:unknown-p} in Appendix~\ref{sec:discussion-assumptions} shows that the theoretical guarantees of Algorithm~\ref{alg:active-lcb} are robust to approximation of $p$ in the sense that they decay {smoothly} with the total-variation distance between $p$ and $\hat{p}$. This enables, for example, the following corollary. 
\end{itemize}

\begin{restatable}{corollary}{corrolaryapproxp} Let $\cB = (\cX, \cA, \phi, p, \nu, \thetastar)$ be an SCLB, $\lambda > 0$, $\delta \in (0, 1)$, and $T > 0$ be a sample budget. Suppose that $p$ is unknown but that one has access to $M = \tilde{\Theta}(\abs{\cX} d^2\epsilon^{-2})$ historical iid samples from $p$. Let $\hat{p}$ be the empirical context distribution constructed from these $M$ samples. Let $\hat{\cB} = (\cX, \cA, \phi, \hat{p}, \nu, \thetastar)$ be the corresponding approximate bandit instance and $\beta$ be as defined in \eqref{eq:beta-lambda}. There exists $T_0 = \tilde{O}(d^2)$ so that whenever $T \geq T_0$, with probability $1-\delta$, Algorithm~\ref{alg:active-lcb} outputs $\hat{\pi}$ such that the regret evaluated on the \emph{true} bandit instance ${\cB}$ satisfies 
\begin{align*}
    R(\hat{\pi}) &= \E_{x \sim p} [\max_{a \in \cA} r(x,a) - r(x, \hat{\pi}(x))] \\
    &\leq \tilde{O}(\sqrt{\beta \cdot (\cC_\cB + \epsilon)/{T}}). 
\end{align*}
Moroever, the algorithm runs in polynomial time. 
\end{restatable}

%% file: main_body/alg.tex
\setlength{\algomargin}{.9em}
\begin{algorithm2e}[!b]
\DontPrintSemicolon 
\SetAlgoLined                  
\KwIn{SCLB $\cB = (\cX, \cA, \phi, p, \nu, \theta^\star)$, $\lambda > 0$, $T \in Z_{>0}$, smoothing parameter $\alpha \in [0, 1]$.}
\caption{Active-SCLB}\label{alg:active-lcb}
\BlankLine
\tcp{Compute smoothing distribution ${q}$ (using Franke-Wolfe or an SDP solver)}
Find ${q} \in \Delta^{\cX \times \cA}$ st. $\mathop{\max}\limits_{(x,a) \in \cX \times \cA} \normInline{\phi(x,a)}_{{\SigmaOInv{{q}}}}^2 \leq 2 d$ \label{line:smooth}\; 
\BlankLine
 \tcp{Compute a 2-multiplicative approximation of the solution to \eqref{eq:smooth-opt} using an SDP solver}
 Compute $w \in \Delta^{\cX \times \cA}$ st. $w(x,a) \geq \alpha {q}(x,a)$ for all $(x,a) \in \cX \times \cA$, and $\mathop{\E}\limits_{x \sim p} \mathop{\max}\limits_{a \in \cA} \normInline{\phi(x,a)}_{\SigmaOInv{w}}^2 \leq 2 \mathop{\E}\limits_{x \sim p}\mathop{\max}\limits_{a \in \cA} \normInline{\phi(x,a)}_{\SigmaOInv{{w}^*}}^2$\label{line:approximate}\; 
\BlankLine
\tcp{Ridge regression}
 $\cS \gets \smash{\{(x_1, a_1), ..., (x_T, a_T) \}}$ with $(x_t,a_t) {\sim} w$ \iid \label{line:cS}\; 
 For $t \in [T]$ sample $r_t \gets r(x_t, a_t) $\; 
 $\thetahat \gets \SigmaLambdaSInv{\cS}{\lambda} \sum_{t \in [T]} \phi(x_t, a_t) r_{t} $\; 
 \Return{$\hat{\pi}(x) \gets x \mapsto \argmax_{a \in \cA} \phi(x,a)^\top \thetahat$}
\end{algorithm2e}

%% file: main_body/comparison.tex
\section{ACTIVE VS PASSIVE SAMPLING}\label{sec:comparison}

We now demonstrate that the instance-dependent bound in Theorem~\ref{theorem:main} leads to quantifiably improved performance on an instance-dependent basis. We describe the following family of SCLB instances (parameterized by $A$ and $d$) where active context sampling (Algorithm~\ref{alg:active-lcb}) achieves improved regret bounds compared to the minimax rate. This instance is adapted from the hard instance for passive learning in maximum likelihood estimation \citep{chaudhuri2015convergence}. 

\begin{definition}\label{def:hard-instance} Let $d \in \Z_{>0}, A \in \Z_{>1}$. Let $\cB^\star_{d,A} = (\cX, \cA, \phi, p, \nu, \thetastar)$ be an SCLB with $\cX = [d]$, $\cA = [A]$, $\nu = \cN(0, 1)$ and $\phi, p, \thetastar$ as follows, where $e_i$ is the $i$-th standard basis vector and $\mathds{1}$ is the indicator function:
\begin{align*}
    &\phi(x, a) = 
        e_x \cdot \mathds{1}[a=1] + e_1 \cdot \mathds{1}[a\neq 1], \quad [\theta]_i^\star = \mathds{1}[i \neq 1], \\
        &p(x) = (1 - (d-1)/d^2) \cdot \mathds{1}[x=1] +
        1/d^2 \cdot \mathds{1}[x\neq1].
\end{align*}
\end{definition}

Figure~\ref{fig:hard-instance-visualization} provides a visualization for $A = d= 2$. This SCLB has one high-probability context $x=1$ and $(d-1)$ low-probability contexts. The high probabilty context \emph{always} reveals $\theta_1^\star$, which is 0. Any algorithm can only learn about high-reward actions when it queries a context $x \neq 1$. If an algorithm passively samples contexts, this happens rarely (only with probability roughly $1/d$.) 
On the other hand, an active context sampling algorithm can actively \emph{upsample} these rarer contexts in order to gain information about $\thetastar$'s remaining $(d-1)$ coordinates---allowing it to more efficiently determine the best action to select for these contexts. This suggests that active context sampling should perform well on $\cB^\star_{d,A}$. We formalize this by showing $\cC_{B^\star_{d, A}}$ is \emph{independent of $d, \lambda, |\cA|$} as follows. 

\begin{restatable}{lemma}{activesdphardinstance}\label{lemma:active-sdp-hard-instance} For any $d \in \Z_{>0}, A \in \Z_{>1}$, $\smash{\cC_{\cB^\star_{d,A}} \leq 4}$. 
\end{restatable}

Thus, on $\cB^\star_{d, A}$ Theorem~\ref{theorem:main} gives a regret bound of $\smash{\Tilde{O}(\sqrt{\beta/T})}$. In contrast, prior polynomial-time algorithms (RFLinUCB and Planner-Sampler) passively sample contexts and only guarantee a bound of $\smash{\Tilde{O}(\sqrt{\beta d/T})}$, which is worse by a factor of $\sqrt{d}$ \citep{zanette2021design}. This illustrates a concrete case where our result can be stronger than prior work by up to a $\smash{\sqrt{d}}$ factor! We validate this improvement empirically in Section~\ref{sec:experiments}. Appendix~\ref{sec:subsample} shows this $\sqrt{d}$ improvement persists \emph{even} if $p$ is unknown but \emph{approximated} from roughly $\Theta(d^2)$ $\text{i.i.d.}$ \emph{context samples} from $p$. 



\paragraph{Adaptive sampling SCLB bounds.} \citet{li2022instance} use passive context sampling with a \emph{data-adaptive} exploration policy to obtain instance-dependent rates, which, in some cases might be tighter than Theorem~\ref{theorem:main}. Their algorithm uses passive context sampling to strategically sample actions with a large reward gap data-adaptively. In other words, \citet{li2022instance} extends prior work on data-adaptive best-arm identification for stochastic non-contextual bandits to stochastic contextual linear bandits (SCLBs). However, to our knowledge there is no polynomial-time implementation of \citet{li2022instance}'s data-adaptive algorithm. Motivated by applications, our focus is on designing polynomial-time algorithms. 

Nevertheless, introducing active context sampling provably \emph{still} improves the passive-learning data-adaptive sampling algorithm in \citep{li2022instance}! Appendix~\ref{apx:additional_results} presents a data-adaptive active context sampling analog of \citet{li2022instance}'s exponential-time algorithm and proves this would improve over the regret bounds attained in \cite{li2022instance}---in particular, on the class of SCLBs proposed in Definition~\ref{def:hard-instance}, our active-context sampling variant of \citet{li2022instance}'s original passive-context sampling algorithm tightens the resulting bound by a factor of $\sqrt{d}$---just as we find for our polynomial-time algorithm. While we omit this analysis in the main body (as our focus is on tractable algorithms) our results highlight that active context sampling has the potential to strengthen contextual bandit algorithms more broadly.

%% file: main_body/experiments.tex
\section{NUMERICAL EXPERIMENTS}\label{sec:experiments}

We implement our exploration policy (Algorithm~\ref{alg:active-lcb}) with $\alpha \gets 0$ (denoted Active-SCLB). We compare Active-SCLB to RF-LinUCB \citep{zanette2021design, abbasi2011improved} and the Planner-Sampler method of \cite{zanette2021design}. We also compare to a third baseline (Passive-SCLB), which is a natural passive learning analog of our Active-SCLB, where in Line~\ref{line:approximate} we enforce additional constraints to force the contexts to be sampled passively from $p$. 

Additional details explaining these choice of baselines are in Appendix~\ref{apx:additional-experiments}, where, for example, we explain why RFLinUCB is a fairer comparison than LinUCB. Using each method (Active-SCLB and baselines) we collect a dataset of $T$ samples and train a policy using ridge regression (recall Lemma~\ref{lemma:simple-regret-bound-general}) and examine how the regret defined in \eqref{eq:simple-regret-general} decays with $T$. 

All experiments were performed on a CPU machine with 12 cores and 36 GB RAM. We use MOSEK as our numerical SDP solver. Further detail about the experimental setup can be found in Appendix~\ref{apx:additional-experiments}. Code to reproduce experiments is linked in Appendix~\ref{sec:organization}. 

\paragraph{Synthetic experiments.} We first study the synthetic SCLB instance $\cB_{d, A}^\star$ in Definition~\ref{def:hard-instance}. We fix $\lambda=1\mathrm{e-}6$ and vary $d$ in the first three columns of Figure~\ref{fig:all-experiments}. The top row shows the regret gap between Active-SCLB and baselines is more pronounced as $d$ grows. This is consistent with the $\smash{O(\sqrt{d})}$ theoretical gap between the regret bound of Active-SCLB and other methods, as dicussed in Section~\ref{sec:comparison}. The bottom row demonstrates that the baselines need far more data to match the performance of Active-SCLB at a given sample budget. 

\paragraph{Real-world data.} Next, we evaluate our method on two real-world datasets, Warfarin and Jester. For these real-world datasets, we also report the regret of a \emph{naive baseline}, which is the regret of the naive policy that selects the \emph{same action for each context} ($\smash{\pi_{\mathrm{naive}}(x) = \argmax_{a \in \cA} \E_{x' \sim p} r(x',a)}$ for all $\smash{x \in \cX}$). This baseline performs poorly, which certifies that on these real-world datasets, the contextual information is useful. To realistically model real-world settings where context-action pairs would generally be drawn \emph{without} replacement, for these real-world datasets we (slightly) modify each of the methods  to sample context-action pairs without replacement (using rejection sampling) when reporting our results. 

For Warfarin, we use the Warfarin Pharmacogenetics Consortium dataset in the warfit-learn package, which is a cleaned clinical dataset of 5650 patients taking the blood thinner warfarin \cite{international2009estimation, truda2021evaluating} (after removing duplicates). Each patient is associated with 31 features corresponding to demographics and health history. The task is to select the best dosage (action) of $\{$``low'', ``medium'', and ``high''$\}$ for a given patient (context), corresponding to 3 actions and 17,223 context-action pairs. The right dosage is patient-specific and challenging to determine; too high a dose may cause internal bleeding but too low a dose may cause blood clots \citep{international2009estimation, truda2021evaluating}. We model the context distribution $p$ as uniform over contexts. The reward of a context-action (patient-dosage) pair $(x,a)$ is $+1$ if the dosage is correct ($0$ otherwise). The regret of a (dosage) policy is the fraction of patients \emph{mis-dosed}. 

For Jester, we use the cleaned version \citep{kong2020sublinear} of the Jester dataset \cite{goldberg2001eigentaste}, which contains ratings of 48,447 users on 100 jokes (we subsample down to 2,000 users to keep experiments tractable.) Similar to \citet{kong2020sublinear}, we hold out the top 5 ``gold'' jokes with the highest average ratings. For each user, we create a $30$-dimensional feature vector by multiplying their $95$-dimensional feature vector of joke ratings (for the remaining jokes) with a $95 \times 30$ matrix whose entries are iid $\cN(0,1)$ and applying a sigmoid to each of the resulting $30$ values. The task is to predict the best ``gold'' joke (action) for each user (context), resulting in 10,000 context-action pairs. We model $p$ as uniform across contexts. The reward of $(x,a)$ is the user $x$'s rating of gold joke $a$.

\begin{figure*}[!tb]
  \centering
  \captionsetup[subfigure]{font=small}
  \setlength{\tabcolsep}{1.5pt}
  \renewcommand{\arraystretch}{0.95}

  \begin{tabular}{@{\hskip -2pt}ccccc@{\hskip -2pt}}
    \begin{minipage}[t]{0.19\linewidth}
      \centering
      \includegraphics[width=\linewidth]{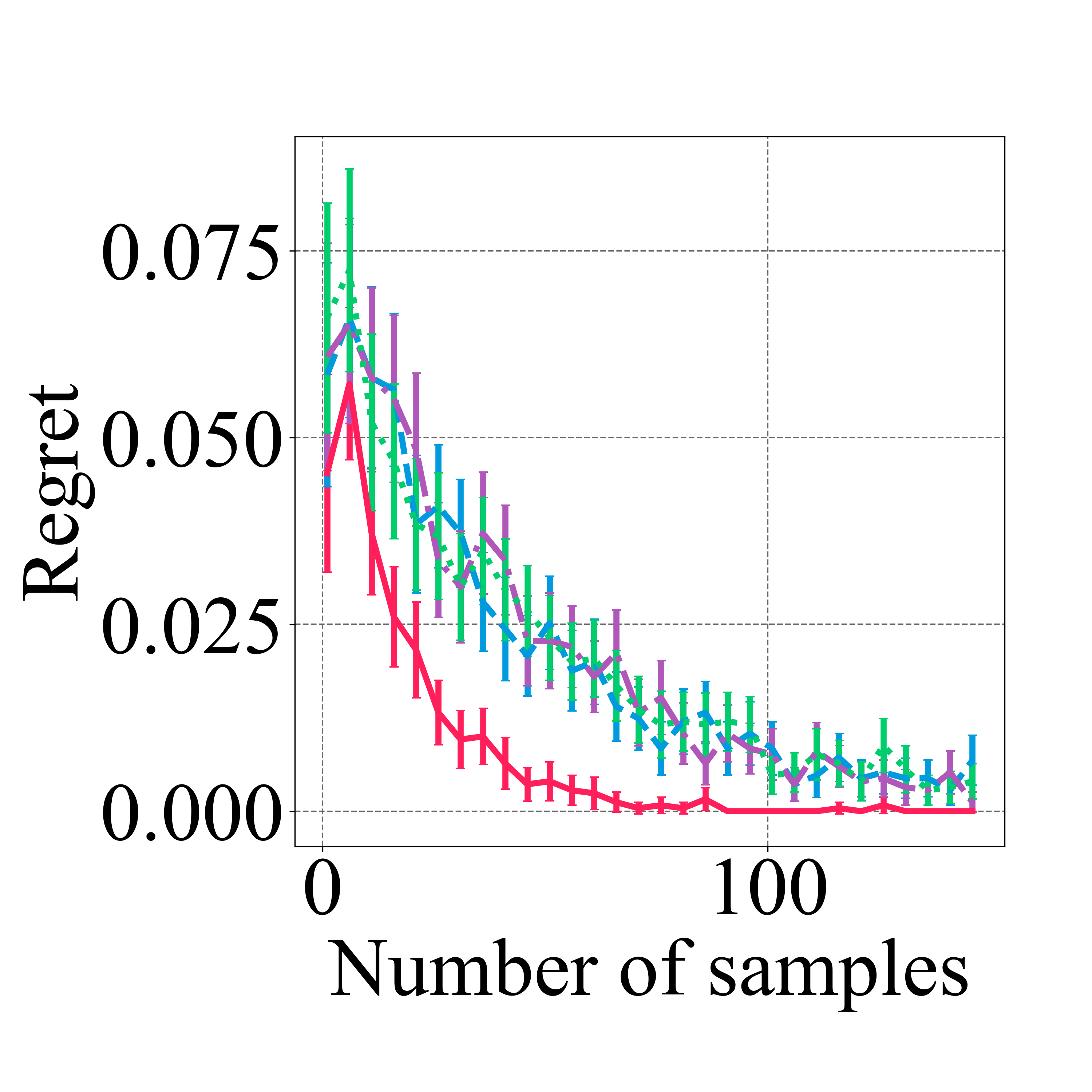}\vspace{-.5em}
      \includegraphics[width=\linewidth]{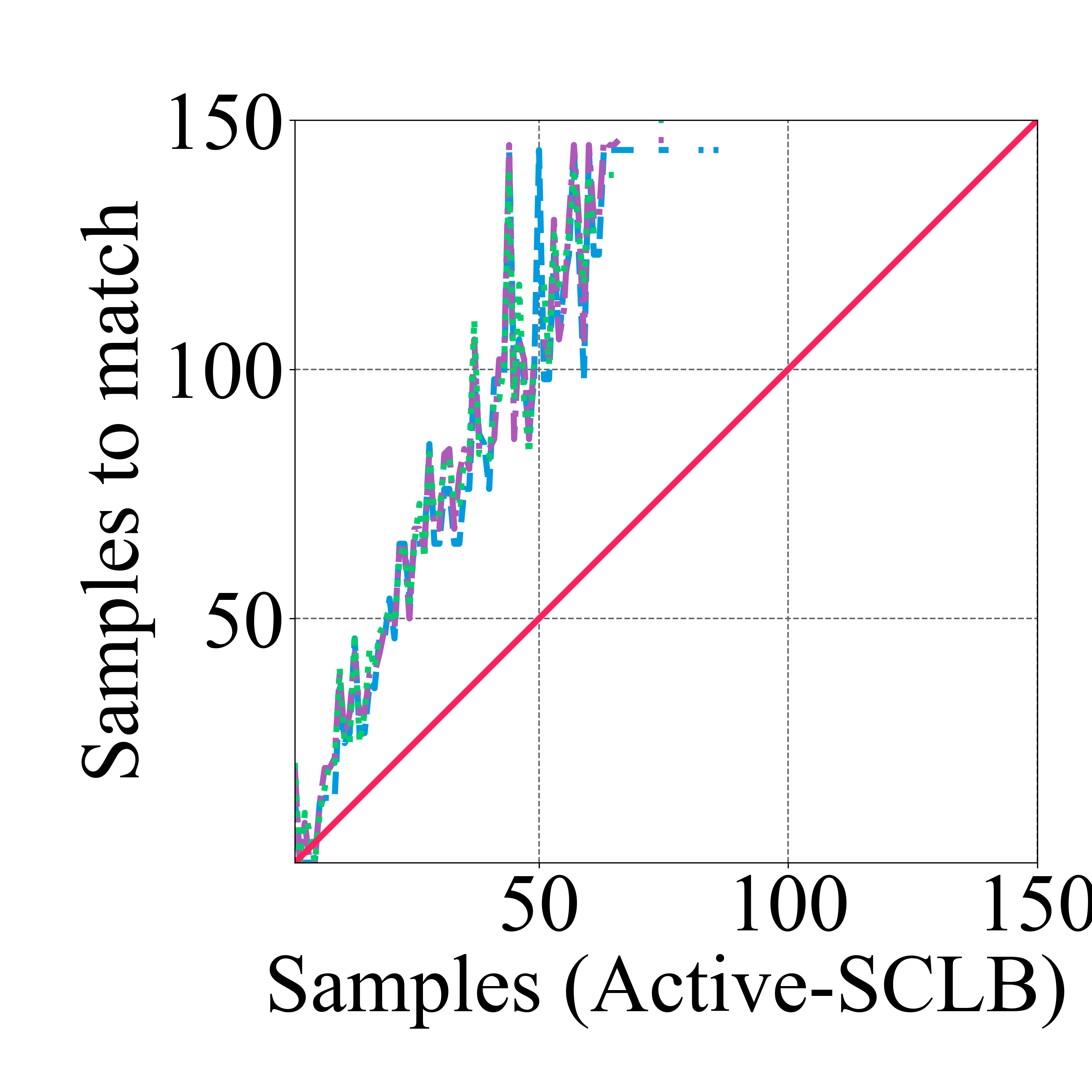}
      \caption*{(a) $\cB^\star_{5, 10}$}
    \end{minipage}&
    \begin{minipage}[t]{0.19\linewidth}
      \centering
      \includegraphics[width=\linewidth]{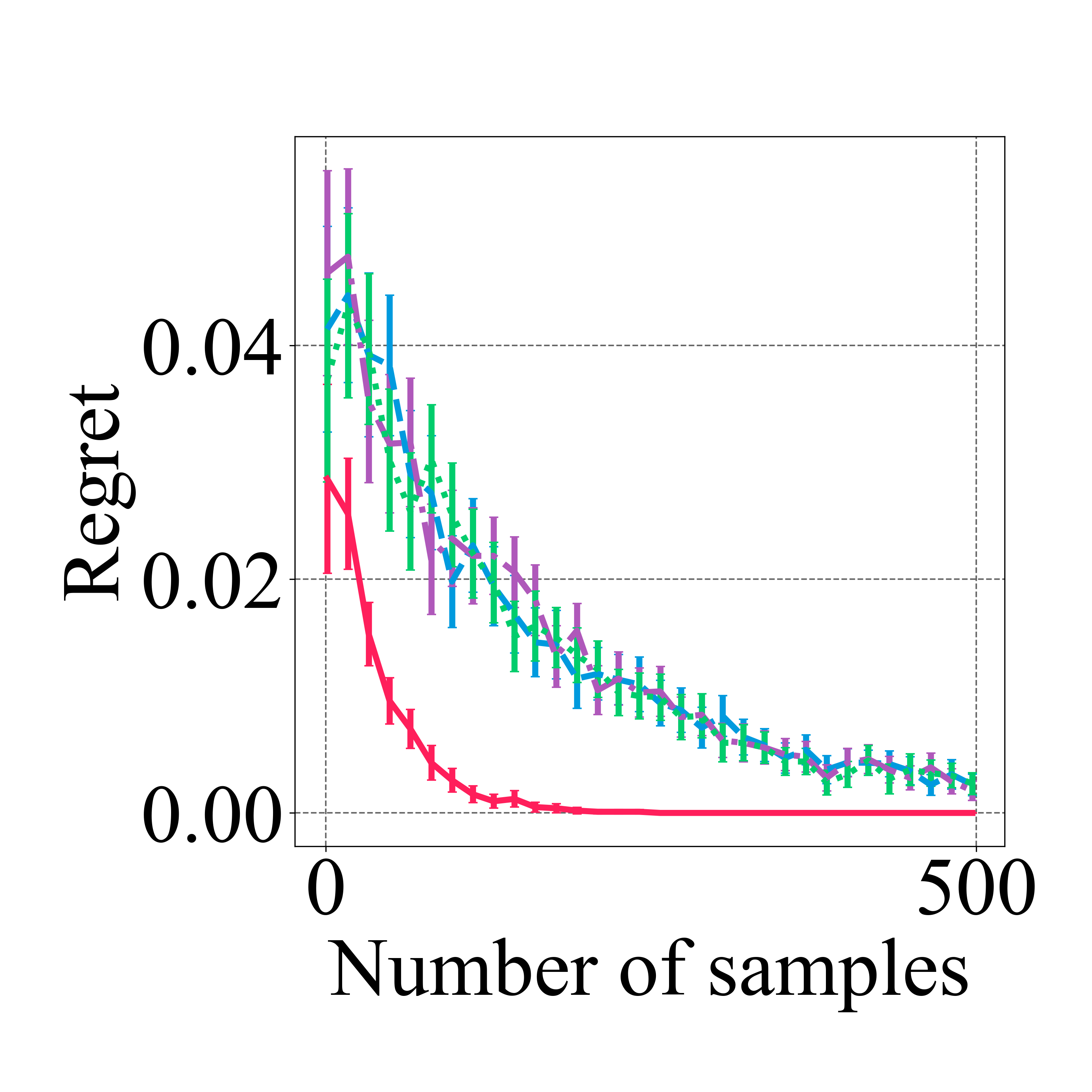}\vspace{-.5em}
      \includegraphics[width=\linewidth]{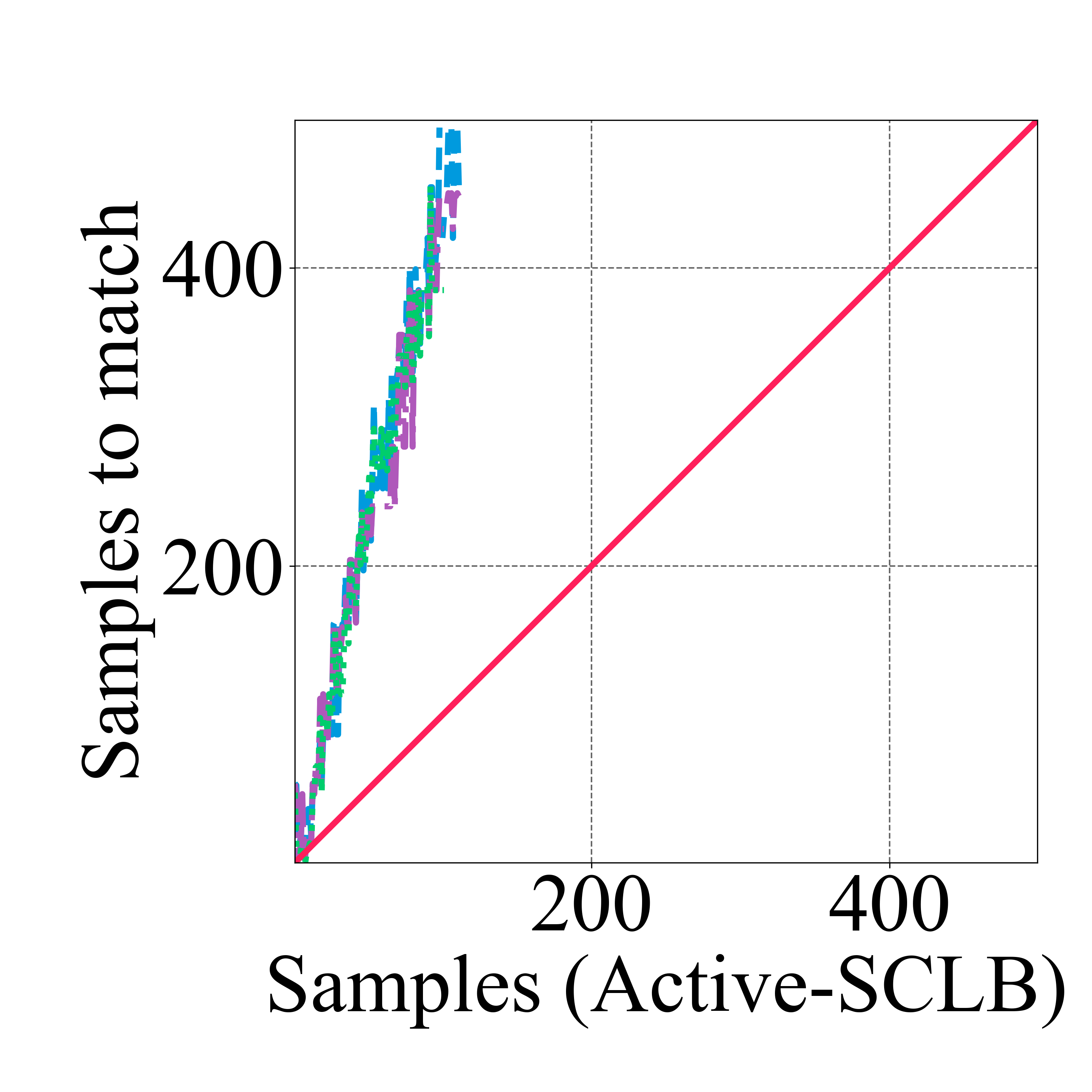}
      \caption*{(b) $\cB^\star_{10, 10}$}
    \end{minipage}&
    \begin{minipage}[t]{0.19\linewidth}
      \centering
      \includegraphics[width=\linewidth]{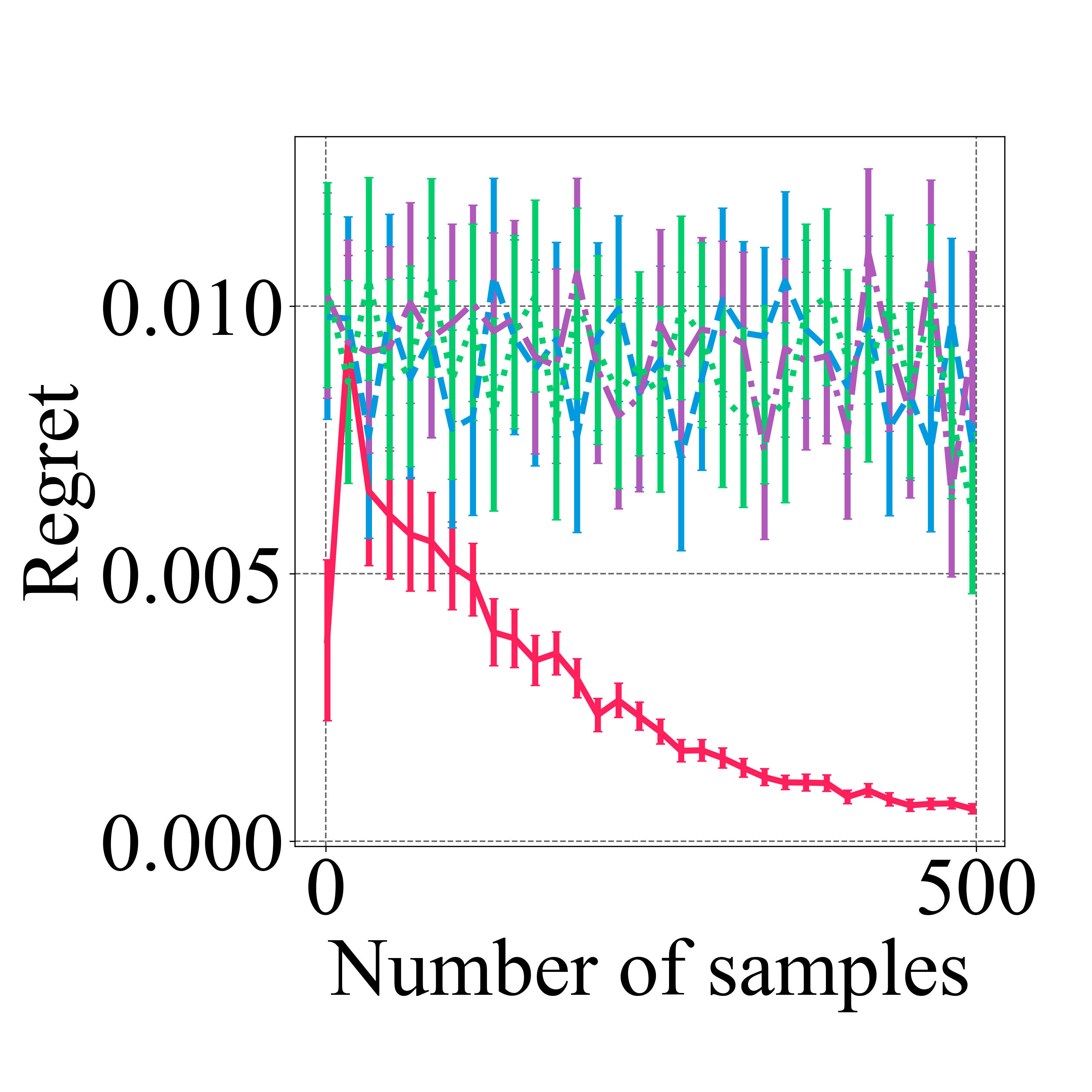}\vspace{-.5em}
      \includegraphics[width=\linewidth]{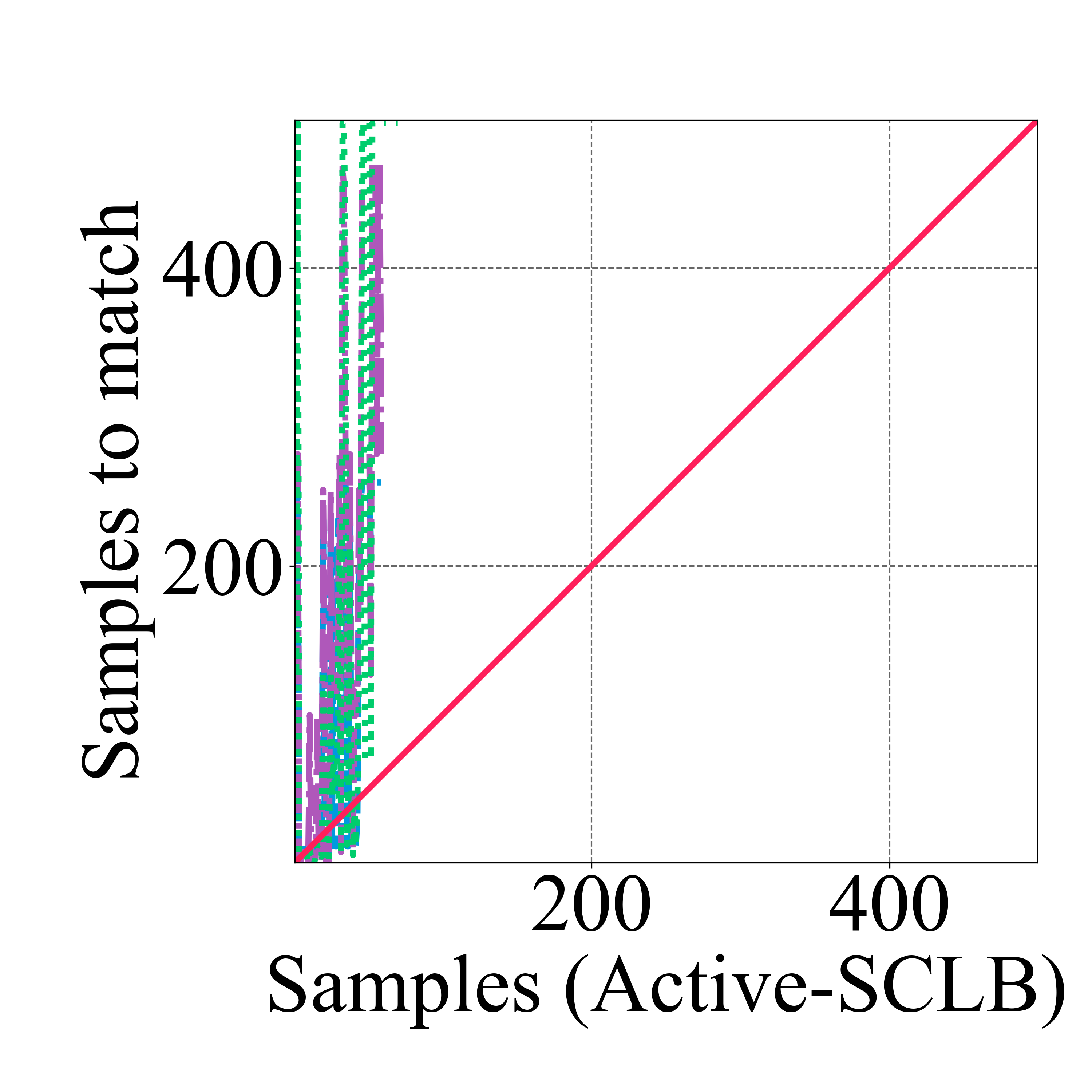}
      \caption*{(c) $\cB^\star_{50, 10}$}
    \end{minipage}&
    \begin{minipage}[t]{0.19\linewidth}
      \centering
      \includegraphics[width=\linewidth]{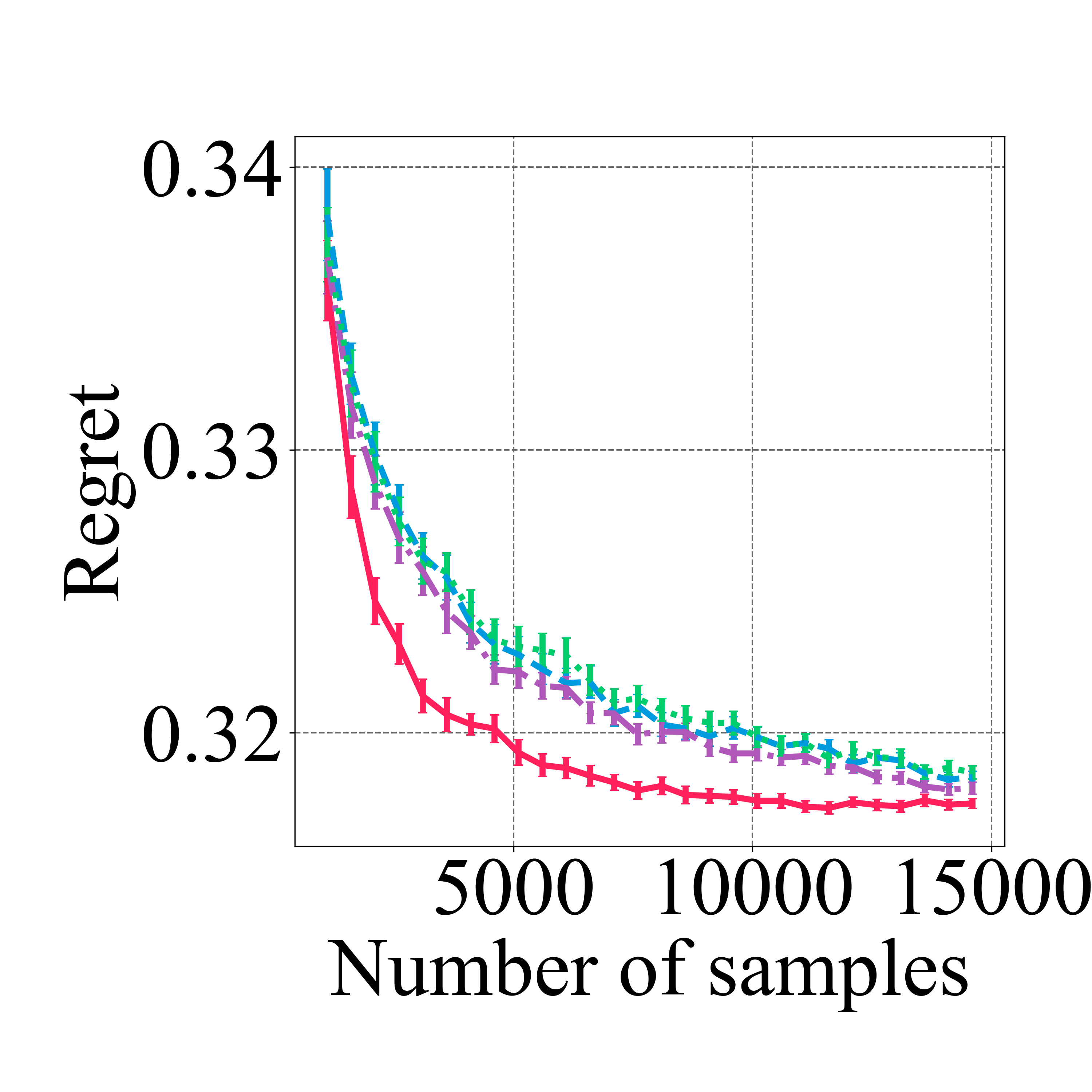}\vspace{-.5em}
      \includegraphics[width=\linewidth]{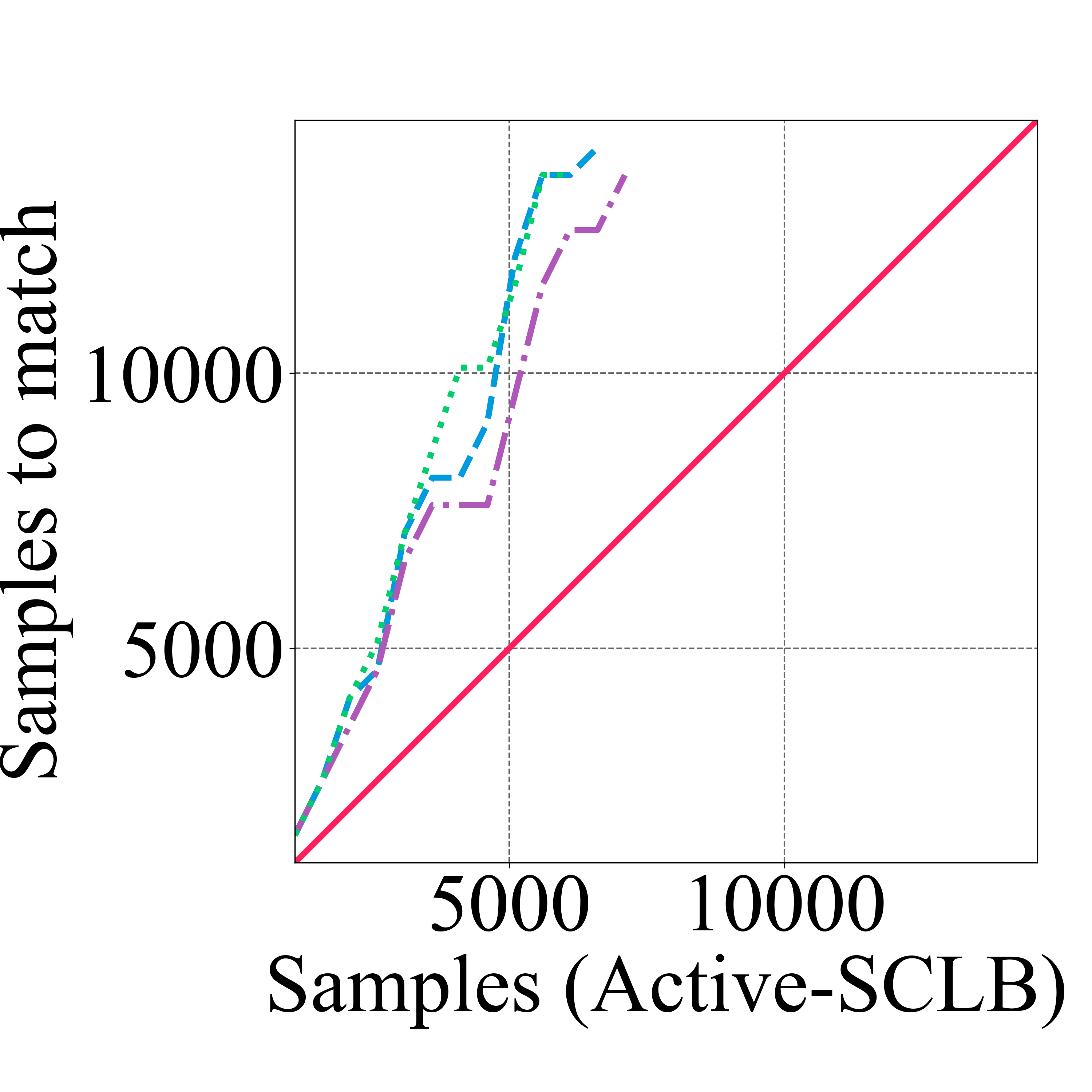}
      \caption*{(d) Warfarin}
    \end{minipage}&
    \begin{minipage}[t]{0.19\linewidth}
      \centering
      \includegraphics[width=\linewidth]{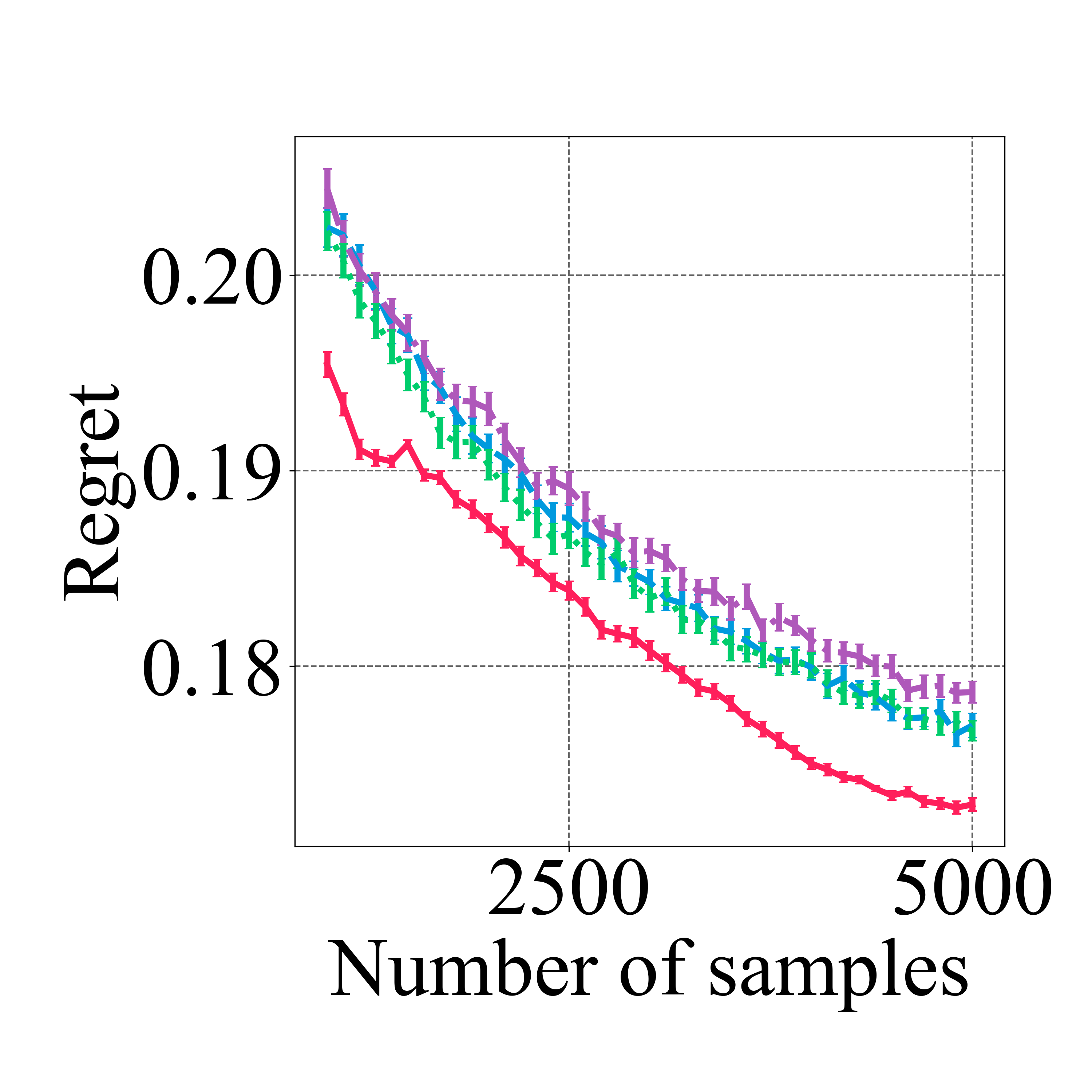}\vspace{-.5em}
      \includegraphics[width=\linewidth]{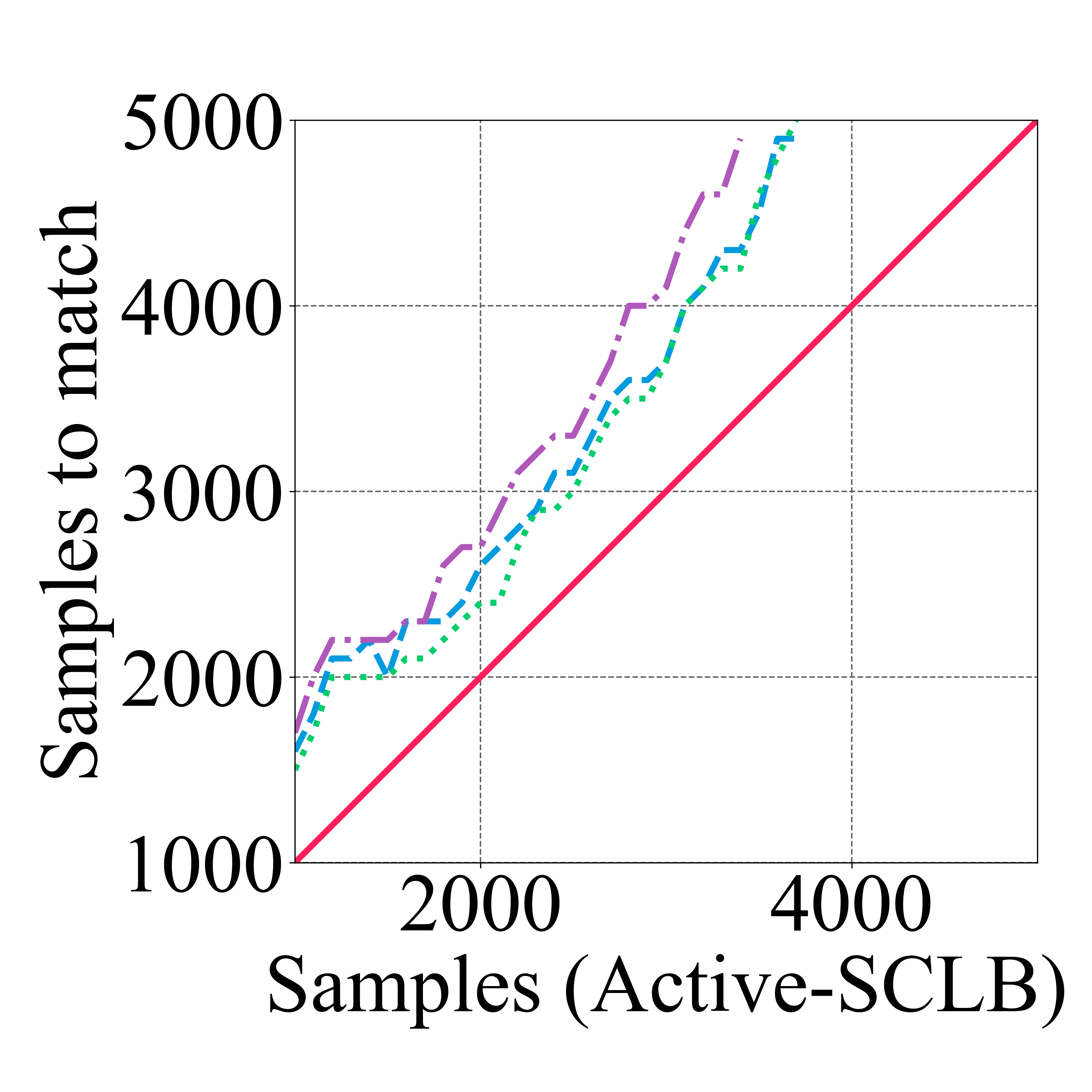}
      \caption*{(e) Jester}
    \end{minipage}
  \end{tabular}

  \vspace{0.6em}
  \includegraphics[width=.8\linewidth]{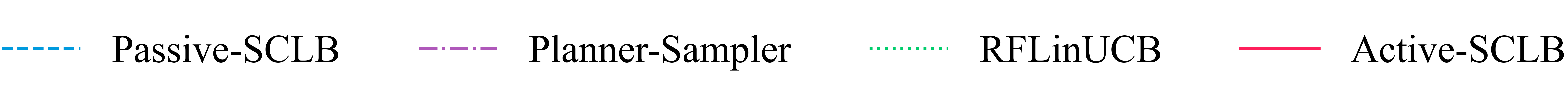}

  \caption{
    Numerical experiments: first three columns show synthetic results on $\mathcal{B}^\star_{d,10}$ (Definition~\ref{def:hard-instance}) for $d \in \{5, 10, 50\}$;
    last two columns show real-world datasets (Warfarin, Jester).
    Top row: Regret vs.\ number of samples (mean $\pm$\,2 standard errors over 100 trials).
    Bottom row: Minimum samples required for baselines' mean regret to match Active-SCLB's mean regret.
    The \emph{naive baseline} has regret 0.382 (Warfarin) and 0.219 (Jester).
  }
  \label{fig:all-experiments}
\end{figure*}

If $k$ users have the same rankings, then they would have the same 95-dimensional feature vector, and  the implicit effect, in our implementation, is that context is up-weighted by a factor of $k$ in the effective context distribution, since we assume that 
is uniform over the original possibly duplicated users. However, our algorithms and experiments continue to work, since there is no assumption that different contexts must have different feature vectors.

Th right two columns of Figure~\ref{fig:all-experiments} show results for $\lambda = 1\text{e-}6$. The top row shows Active-SCLB consistently outperforms baselines. The bottom row shows that on Warfarin, Active-SCLB often requires 5,000 fewer samples to achieve similar regret to the baselines; meanwhile, on Jester, it sometimes requires 1,500 fewer samples.

%% file: main_body/discussion.tex
\paragraph{Runtime.} Active-SCLB and Passive-SCLB have a runtime overhead to initialize a sampling distribution, but after this initial runtime overhead, sampling values from this distribution is fast even if  $|\cX|$ is large and in particular, the reward sampling can be done in parallel \citep{zanette2021design}. RFLinUCB has no initialization overhead because it sequentially selects the next action, depending on the past. However, each sample must be implemented sequentially, so its complexity scales like $T \times R$ where $R$ is the time to collect a single reward observation. Thus, while RFLinUCB has low overhead and is computationally cheap, its practical application can be slow in settings where each reward observation takes time (e.g., clinical trials, educational interventions, world simulations etc.), i.e., when $R$ is large. As \citep{zanette2021design}, this is a key consideration which is not captured in pure numerical simulations, where $R$ is artificially small.

In Appendix~\ref{sec:runtime}, we report runtime (minutes) on Warfarin. Initialization overhead of Active-SCLB is slightly higher than Passive-SCLB and Planner-Sampler due to its added complexity. RFLinUCB has no initialization overhead and is the fastest; however, Active-SCLB consistently achieves lower regret (see Figure 1 of our paper.)

\paragraph{Experiments with approximate $p$.} To evaluate our results in the setting where $p$ is only known approximately, in Appendix~\ref{sec:approximate_p}, we run additional experiments on Warfarin and Jester, where we modified Active-SCLB to replace $p$ with $\hat{p}$, where $\hat{p}$ is the empirical distribution constructed from $|\cX|/4$ random contexts (denoted Active-SCLB-Empirical.) We find that Active-SCLB-Empirical is sometimes slightly worse than Active-SCLB but performs similarly overall and consistently outperforms baselines. Baselines require significantly more samples for the average regret to match that of Active-SCLB-Empirical. 

\paragraph{Varying regularization parameters.} In Appendix~\ref{sec:regularizer}, we include additional experiments ablating the effect of the choice of regularization parameter $\lambda$ on our results. Overall, we find that our results are largely consistent across different scales of $\lambda$.

\section{DISCUSSION}\label{sec:conclusion}

We present an active context sampling algorithm for SCLBs. We prove an \emph{instance-dependent} regret guarantee which in the worst-case, matches the minimax-optimal rate and in the best-case is up to a $\smash{\sqrt{d}}$-factor tighter than comparable works \cite{zanette2021design, abbasi2011improved}. We validate our theoretical analysis with empirical results illustrating the benefits of active context sampling in SCLBs. We hope our work provides useful techniques for future research on the role of active sampling in contextual learning. 

We acknowledge two limitations in our work, which we hope will motivate directions future research. First, works on SCLBs often make a linear realizability assumption \citep{zanette2021design, abbasi2011improved}. However, on the Warfarin and Jester datasets, our experiments show that Algorithm~\ref{alg:active-lcb} performs well even though we expect the true reward models are nonlinear. This suggests our procedure can perform well empirically \emph{even} if the problem is \emph{misspecified} as an SCLB. Thus, extending our techniques to more complex reward models (either empirically or theoretically) would be interesting avenues for future work.

Another limitation in our work is that although SDP-solving is fast for moderately-sized SDPs, SDP-solving could be expensive in massive context-action spaces. On Warfarin and Jester, our experiments show that our method easily scales to at 15,000 context-action pairs. On both datasets, we find that SDP solving takes $<1$ hour. To scale to massive problems, practitioners might explore GPU-accelerated SDP solvers \citep{fujisawa2012high,lin2025pdcs} or experiment with heuristic approximations of Line~\ref{line:approximate} (e.g., sub-sampling context-action pairs prior to solving the SDP). While we conjecture it is nontrivial to provide theoretical guarantees for such heuristics, exploring heuristic active context sampling techniques may be valuable direction for empirical researchers. 

%% file: appendix/organization.tex
\section{ORGANIZATION OF APPENDIX SECTIONS}\label{sec:organization}

We briefly provide an outline of the contents in this appendix. In Appendix~\ref{sec:discussion-assumptions} we discuss how some of the SCLB modeling assumptions in the main body can be relaxed to still yield interesting results. In Appendix~\ref{sec:comparing_settings} we include a helpful visualization to compare our setting of active context sampling for contextual bandits with prior work on linear bandits, passive context sampling for contextual bandits, and traditional active learning for regression. In Appendix~\ref{apx:additional_results} we discuss additional theoretical findings, including an active context sampling version of the data-adaptive ContextualRAGE algorithm proposed in \cite{li2022instance}. In Appendix~\ref{apx:additional-experiments} we discuss additional implementational details, including additional experiments on Warfarin dosage and Joke recommendation with different values of the regularization parameter $\lambda$. In Appendix~\ref{apx:omitted-proofs} we include the full proofs of the theoretical results which were omitted (or only sketched) in the main body. 

\paragraph{Link to code for experiments.} We have made code available to the reviewers at the following anonymous repository: \url{https://anonymous.4open.science/r/ACLB_release-1B6E/README.md}. 

%% file: appendix/assumptions.tex
\section{DISCUSSION OF SCLB MODEL, LIMITATIONS, AND RELAXATIONS}\label{sec:discussion-assumptions}

In this section, we discuss in more detail the role of some assumptions of our SCLB model, and how some of them can be relaxed under appropriate conditions. 

\subsection{Relaxing the assumption that the context distribution $p$ is known.}\label{subsec:approxpsection} One assumption in our work is that the context distribution $p$ is known a-priori. Although we believe this to be reasonable for some applications (see for example, the motivating examples presented in Section~\ref{sec:intro}); in other settings $p$ may only be known \emph{approximately} from \emph{prior historical data from the context distribution}. Consequently, a natural question is the following: 
\begin{center}
    \emph{Suppose that the exact context distribution $p$ is unknown a priori and that the learner only has access to an approximate context distribution $\hat{p} \in \Delta^{\cX}$. Suppose that one runs Algorithm~\ref{alg:active-lcb} with $\hat{p}$ in place of $p$. How do the theoretical guarantees decay as a function of the error between $p$ and $\hat{p}$?}
\end{center}

To answer this question, in Section~\ref{sec:TV} we show that the theoretical guarantees of our algorithm decay smoothly as a function of the total variation distance between $p$ and $\hat{p}$, and we quantify the amount of historical context data needed to control this total variation distance. 

Furthermore, Section~\ref{sec:subsample} we show that on the bandit instance from Definition~\ref{def:hard-instance}, a coarse approximation of $\hat{p}$ constructed from a dataset of $\Theta(d^2)$ sampled contexts is sufficient to maintain the $\sqrt{d}$ improvement over passive context sampling. 

\subsubsection{Theoretical guarantees decay with the total variation distance between $p$ and $\hat{p}$}\label{sec:TV}

In the following, we use $\mathrm{tv}(p, p') \defeq 1/2 \cdot \sum_{x \in \cX} |p(x) - p'(x)|$ to denote the total variation distance between two discrete distributions over a context set $\cX$. 

\begin{theorem}\label{thm:unknown-p} Let $\cB = (\cX, \cA, \phi, p, \nu, \thetastar)$ be an SCLB, $\lambda > 0$, $\delta \in (0, 1)$, and $T > 0$ be a sample budget. Suppose that $p$ is unknown but that one has access to an arbitrary approximation $\hat{p} \in \Delta^\cX$. Let $\hat{\cB} = (\cX, \cA, \phi, \hat{p}, \nu, \thetastar)$ be the corresponding approximate bandit instance and $\beta$ be as defined in \eqref{eq:beta-lambda}. There exists $T_0 = \tilde{O}(d^2)$ so that whenever $T \geq T_0$, with probability $1-\delta$, Algorithm~\ref{alg:active-lcb} outputs $\hat{\pi}$ such that the regret evaluated on the true bandit instance ${\cB}$ satisfies 
\begin{align*}
    R(\hat{\pi}) = \E_{x \sim p} [\max_{a \in \cA} r(x,a) - r(x, \hat{\pi}(x))] \leq \tilde{O}\paren{\sqrt{\beta \cdot \frac{\cC_\cB + d \cdot \mathrm{tv}(p,p')}{T}}}. 
\end{align*}
Moroever, the algorithm runs in polynomial time. 
\end{theorem}

To prove the theorem, we first prove the following helper lemma. 
\begin{lemma}\label{lemma:helpmeout} Consider the setting of Theorem~\ref{thm:unknown-p}. Let $q$ be as in Line~\ref{line:smooth}. Let $\hat{w}$ be the choice that Line~\ref{line:approximate} of Algorithm~\ref{alg:active-lcb} would select when invoked on $\hat{\cB}$ with $\alpha \gets 1/2$. Then, 
\begin{align*}
      \E_{x \sim p} \max_{a \in \cA} \normInline{\phi(x,a)}_{\Sigma_{\hat{w}}^{-1}}^2 \leq 4\cC_{\cB} + 32 d \cdot \mathrm{tv}(p, \hat{p}). 
\end{align*}
\end{lemma}
\begin{proof} For notational convenience, for any bandit $\cB' = (\cX, \cA, \phi, p', \nu, \thetastar)$ and for any distribution $w \in \Delta^{\cX, \cA}$, we use the following shorthand in the remainder of the proof.
\begin{align*}
    \kappa(\cB', w) \defeq \E_{x \sim p'} \max_{a \in \cA} \normInline{\phi(x,a)}_{\Sigma_{w}^{-1}}^2. 
\end{align*}

Consider any $w$ which is feasible for \eqref{eq:smooth-opt} for $\alpha = 1/2$. Then, by the constraint that $w$ dominates $\alpha \cdot q = 1/2 q$ we have that $\Sigma_w \succeq 1/2 \cdot \Sigma_{q}$, and consequently, 
\begin{align*}
     \abs{\kappa(\cB, w) - \kappa(\hat{\cB}, w)} &= \abs{\mathbb{E}_{x \sim p} [\max_{a \in \mathcal{A}} \Vert \phi(x,a) \Vert_{\Sigma_w^{-1}}^{2}] - \mathbb{E}_{x \sim \hat{p}} [\max_{a \in \mathcal{A}} \Vert \phi(x,a) \Vert_{\Sigma_w^{-1}}^{2}]} \\
     &= \abs{\sum_{x \in \mathcal{X}} [p(x) - \hat{p}(x)] \cdot \max_{a \in \mathcal{A}} \Vert \phi(x,a) \Vert_{\Sigma_w^{-1}}^{2}} \\
     &\leq \sum_{x \in \mathcal{X}} |p(x) - \hat{p}(x)| \cdot \max_{a \in \mathcal{A}} \Vert \phi(x,a) \Vert_{\Sigma_w^{-1}}^{2} \\
    &\leq 4d \sum_{x \in \mathcal{X}} |p(x) - \hat{p}(x)| \\
    &\leq 8d \cdot \mathrm{tv}(p, \hat{p})
\end{align*}
where the second-to-last inequality holds because $\Sigma_w \succeq 1/2 \cdot \Sigma_{{q}}$ ensures that for any $(x, a) \in \cX \times \cA$
\begin{align*}
\Vert \phi(x,a) \Vert_{\Sigma_w^{-1}}^{2} \leq 2 \Vert \phi(x,a) \Vert_{\Sigma_q^{-1}}^{2} \leq 2 \cdot 2d = 4d,
\end{align*}
and the last inequality holds by definition of the total-variation distance $\mathrm{tv}(p, \hat{p})$. 

Now, let $\hat{w}^*$ and $w^*$ be the choices of distributions that \eqref{eq:smooth-opt} would select when invoked on $\hat{\cB}$ and $\cB$ respectively. That is, 
\begin{align*}
    w^* &= \argmin_{w \in \Delta^{\cX \times \cA}} \E_{x \sim p} \max_{a \in \cA} \normInline{\phi(x,a)}_{\Sigma_w^{-1}}^2, ~~~\text{ subject to } w(x,a) \geq \alpha \cdot {q}(x,a), ~\forall (x,a) \in \cX \times \cA \\
    \hat{w}^* &= \argmin_{w \in \Delta^{\cX \times \cA}} \E_{x \sim \hat{p}} \max_{a \in \cA} \normInline{\phi(x,a)}_{\Sigma_w^{-1}}^2, ~~~\text{ subject to } w(x,a) \geq \alpha \cdot {q}(x,a), ~\forall (x,a) \in \cX \times \cA. 
\end{align*}
Then, by the argument above, we have that 
\begin{align}\label{eq:bounder}
    \kappa({\cB}, \hat{w}^*) \leq \kappa(\hat{\cB}, \hat{w}^*) + 8d\mathrm{tv}(p, \hat{p}) \leq \kappa(\hat{\cB}, w^*) + 8d\mathrm{tv}(p, \hat{p}) \leq \kappa(\cB, w^*) + 16 d \mathrm{tv}(p, \hat{p}), 
\end{align}
where the first step follows because $|{\kappa(\cB, \hat{w}^*) - \kappa(\hat{\cB}, \hat{w}^*)}| \leq 8d \mathrm{tv}(p, p')$, the second step follows because of the optimality conditions for $\hat{w}^*$, and the third step follows because $|\kappa(\cB, {w}^*) - \kappa(\hat{\cB}, {w}^*)| \leq 8d \mathrm{tv}(p, p')$. 

By the definition of $\hat{w}$, note that
\begin{align}\label{eq:approximationhere}
    \kappa(\hat{\cB}, \hat{w}) \leq 2 \kappa(\hat{\cB}, \hat{w}^*). 
\end{align}
Consequently, 
\begin{align*}
    \kappa(\cB, \hat{w}) \leq \kappa(\hat{\cB}, \hat{w}) + 8 d\mathrm{tv}(p, p') 
    \leq 2 \kappa(\hat{\cB}, \hat{w}^*) + 8 d \mathrm{tv}(p, p') 
    \leq 2 \kappa(\cB, w^*) + 24 d \mathrm{tv}(p, \hat{p}), 
\end{align*}
where the first inequality holds because $|\kappa(\cB, \hat{w}) - \kappa(\hat{\cB}, \hat{w})| \leq 8d \mathrm{tv}(p, p')$; the second inequality holds by \eqref{eq:approximationhere}; and the last inequality holds due to \eqref{eq:bounder}. The result now follows by Lemma~\ref{lemma:no-loss}, which ensures that $\kappa(\cB, w^*) \leq 2 \cC_{\cB}$. 


\end{proof}

\begin{proof}[Proof of Theorem~\ref{thm:unknown-p}] The proof follows identically as that of Theorem~\ref{theorem:main}, except that the second-to-last display instead holds (by Lemma~\ref{lemma:helpmeout}) with
\begin{align*}
    \E_{x\sim p} \max_{a \in \cA} \normInline{\phi(x,a)}_{\Sigma_{\cS}^{-1}} \leq \frac{2}{T} \E_{x \sim p} \max_{a \in \cA} \normInline{\phi(x,a)}_{\Sigma_w^{-1}}^2 \leq \frac{8}{T} \cC_{\cB} + 48 d \cdot \mathrm{tv}(p, p'). 
\end{align*}
\end{proof}

We also obtain the following useful corollary. 
\corrolaryapproxp*
\begin{proof} The proof follows by applying Theorem~\ref{thm:unknown-p} and noting that $M = \tilde{O}(|\cX| d^2 \epsilon^{-2})$ is sufficient samples such that with high probability, $\mathrm{tv}(p, p') \leq \epsilon/d$ \citep{canonne2020short}. 


\end{proof}

\subsubsection{Sustained improvement on the bandit instance from Definition~\ref{def:hard-instance}}\label{sec:subsample}

A natural and important question is whether the gains to active context sampling over passive sampling strategies still persist when we use an empirical estimate of the context distribution. 

We expect this to be true in light of the results in the previous sub-section, and we now show on the example from Definition~\ref{def:hard-instance}, that optimizing with respect to an empirical context distribution (instead of the true context distribution) still enables our approach to recover our $\sqrt{d}$-factor improvement over passive context sampling. 

For the example instance in Definition~\ref{def:hard-instance}, suppose that we build an empirical estimate $\hat{p}$ using $M$ samples of \textit{only contexts}. We believe it is very common for there to be prior large datasets of contexts -- consider recommendation systems, etc. Recall the multiplicative Chernoff bounds: for $X_{1\cdots M}$ independent Bernoulli random variables in (0,1) with mean $p_k$, and $\epsilon < 1$ then
\begin{equation*}
\mathbb{P}\left(\frac{1}{M} \sum_i X_i \geq (1+\epsilon) p_k\right) \leq e^{- \epsilon^2 M p_k / (2 + \epsilon)}
\end{equation*}
and 
\begin{equation*}
\mathbb{P}\left(\frac{1}{M} \sum_i X_i \leq (1-\epsilon) p_k\right) \leq e^{- \epsilon^2 M p_k / 2}
\end{equation*}
Set $\epsilon=.25$. Thus, $M_k = \frac{36}{p_k} \log (2 / \delta)$  is sufficient to ensure that for the resulting estimate $\mathbb{P}(\hat{p}_k \geq (1+\epsilon) p_k) + \mathbb{P}(\hat{p}_k \leq (1-\epsilon) p_k) \leq \delta$, where $\hat{p}_k$ is the empirical estimate $\frac{1}{M} \sum_i X_i$. 
We use a union bound to ensure this holds for all $d$ contexts, and chose the minimum probability $p_k = \frac{1}{d^2}$, $M = 36 d^2 \log (2 d / \delta)$. This ensures that with probability at least $1-\delta$, 
\begin{align*}
    \mathbb{P}\left[\max_{x \in \mathcal{X}} |p(x) - \hat{p}(x)| \leq p(x)/4 \right]. 
\end{align*}
Condition on the above event in the remainder of this argument. Note that under the above event, we have that for all $x > 1$, 
\begin{align*}
    \hat{p}(1) \geq \frac{3}{4} \left( 1 - \frac{d-1}{d^2}\right) > \frac{5}{4} \cdot \frac{1}{d^2} \geq \hat{p}(x). 
\end{align*}

Now, consider the SDP problem we solve in Algorithm 1 if we use $\hat{p}$ in place of $p$. Note that every context $(x,a)$ where $a \neq 1$ has the feature vector $\phi(x,a) = e_1$, and $\Sigma_w$ depends \emph{only} on the mass that $w$ places on each $e_i$ for $i \in [d]$. Thus, without loss of generality, we may assume that under $w$, the feature $e_1$ is only accessed at the context-action pair $(x=1, a=1)$. That is, we may assume without loss of generality that $w(x,a)$ is only supported at $a=1$. Consequently, we have
\begin{align*}
    \Sigma_w &= \sum_{x,a}w(x,a)\phi(x,a)\phi(x,a)^\top= \sum_{i=1}^d z_i e_ie_i^\top
\end{align*}
where $z_i \defeq w(i,1)$ for each $i \in [d]$. Then
\begin{align*}
    \Sigma_w^{-1}&=\sum_{i=1}^d z_i^{-1} e_ie_i^\top, 
\end{align*}
and consequently,
\begin{align*}
    \max_{a}\|\phi(x,a)\|_{\Sigma_w^{-1}}^2 = \max_{a}\phi(x,a)^\top \Sigma_w^{-1}\phi(x,a) =  \sum_{i=1}^d \max_a z_i^{-1} \normInline{\phi(x,a)}_2^2.
\end{align*}
Recall the definition that $\phi(x, a)= e_x$ if $a = 1$ and $e_1$ otherwise. Hence, the max becomes $\min(z_1, z_x)^{-1}$. Plugging this in the formulation of SDP gives 
\[\E_{x \sim \hat{p}} \max_{a \in \cA} \normInline{\phi(x,a)}_{\Sigma_w^{-1}}^2=\sum_x \hat{p}(x)\cdot \frac{1}{\min(z_1, z_x)}.\]

Next, we prove that $z_1 \geq z_i$ for all $i \neq 1$. 
Indeed, suppose for the sake of contradiction that $z_1 < z_i$ for some $i \neq 1$. Then consider $z'\in\Delta^{\mathcal{X}}$ such that $z'_x=\frac{1}{2}(z_1+z_i)$ if $x=1$ or $x=i$, and $z'_x=z_x$ for $x\neq 1$ and $x\neq i$. Note that $z'_1=z'_i>z_1$ by this construction. Let $w$ and $w'$ be the weights corresponding to $z$ and $z'$. Then 
\begin{align*}
    \E_{x \sim \hat{p}} \max_{a \in \cA} \normInline{\phi(x,a)}_{\Sigma_w^{-1}}^2&=\frac{\hat{p}(1)}{z_1}+\frac{\hat{p}(i)}{z_1}+\sum_{x\neq 1,x\neq i} \hat{p}(x)\cdot \frac{1}{\min(z_1, z_x)}\\
    &> \frac{\hat{p}(1)}{\frac{1}{2}(z_1+z_i)}+\frac{\hat{p}(i)}{\frac{1}{2}(z_1+z_i)}+\sum_{x\neq 1,x\neq i} \hat{p}(x)\cdot \frac{1}{\min(z_1, z_x)}\\
    &= \frac{\hat{p}(1)}{z'_1}+\frac{\hat{p}(i)}{z'_i}+\sum_{x\neq 1,x\neq i} \hat{p}(x)\cdot \frac{1}{\min(z'_1, z'_x)}\\
    &=\E_{x \sim \hat{p}} \max_{a \in \cA} \normInline{\phi(x,a)}_{\Sigma_{w'}^{-1}}^2.
\end{align*}
Therefore, such $w$ (and thus such $z$) cannot be a minimizer of the expectation. Hence, the SDP reduces to 
\begin{align*}
    \min_{z \in \Delta^{n}} \sum_{x \in \mathcal{X}} \hat{p}(x) \cdot \frac{1}{z_x}. 
\end{align*}
Using Lagrange multipliers we get the closed form solution, 
$z_x = \frac{\sqrt{ \hat{p}(x)}}{\sum_x \sqrt{ \hat{p}(x)} }$, and optimal value 
\begin{align*}
\paren{\sum_{x \in \mathcal{X}} \sqrt{\hat{p}(x)}}^2. 
\end{align*}
Moreover, using the guarantee above on the accuracy of the learned $\hat{p}(x)$ parameters given $M$ prior contexts, we prove 
\begin{align*}
\paren{\sum_{x \in \mathcal{X}} \sqrt{\hat{p}(x)}}^2 \leq \paren{\sum_{x \in \mathcal{X}} \sqrt{1.25 \cdot p(x)}}^2 = \paren{\sqrt{1.25 \cdot \frac{d^2-d+1}{d^2}} + 1.25 \cdot \sum_{x =2}^d \frac{1}{d}  }^2 = O(1).
\end{align*}
Therefore, given $M = 36d^2 \log(2d/\delta)$ samples, with high probability the active context sample algorithm using the empirically-derived context probabilities will also improve over passive context sampling by a $\sqrt{d}$ factor. This shows that at least in some settings, using the empirical distribution to estimate the context distribution will yield the same $\sqrt{d}$ improvement in the resulting regret bounds, compared to passive context sampling methods. 

\subsection{Using the G-optimal distribution directly is insufficient}\label{sec:insufficient}

Another, perhaps related, question is the following: 

\begin{center}
    \emph{When constructing $w$, would using G-optimal design $(w \gets {q}^\star)$ \\ obtain a similarly strong rate as our Active-SCLB?
    }
\end{center}

In this subsection, we provide a \emph{negative} answer to this question. Indeed, Theorem~\ref{thm:kiefer-wolofitz} shows that the G-optimal design $\hat{q}$ satisfies $\max_{(x,a)} \Vert \phi(x,a) \Vert_{\Sigma_{\hat{q}}^{-1}}^2 = d$. By Lemma~\ref{lemma:simple-regret-bound-general}, this in general yields a bound of $\sqrt{d\beta/T}$. This gives \emph{no improvement} over passive context sampling. 

As an illustrative example, consider again the bandit instance from Definition~\ref{def:hard-instance}. A valid $G$-optimal design is the one that selects $(x,1)$ with probability $1/d$ for each $x \in [d]$. To see why, note that the $G$-optimal design tries to minimize $\max_{x,a} \normInline{\phi(x,a)}_{\Sigma_q^{-1}}$ over all $q \in \Delta^{\cX \times \cA}$. In this setting, because each $\phi(x,a)$ is a standard basis vector, by symmetry, the $G$-optimal design must place equal mass on each of the $d$ features. One way to achieve this is to select the context-action pair $(x,1)$ with probability $1/d$ for each $x \in [d]$.

The $G$-optimal design upsamples the rare contexts too much, causing a $d$-dependence in the bound. Indeed, we find that
\begin{align*}
    \mathbb{E}_{x \sim p} \Vert e_x \Vert_{\Sigma_{{q}}^{-1}}^2 =  \left(\frac{d^2-d+1}{d^2}\right) \cdot d + \sum_{i = 2}^d \frac{1}{d^2} \cdot d = \Omega(d),
\end{align*}
and hence, Lemma~\ref{lemma:simple-regret-bound-general} suggests that even on this simple bandit instance, if one were to modify  Algorithm~\ref{alg:active-lcb} to use $w \gets \hat{q}$ in Line~\ref{line:approximate}, then one would again obtain a worse $d$-dependence than using the active context sampling distribution $w = \alpha \hat{q} + (1-\alpha) \hat{w}$ as in our original Algorithm~\ref{alg:active-lcb}. 

\subsection{Our framework naturally enables certain nonlinear reward models} 

Our method is immediately amenable to some non-linear reward models in the sense that one could always lift the feature vector into a higher-dimensional space which captures non-linear relationships between the original feature entries. To briefly explain why this is the case, recall that one can always use linear regression algorithms to fit a quadratic regression model  just by modifying the feature vector to include the transformed variable $x^2$ in addition to $x$ as a feature. What is important is just that the reward is linear (or approximately so) in the lifted space. See, for example, the empirical section of \citet{kong2020sublinear}. 

Of course, the difficulty with capturing \emph{arbitrariy complex} nonlinear reward models is that in order to represent a complex family of reward models, we may need to lift the features into a \emph{very high dimensional space} (increase $d$) and such a transofmation is associated with higher sample complexity and computational costs.  

To circumvent this curse of dimensionality, in principle, the idea of active context sampling could also apply to kernelized contextual bandits. However, the main challenge is that in these more complex reward models, it is not clear whether the optimization problem to solve for the optimal sampling distribution is computationally tractable (e.g., solvable in polynomial time.) Some prior works have discussed greedily sampling contexts for the kernelized settings, but this does not result in tight guarantees that avoid a dimension dependence (see the final paragraph of Section~\ref{sec:setup}.) Thus, while developing theoretically-grounded techniques for very general reward models might be challenging, we hope that our work provides potentially useful insights towards tackling more sophisticated reward models as future work. Moreover, as discussed in Section~\ref{sec:conclusion}, our empirical results on real-world data seem robust to model misspecification, which indicates that our methods could outperform their theoretical guarantees even if the true reward model is not linear. 

 \subsection{Pure exploration setting and choice of simple regret as the objective function}\label{sec:regret-setting}

 In this section, we discuss the motivation for our pure exploration setting as well as our choice of regret function. 
 
 \paragraph{Pure exploration setting vs cumulative regret minimization.} Recall that our work focuses on the pure exploration setting, where a learner first collects some reward observations during an exploration phase and then deploys a learned policy. The quality of the policy is the expected regret \emph{during deployment}; the learner is not penalized for regret incurred during the exploration phase \citep{krishnamurthy2023proportional, zanette2021design, li2022instance, deshmukh2018simple}. 
 
 However, we acknowledge that if one's goal is to minimize cumulative regret during exploration, the techniques in our paper may not be ideal and alternative algorithms may be preferable, such as the LinUCB algorithm of \citet{abbasi2011improved}. This is because, similar to \citet{li2022instance, zanette2021design}, we do not provide bounds on the \emph{cumulative} regret.  

 The reason we focus on pure exploration is that it is better aligned with some important settings where a fixed experimental period is common. For example:
 \begin{itemize}[leftmargin=*, nosep]
     \item Social media companies often run experimental studies on users. A company might be willing to test out several different strategies (knowing that it may temporarily lose on engagement) in order to quickly learn a high-quality strategy to deploy in production long-term.
     \item In medical settings, one might design a monitored clinical trial in which patients are actively monitored in case of any adverse consequences to the administered treatment. Such monitoring is not possible long-term or at-scale, however, the willingness to take a measured risk to quickly explore new treatments during a short-term trial can be valuable for eventually identifying good treatments to prescribe to the general public.
     \item Another natural setting where the pure exploration setting is well-suited is applications where we have the ability to simulate the reward of a particular context-action pair by running an (expensive) simulation. This may arise in scientific decision making where physical or medical simulations are available. In such cases, one is unconcerned with the regret during exploration.
 \end{itemize}

To summarize, during the exploration period, our regret may be high, however, this freedom to explore actions freely enables us to quickly (i.e., with few samples) converge upon a near-optimal policy which can be deployed at scale after the exploration stage. In contrast, if one tries to balance exploration and exploitation as in the online bandit setting, convergence to an optimal policy could be slower. This is precisely what motivates the extensive prior work on the pure exploration model \cite{zanette2021design, deshmukh2018simple, krishnamurthy2023proportional, li2022instance}. 

As a very important related note, note that there is a sense in which active context sampling \emph{only} makes sense in the pure exploration setting and does not make sense in the cumulative regret minimization setting. As a thought experiment, suppose we modify the cumulative regret minimization setting to allow the learner to actively select the context at each round $t$. Then, there is a trivial way to achieve low cumulative regret: the algorithm could just select the \emph{same} context $x'$ in every single round and learn an optimal action for context $x'$. This would have low cumulative regret, since the learner only needs to learn a good action in context $x'$. However, evidently, this is not a very interesting setting, since the learner would never learn to perform well on the real-world distribution $p$. Thus, active context sampling is, in a sense, trivial in the cumulative regret setting. 

In contrast, our work allows the learner to actively sample contexts freely during the exploration phase; however, it must learn a policy that will perform well in the eventual real-world distribution $p$. This is diagrammed in Figure~\ref{fig:learning-paradigms}. 

\paragraph{Simple regret vs best arm identification.}

In PAC-learning for best-arm identification, the goal is to find the best arm in every context \cite{soare2014best}. That is, the goal would be to identify an approximately optimal action in \emph{every} context.

In contrast, in simple regret, we are okay with using a suboptimal arm in some rare contexts as long as the learned policy performs well \emph{on average} over the distribution of contexts. In this sense, best-arm-identification is a very strict “solution concept” which requires strong performance uniformly over every context. Meanwhile, simple regret is more relaxed and often more realistic in that it only asks for learning context-to-action policy which performs well (on average) over the context distribution. In practical scenarios, we are often satisfied with a policy which works very well on average (even if it is suboptimal in some rare contexts). However, to truly find the best arm in every context as in best-arm identification, we might need to spend many more samples.

This is why prior works have also looked at simple regret under the SCLB model to better capture real-world scenarios where average-performance of a policy is the key performance indicator \cite{zanette2021design, deshmukh2018simple, krishnamurthy2023proportional, li2022instance}. 

%% file: appendix/compare_settings.tex
\newpage

\section{Comparison of active learning settings}\label{sec:comparing_settings}

In this section, we include a helpful visualization (Figure~\ref{fig:learning-paradigms}) to compare linear bandits, active learning (e.g., for regression), passive learning for SCLBs, and active learning for SCLBs. 

\begin{figure}[H]
  \centering
  \begin{subfigure}[b]{\textwidth}
    \centering
    \includegraphics[width=\linewidth]{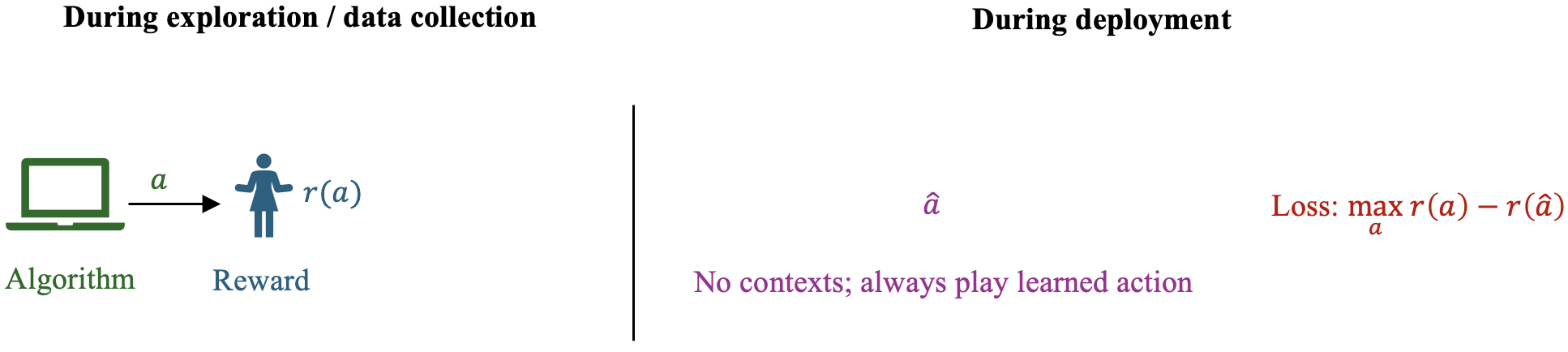}
    \caption{Linear bandits}
    \label{fig:bandit}
  \end{subfigure}
  \vspace{.5em}
  \begin{subfigure}[b]{\textwidth}
    \centering
    \includegraphics[width=\linewidth]{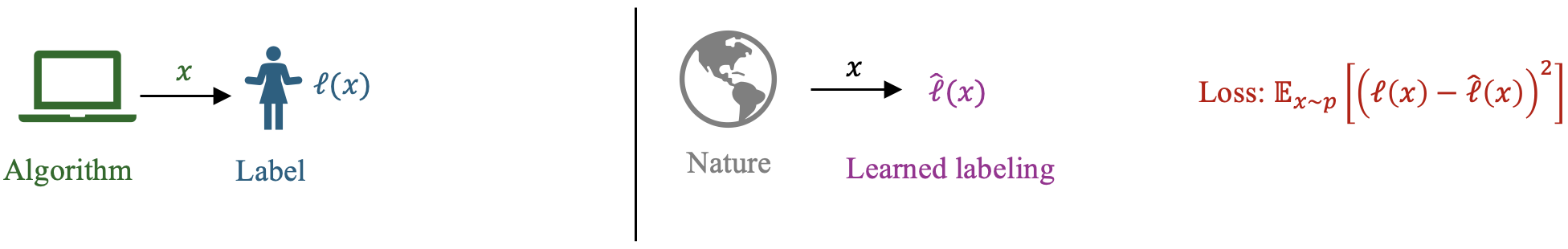}
    \caption{Active learning (e.g., for regression)}
    \label{fig:active}
  \end{subfigure}
  \vspace{.5em}
  \begin{subfigure}[b]{\textwidth}
    \centering
    \includegraphics[width=\linewidth]{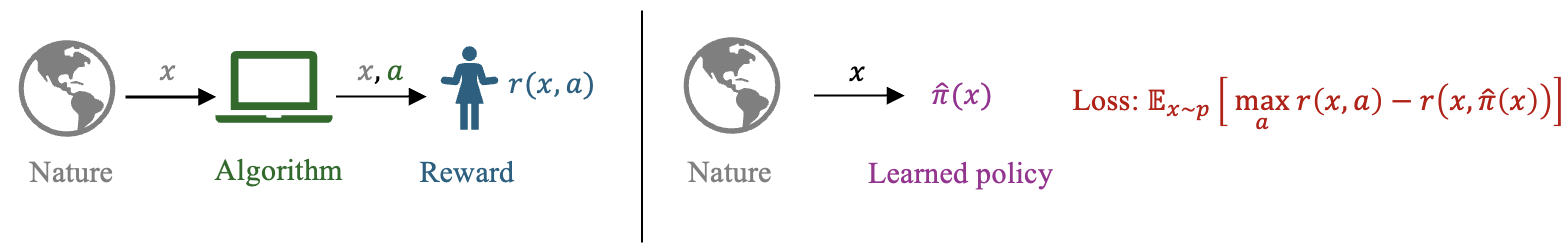}
    \caption{Passive context sampling for SCLBs 
    }
    \label{fig:contextual-passive}
  \end{subfigure}
  \vspace{.5em}
  \begin{subfigure}[b]{\textwidth}
    \centering
    \includegraphics[width=\linewidth]{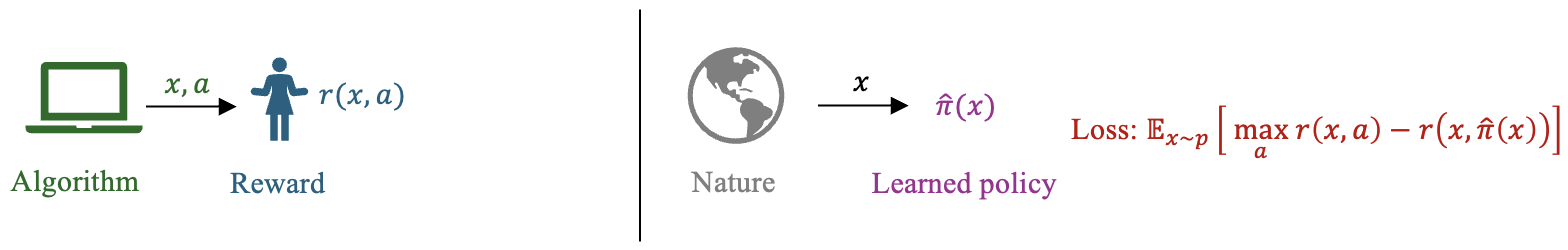}
    \caption{Active context sampling for SCLBs }
    \label{fig:contextual-active}
  \end{subfigure}

  \caption{Four learning paradigms. Figure~\ref{fig:bandit} describes the typical linear bandit setting, where there are no \emph{contexts} and the learner's goal is to learn a good universal action. Figure~\ref{fig:active} shows active learning for regression, with the mean-squared loss. During exploration, the learner actively samples data $x$ to learn a label function $\hat{\ell}(x)$; but during deployment, inputs arrive according to the distribution $p$. In other settings, such as active learning for classification, this mean-squared loss might be replaced with another \emph{smooth, convex} loss function. In contrast, in contextual bandits (Figures~\ref{fig:contextual-passive} and ~\ref{fig:contextual-active}), the learner needs to learn a policy-to-action mapping. In this case, the loss function function is the sub-optimality of the policy---which is discontinuous---and consequently, traditional techniques from continuous optimization do not immediately apply.}
  \label{fig:learning-paradigms}
\end{figure}

%% file: appendix/additional_theory.tex
\newpage
\section{Additional theoretical results}\label{apx:additional_results}

In this appendix, we expand on additional results which were omitted from the main body. 

\subsection{Lower bound in active context sampling setting}\label{sec:lower-bound-discussion}

In this sectionl, we discuss why the minimax optimal lower bound remains valid in our setting, despite our use of active context sampling. Indeed, the claimed lower bound actually follows directly from lower bounds in Gaussian linear bandits because SCLBs reduce to Gaussian linear bandits in the case of a single context. Moreover, in the case of a single context, the knowledge of the context distribution $p$ and ability to actively sample contexts adds no new statistical power, so we inherit lower bounds from Gaussian linear bandits. 

As in \citep{zanette2021design}, we consider the large and small context-action regime separately. First take the large context-action space regime where the second term in $\beta$ \eqref{eq:beta-lambda} dominates. Theorem 24.1 in Lattimore and Szepesvari (as cited in our paper) guarantees the regret must be at least on the order of $d / \sqrt{T}$ and therefore our result is minimax optimal, up to polylog factors. 

Consider the small context-action space regime, where the first term in $\beta$ dominates and $\sqrt{\beta}=  \tilde{O}(\sqrt{\log(|\mathcal{X}|\mathcal{|A|}})$ and our resulting regret bound in Theorem 4.3 is, up to polylog factors, $\sqrt{d/T}$. From Theorem 2 in \citep{chu2011contextual} the regret for $d^2 \leq T$ is at least of the order $\sqrt{d/T}$. Thus our result is minimax optimal, up to polylog factors. 

\subsection{Probably-approximately-correct guarantee}

Here, we state a probably-approximately-correct (PAC)-style version of our main result (Theorem~\ref{theorem:main}). The following theorem is an immediate corollary of Theorem~\ref{theorem:main}, and makes an assumption on $\thetastar, \lambda$ in order to more directly compare to the PAC-learning versions reported in \cite{zanette2021design, abbasi2011improved}. 

\begin{corollary}[PAC-learning version of Theorem~\ref{theorem:main}]\label{corr:pac} Let $\cB = (\cX, \cA, \phi, p, \nu, \thetastar)$ be an SCLB and $\epsilon, \delta \in (0,1)$. Assume $\normInline{\thetastar} \leq \Tilde{O}(1)$ and $\lambda \leq \Tilde{O}(1)$. When initialized with a sample budget of $T \geq T_1 = \Tilde{O}(\cC_{\cB} \cdot d\epsilon^{-2} + d^2)$ and $\alpha = 1/2$, Algorithm~\ref{alg:active-lcb} returns a policy $\hat{\pi}$ such that with probability $1-\delta$, $R(\hat{\pi}) \leq \epsilon$. 
\end{corollary} 
\begin{proof} From the definition of $\beta$ \eqref{eq:beta-lambda}, it is clear that so long as $\lambda, \normInline{\thetastar} \leq \tilde{O}(1)$, $\beta \leq \tilde{O}(d)$. The corollary now follows immediately Theorem~\ref{theorem:main}. 
\end{proof}

In comparison, the methods from \citet{zanette2021design} and \citet{abbasi2011improved} (RFLinUCB and Planner-Sampler) require $\Tilde{O}(d^2\epsilon^{-2})$ samples to achieve $\epsilon$ regret under the same assumptions. 

Recalling that $\cC_{\cB}\leq d$, we see that our sample-complexity guarantee of $\Tilde{O}(\cC_{\cB} d\epsilon^{-2})$ is always \emph{at least} as strong as that of Planner-Sampler and RFLinUCB \citet{zanette2021design, abbasi2011improved} and may be as low as $\tilde{O}(d + d/\epsilon^2)$ when $\cC_{\cB} = O(1)$ (recall Section~\ref{sec:comparison} which gives an example where this occurs.) Thus, in the PAC-learning setting, we improve over Planner-Sampler and RFLinUCB by up to a dimension factor. 

\begin{remark}Note that our restriction to $\epsilon \in (0, 1)$ is without loss of generality. Because we assume that rewards are bounded between $[0,1]$ in expectation, if $\epsilon > 1$, then \emph{any} policy has regret at most $\epsilon$, and consequently, the problem is trivial. 
\end{remark}

\subsection{Instance-dependent guarantees and comparison to \cite{li2022instance}}

In this section, we compare in more detail against the instance-dependent rates of \cite{li2022instance} for SCLBs. We did not discuss this in detail in the main body, because the instance-dependent sample complexity rates of \cite{li2022instance} for SCLBs are obtained using a computationally intractable algorithm (see, e.g., discussion around Theorem 2.14 of and conclusion of \cite{li2022instance}.) In contrast, our goal in this work is to focus on practical applications where active context sampling may be helpful; hence, our goal was to design polynomial-time implementable algorithms which obtain instance-dependent rates for SCLBs. Because  it is not known how to implement or even approximately implement the guarantees of \cite{li2022instance} in polynomial time, note that the result of Theorem 2.14 in \cite{li2022instance} may be an unfair comparison to our result Theorem~\ref{theorem:main}.

Correspondingly, our goal of this section, is to discuss the result of \cite{li2022instance} in greater detail and show that their rates can also be improved with active context learning---if we disregard the concerns over computational tractability. In Section~\ref{sec:result}, we first state the main result of \cite{li2022instance}, which uses passive context sampling to design an instance-dependent algorithm (ContextualRAGE) for SCLBs. In Section~\ref{sec:extend}, we will show that using active context sampling, we can generalize (ContextualRAGE) to design a new active context sampling algorithm, which we call Active-ContextualRAGE. Finally, in Section~\ref{sec:separation} we show that our family of SCLBs from Definition~\ref{def:hard-instance} remains an instance where Active-ContextualRAGE outperforms ContextualRage by a $\sqrt{d}$ factor in its regret bound (which corresponds to a $d$ factor in the PAC-learning sample complexity). 

The main takeaway for the results in this section is (1) theoretically, we can prove that active context sampling improves  over the rates of \citet{li2022instance}, however, the algorithm achieving this improved rate is not computationally intractable; and (2) even if we consider computationally intractable algorithms as in \citep{li2022instance}, active context sampling improves over the rates achieved by passive context sampling on the family of SCLBs proposed in Section~\ref{sec:comparison}. 

\subsubsection{Restating the main result of \cite{li2022instance} for SCLBs}\label{sec:result}

To aid in the statement of the main result of \cite{li2022instance}, we first introduce some additional notation. For any policy $\pi \in \Pi$, we denote $\phi_\pi \defeq \E_{x \sim p} \phi(x, \pi(x))$ . 

\begin{theorem}[Theorem 2.14 of \cite{li2022instance}, restated]\label{thm:instance-optimal} Let $\cB = (\cX, \cA, \phi, p, \nu, \thetastar)$ be an SCLB, $\normInline{\thetastar} \leq 1$, and $\epsilon, \delta \in (0,1)$. Define 
\begin{align*}
    \cF &\defeq \left\{w \in \Delta^{\cX \times \cA} : \forall x\in\cX, \sum_{a\in\cA} w(x,a) = p(x)\right\}.
\end{align*}
Let $\pi^\star: \cX \to \cA$ be the optimal policy given by 
\begin{align}\label{eq:optimal-policy}
    \pi^\star(x) \defeq \argmax_{a \in \cA} \phi(x,a)^\top \thetastar. 
\end{align}
Define
\begin{align*}
    \rho_{\cB, \epsilon}^{\mathrm{pas}} \defeq \min_{w \in \cF} \max_{\pi: \pi \neq \pi^\star} \frac{\normInline{\phi_\pi - \phi_\pi^\star}_{\Sigma_w^{-1}}^2}{\max(\epsilon, \langle \phi_\pi- \phi_{\pi^\star}, \thetastar \rangle)^2}. 
\end{align*}
Moroever, let 
\begin{align}\label{eq:Delta}
    \Delta_{\epsilon} = \max\paren{\epsilon, \min_{\pi \in \Pi : \pi \neq \pi^\star} \langle \phi_{\pi^\star} - \phi_{\pi} , \thetastar \rangle}. 
\end{align}
Then, with probability at least $1-\delta$, ContextualRAGE (Algorithm 1 of \cite{li2022instance}) returns a policy $\hat{\pi}$ such that $R(\hat{\pi}) \leq \epsilon$  after making at most
\begin{align}\label{eq:rho}
    O\paren{\rho_{\cB, \epsilon}^{\mathrm{pas}} \log(\min\{d\log(1/\epsilon),\log|\Pi|\}+\log(1/\delta))\log(\Delta_\epsilon^{-1})}
\end{align}
reward observations. Moreover, it is \emph{always} the case that \eqref{eq:rho} is upper bounded by $\tilde{O}(d^2/\epsilon^2).$
\end{theorem}

This result gives a fine-grained PAC-learning guarantee, which may be much stronger than the minimax-rate of $\tilde{O}(d^2/\epsilon^2)$ when $\rho_{\cB, \epsilon}^{\mathrm{pas}} \ll d^2/\epsilon^2$. 

However, as discussed in Section 3.3 and Section 4 of \citet{li2022instance}, the ContextualRAGE algorithm presented in \citep{li2022instance} is computationally ineffcient because it requires maintaining a set of policies $\Pi_\ell$ from round to round. Because $|\Pi| = \cX^{\cA}$, even maintaining $\Pi_1$ at the start of the algorithm requires exponential-time. Thus, while Theorem~\ref{thm:instance-optimal} is very interesting information theoretically, it does not directly lend itself well to practical applications of contextual bandits.  

Nonetheless, in the following section, we show that if we disregard computational implementability, we can obtain a similar result to our Theorem~\ref{thm:instance-optimal}: in particular, we present a new active-learning variant of \cite{li2022instance}'s ContextualRAGE that achieves an \emph{even tighter} instance-dependent rate than Theorem~\ref{thm:instance-optimal}.

\subsubsection{Developing tighter instance-dependent rates using active context sampling}\label{sec:extend}

Here, we propose an algorithm that takes advantage of active context sampling to achieve a better instance-dependent PAC learning guarantee than the passive instance-dependent sample complexity bound stated in \cite{li2022instance}. The pseudocode is shown in Active-ContextualRAGE (Algorithm~\ref{alg:rage}). As with the ContextualRAGE algorithm (Theorem~\ref{thm:instance-optimal}), Active-ContextualRAGE is computationally infeasible for the same reason that ContextualRAGE is infeasible: it requires explicitly maintaining a set of policies which could be as large as $\Pi$ in every iteration.

\begin{algorithm2e}[!htb]
    \centering
    \caption{Active-ContextualRAGE}
    \begin{algorithmic}[1]\label{alg:rage}
    \REQUIRE  $\phi : \mathcal{X} \times \mathcal{A} \rightarrow \R^d$, $\delta\in (0,1)$
    \STATE \textbf{Initialize} $\Pi_{1}=\Pi$
    \FOR{$\ell=1,2,\cdots,\lceil\log_2(1/\epsilon)\rceil$}
    \STATE $\epsilon_\ell:=2^{-\ell}, \delta_{\ell} := \delta/(2\ell^2|\Pi|)$
    \STATE Let $n_{\ell}$ be the minimum value s.t.:
    \[\hspace{-10pt}\min_{w \in \triangle_{\mathcal{X} \times \mathcal{A}}}\max_{\pi,\pi' \in \Pi_{\ell}} \frac{\| \E_{x\sim p}[\phi(x,\pi(x))-\phi(x,\pi'(x))]\|_{\Sigma_w^{-1}}^2\log(1/\delta_{\ell})}{n_{\ell}} \leq \epsilon_{\ell}^2\]
    with argmin given by $w^{(\ell)}$.
    \STATE 
    For each $t\in [n_\ell]$, pull $(c_t,a_t)\sim w^{(\ell)}$, observe reward $r_t$ 
    \STATE Compute $O_t=\Sigma_{w^{(\ell)}}^{-1} \phi(c_t,a_t) r_t$
    \STATE For  $\pi,\pi' \in \Pi_\ell$  \[\hat{\Delta}_{\ell}(\pi,\pi') = {\sf Cat}( \{ \langle\E_{x\sim p}[\phi(x,\pi(x))-\phi(x,\pi'(x))],O_{i}\rangle \}_{i=1}^{n_\ell})\]
    \STATE 
    Update
    \[
    {\Pi}_{\ell+1}={\Pi}_\ell \setminus \{\pi' \in {\Pi}_l\mid \max_{\pi \in {\Pi}_\ell} \hat{\Delta}_{\ell}(\pi,\pi')>\epsilon_\ell\}\]
    \ENDFOR
    \RETURN $\Pi_{\ell+1}$
    \end{algorithmic}
\end{algorithm2e}

We first state a lemma that guarantees the optimal policy is inside the candidate policy set $\Pi_\ell$ after elimination, and all policies in $\Pi_\ell$ have small gaps. The following lemma is the same as Lemma 3.1 in \cite{li2022instance} and the proof follows the identical argument as in \cite{li2022instance}. 

\begin{lemma}[Lemma 3.1 of \cite{li2022instance}, restated]\label{lem:rage_correctness} In the execution of Algorithm~\ref{alg:rage}, with probability at least $1-\delta$, for all $\ell > 1$, 
$\pi_* \in \Pi_\ell$ and $\max_{\pi \in \Pi_\ell} V(\pi^*)-V(\pi) \leq 4 \epsilon_\ell$.
\end{lemma}

Next, we state a theorem that gives a PAC upper bound for Algorithm~\ref{alg:rage} using active context sampling. Note that the upper bound $\tilde{O}(d^2/\epsilon^2)$ in this theorem is quite loose, and we show in Section~\ref{sec:separation} that the active complexity bound can be much smaller than the passive one. 

\begin{theorem}[Active-ContextualRAGE]\label{thm:active-contextual_rage}
Let $\epsilon, \delta \in (0,1)$, $\phi: \mathcal{X} \times \mathcal{A} \rightarrow \R^d$, $\normInline{\thetastar} \leq 1$, and let $\Delta_\epsilon, \pi^\star$ be as in \eqref{eq:optimal-policy} and \eqref{eq:Delta}, respectively. Define 
\begin{align*}
    \rho^{\mathrm{act}}_{\cB, \epsilon} \defeq \min_{w \in \Delta^{\cX \times \cA}} \max_{\pi: \pi \neq \pi^\star} \frac{\normInline{\phi_\pi - \phi_\pi^\star}_{\Sigma_w^{-1}}^2}{\max(\epsilon, \langle \phi_\pi- \phi_{\pi^\star}, \thetastar \rangle)^2}. 
\end{align*}

Then, with probability at least $1-\delta$, Active-ContextualRAGE (Algorithm~\ref{alg:rage}) returns a policy $\hat{\pi}$ such that $R(\hat{\pi}) \leq \epsilon$  after making at most
\begin{align}\label{eq:varrho}
    O\paren{\rho^{\mathrm{act}}_{\cB, \epsilon} (\min\{d\log(1/\epsilon),\log|\Pi|\}+\log(1/\delta))\log(\Delta_\epsilon^{-1})}
\end{align}
reward observations,
which is itself upper bounded by $\tilde{O}(d^2/\epsilon^2)$.

\end{theorem}

\begin{proof} For notational convenience, for any $\pi \in \Pi$, let
\begin{align*}
    V(\pi) = \E_{x \sim p} \phi(x, \pi(x))^\top \thetastar. 
\end{align*}
Define $S_\ell = \{ \pi \in \Pi : V(\pi^*)-V(\pi) \leq 4 \epsilon_\ell \}$.
Lemma~\ref{lem:rage_correctness} implies that with probability at least $1-\delta$ we have $\bigcap_{\ell=1}^\infty \{ \Pi_\ell \subseteq S_\ell \}$.
Observe that if for any $\mathcal{V} \subset \Pi$ we define $$\rho(w^{(\ell)},\mathcal{V}):=\min _{w \in \Delta_{\mathcal{X} \times \mathcal{A}}} \max _{\pi,\pi'\in \Pi_\ell} \left\|\E_{x\sim p}[\phi(x,\pi(x))-\phi(x,\pi'(x))]\right\|_{\Sigma_w^{-1}}^2,$$ then
\begin{align*}
\rho(w^{(\ell)}, \Pi_\ell) &= \min_{w \in \Delta_{\mathcal{X} \times \mathcal{A}}} \max_{\pi,\pi' \in \Pi_{\ell}} \|\E_{x\sim p}[\phi(x,\pi(x))-\phi(x,\pi'(x))]\|_{\Sigma_w^{-1}}^2 \\
&\leq \min_{w \in \triangle_{\mathcal{X}\times\mathcal{A}}} \max_{\pi,\pi' \in S_{\ell}} \|\E_{x\sim p}[\phi(x,\pi(x))-\phi(x,\pi'(x))]\|_{\Sigma_w^{-1}}^2 =: \rho(S_\ell).
\end{align*}
By line 4 of Algorithm~\ref{alg:rage}, we know that for each $\ell$, $n_\ell=\lceil\rho(w^{(\ell)}, \Pi_\ell)\log(1/\delta_\ell)\epsilon_\ell^{-2}\rceil$. Also, for $\ell \geq \lceil \log_2(4\Delta_\epsilon^{-1}) \rceil$ we have for all $\pi\in S_\ell$, $V(\pi^*)-V(\pi)\leq\epsilon$ (by Lemma~\ref{lem:rage_correctness} and Definition of $S_\ell$ and $\epsilon_\ell$), thus the sample complexity to identify any $\pi$ in $S_\ell$ is 
\begin{align*}
\sum_{\ell=1}^{\lceil \log_2(4\Delta_\epsilon^{-1}) \rceil} n_\ell 
&= \sum_{\ell=1}^{\lceil \log_2(4\Delta_\epsilon^{-1}) \rceil} \lceil 4 \epsilon_\ell^{-2} \rho( w^{(\ell)}, \Pi_\ell ) \log(2 \ell^2|\Pi|/\delta) \rceil \\
&\leq \sum_{\ell=1}^{\lceil \log_2(4\Delta_\epsilon^{-1}) \rceil} 4 \epsilon_\ell^{-2} \rho(S_\ell) \log(2 \ell^2 |\Pi| /\delta) +1\\
&\leq c \log( \log(\Delta_\epsilon^{-1}) |\Pi| /\delta) \sum_{\ell=1}^{\lceil \log_2(4\Delta_\epsilon^{-1}) \rceil}  \epsilon_\ell^{-2} \rho(S_\ell) 
\end{align*}
for some absolute constant $c > 0$.
We now note that
\begin{align*}
&\min _{w \in \Delta_{\mathcal{X} \times \mathcal{A}}} \max _{\pi\in \Pi\setminus\pi^*} \frac{\left\|\E_{x\sim p}[\phi(x,\pi(x))-\phi(x,\pi^*(x))]\right\|_{\Sigma_w^{-1}}^2}{\left\langle\E_{x\sim p}[\phi(x,\pi^*(x))-\phi(x,\pi(x))],\theta^*\right\rangle^2\vee \epsilon^2}\\
&= \min_{w \in\Delta_{\mathcal{X}\times\mathcal{A}}}  \max_{\ell \leq \lceil \log_2(4\Delta_\epsilon^{-1})\rceil} \max_{\pi \in S_\ell}\frac{\left\|\E_{x\sim p}[\phi(x,\pi(x))-\phi(x,\pi^*(x))]\right\|_{\Sigma_w^{-1}}^2}{\left\langle\E_{x\sim p}[\phi(x,\pi^*(x))-\phi(x,\pi(x))],\theta^*\right\rangle^2\vee \epsilon^2} \\
&\geq \frac{1}{\lceil \log_2(4\Delta_\epsilon^{-1})\rceil} \min_{w\in\triangle_{\mathcal{X}\times\mathcal{A}}} \sum_{\ell=1}^{\lceil \log_2(4\Delta_\epsilon^{-1}) \rceil} \max_{\pi \in S_\ell}\frac{\left\|\E_{x\sim p}[\phi(x,\pi(x))-\phi(x,\pi^*(x))]\right\|_{\Sigma_w^{-1}}^2}{\left\langle\E_{x\sim p}[\phi(x,\pi^*(x))-\phi(x,\pi(x))],\theta^*\right\rangle^2\vee \epsilon^2}\tag{max lower bounded by average} \\
&\geq \frac{1}{\lceil \log_2(4\Delta_\epsilon^{-1})\rceil}  \sum_{\ell=1}^{\lceil \log_2(4\Delta_\epsilon^{-1}) \rceil} (4\epsilon_\ell)^{-2} \min_{w\in\triangle_{\mathcal{X}\times\mathcal{A}}} \max_{\pi \in S_\ell}\|\E_{x\sim p}[\phi(x,\pi(x))-\phi(x,\pi^*(x))]\|_{\Sigma_w^{-1}}^2 \tag{$\pi\in S_\ell$ implies that gap less than $\epsilon_\ell$ and $4\epsilon_\ell\geq \epsilon$}\\
&\geq \frac{1}{64 \lceil \log_2(4\Delta_\epsilon^{-1})\rceil}  \sum_{\ell=1}^{\lceil \log_2(4\Delta_\epsilon^{-1}) \rceil} \epsilon_\ell^{-2} \min_{w\in\triangle_{\mathcal{X}\times\mathcal{A}}} \max_{\pi,\pi' \in S_\ell} \|\E_{x\sim p}[\phi(x,\pi(x))-\phi(x,\pi'(x))]\|_{\Sigma_w^{-1}}^2 \\
&= \frac{1}{64 \lceil \log_2(4\Delta_\epsilon^{-1})\rceil}  \sum_{\ell=1}^{\lceil \log_2(4\Delta_\epsilon^{-1}) \rceil} \epsilon_\ell^{-2} \rho(S_\ell) 
\end{align*}
where for the last inequality, we have used the fact that for any $\pi,\pi' \in S_\ell$, 
\begin{align*}
    &\|\E_{x\sim p}[\phi(x,\pi(x))-\phi(x,\pi'(x))]\|_{\Sigma_w^{-1}}^2=\|\phi_\pi - \phi_{\pi'}\|_{\Sigma_w^{-1}}^2\\
    &=(\phi_\pi - \phi_{\pi'})^\top \Sigma_w^{-1}(\phi_\pi - \phi_{\pi'})\\
    &=(\phi_\pi - \phi_{\pi^*}+\phi_{\pi^*}- \phi_{\pi'})^\top \Sigma_w^{-1}(\phi_\pi- \phi_{\pi^*}+\phi_{\pi^*} - \phi_{\pi'})\\
    &=(\phi_\pi - \phi_{\pi^*})^\top \Sigma_w^{-1}(\phi_\pi- \phi_{\pi^*})+(\phi_{\pi^*}- \phi_{\pi'})^\top\Sigma_w^{-1} (\phi_{\pi^*} - \phi_{\pi'})+2(\phi_\pi - \phi_{\pi^*})^\top \Sigma_w^{-1}(\phi_{\pi^*}-\phi_{\pi'})\\
    &\leq 2\max_{\pi''\in S_\ell}\|\phi_{\pi''} - \phi_{\pi^*}\|_{\Sigma_w^{-1}}^2+2(\phi_\pi - \phi_{\pi^*})^\top \Sigma_w^{-1/2}\Sigma_w^{-1/2}(\phi_{\pi^*}-\phi_{\pi'})\\
    &\leq 2\max_{\pi''\in S_\ell}\|\phi_{\pi''} - \phi_{\pi^*}\|_{\Sigma_w^{-1}}^2+2\|\phi_\pi - \phi_{\pi^*}\|_{\Sigma_w^{-1}}\|\phi_{\pi^*}-\phi_{\pi'}\|_{\Sigma_w^{-1}}\tag{Cauchy-Schwarz}\\
    &\leq 4 \max_{\pi \in S_\ell} \|\E_{x\sim p}[\phi(x,\pi(x))-\phi(x,\pi^*(x))]\|_{\Sigma_w^{-1}}^2.
\end{align*}
Therefore, 
\begin{equation}\label{eqn:total_num_1}
    \sum_{\ell=1}^{\lceil \log_2(4\Delta_\epsilon^{-1}) \rceil} n_\ell\leq 64c \log( \log(\Delta_\epsilon^{-1}) |\Pi| /\delta)\lceil \log_2(4\Delta_\epsilon^{-1})\rceil\rho^{\mathrm{act}}_{\cB, \epsilon}.
\end{equation}
We now use a discretization argument to show the bound in \eqref{eq:rho}. Define $\epsilon$-ball $\mathcal{T}_\epsilon:=\{\pi: \forall \pi,\pi', \left\langle\phi_\pi-\phi_{\pi'},\theta^*\right\rangle\leq \epsilon\}$ and let $\mathcal{T}$ be a cover for $\Pi$ using those balls, i.e. $\mathcal{T}:=\{\mathcal{T}_{\epsilon,i}\}_{i=1}^{|\mathcal{T}|}$ 
Since $\theta^*\in\R^d$, the covering number $|\mathcal{T}|\leq O((1/\epsilon)^d)$. Let $\Pi_{\mathcal{T}}:=\{\pi_i:\pi_i\in \mathcal{T}_{\epsilon,i}\}_{i=1}^{|\mathcal{T}|}$ be a collection of policies where we take one policy from each $\epsilon$-ball in the cover $\mathcal{T}$. Then $|\Pi_{\mathcal{T}}|=|\mathcal{T}|\leq O((1/\epsilon)^d)$. The exact argument holds for identifying the optimal policy in $\Pi_{\mathcal{T}}$, i.e. it takes at most 
\begin{equation}\label{eqn:total_num_2}
    64c \log( \log(\Delta_\epsilon^{-1}) |\Pi_{\mathcal{T}}| /\delta)\lceil \log_2(4\Delta_\epsilon^{-1})\rceil\rho^{\mathrm{act}}_{\cB, \epsilon}=O((d\log(1/\epsilon)+\log(1/\delta)) \log_2(\Delta_\epsilon^{-1})\rho^{\mathrm{act}}_{\cB, \epsilon})
\end{equation}
samples to identify $\pi_{\mathcal{T}}^*\in \Pi_{\mathcal{T}}$. However, note that the global optimal policy $\pi^*$ must lie in some ball $\mathcal{T}_{\epsilon,i_0}$, so $V(\pi^*)-V(\pi_{\mathcal{T}}^*)\leq V(\pi^*)-V(\pi_{i_0})\leq \epsilon$, so $\pi_{\mathcal{T}}^*$ is an $\epsilon$-optimal policy. The first part of the statement follows by taking the minimum of \eqref{eqn:total_num_1} and \eqref{eqn:total_num_2}. As for the inequality bounding the sample complexity by $\tilde{O}(d^2/\epsilon^2)$, note that $\rho_{\cB,\epsilon}^{\mathrm{act}}$ is upper bounded by $\rho_{\cB, \epsilon}^{\mathrm{pas}}$, so this upper bound follows directly from the second part of Theorem 2.14 of \cite{li2022instance}. 
\end{proof}


Theorem~\ref{thm:active-contextual_rage} improves over Theorem~\ref{thm:instance-optimal} in the following sense. Note that because the minimization in $\varrho_{\cB, \epsilon}$ is over a \emph{strictly larger} distribution class ($\Delta^{\cX \times \cA}$)than the minimization in $\rho_{\cB, \epsilon}^{\mathrm{pas}}$ (which is only minimized over $\cF$), we have that
\begin{align*}
    \rho_{\cB, \epsilon}^{\mathrm{act}} \leq \rho_{\cB, \epsilon}^{\mathrm{pas}}, 
\end{align*}
In the following section, we demonstrate that that on the hard instance for passive context sampling described in Section~\ref{sec:comparison} (Definition~\ref{def:hard-instance}), we have that $\rho_{\cB, \epsilon}^{\mathrm{act}}$ is indeed a $\Theta(d)$-factor smaller than $\rho_{\cB, \epsilon}^{\mathrm{pas}}. $ 

\subsubsection{Demonstrating the power of active context sampling }\label{sec:separation}

The main result of this section is the following.  

\begin{lemma} Let $A,d \in \mathbb{Z}_{>1}$. Let $\cB^\star_{d, A}$ be the SCLB instance from Definition~\ref{def:hard-instance}. Then, for any $\epsilon \in (0, 1)$ we have that 
\begin{align*}
    \rho_{\cB^\star_{d,A}, \epsilon}^{\mathrm{act}} \leq \frac{8}{d} \cdot \rho_{\cB^\star_{d,A}, \epsilon}^{\mathrm{pas}}. 
\end{align*}
\end{lemma}
\begin{proof} 
For notational convenience, we let $\cB = \cB^\star_{d,A}$ inside this proof. The optimal policy $\pi^\star$ is given by $\pi^\star(x) = 1, \forall x \in \cX$. From the definitions of $\rho_{\cB,\epsilon}, \varrho_{\cB, \epsilon}$, and $\cB$, it is enough to show that
\begin{align}\label{eq:comparison-equation-here}
    \min_{w \in \Delta^{\cX \times \cA}} \max_{\pi: \pi \neq \pi^\star} {\normInline{\phi_\pi - \phi_{\pi^\star}}_{\Sigma_w^{-1}}}^2 \leq \frac{8}{d} \cdot \min_{w \in \cF} \max_{\pi: \pi \neq \pi^\star} {\normInline{\phi_\pi - \phi_\pi^\star}_{\Sigma_w^{-1}}}^2. 
\end{align}

Note that due to the structure of the SCLB---wherein all actions besides $a=1$ reveal the same feature $e_1$---one maximizing policy in the inner maximization will be given by $\pi(x) = 2, \forall x \in \cA$. Thus, without loss of generality, we can fix 
\begin{align}
    {\phi_\pi - \phi_{\pi^\star}} = e_1 - \paren{\paren{1 - \frac{d-1}{d^2}} e_1 + \frac{1}{d^2} \sum_{i =2}^d e_i} = \frac{d - 1}{d^2} e_1 - \sum_{i=2}^d \frac{1}{d^2}e_i. 
\end{align}
and show that 
\begin{align}\label{eq:comparison-equation-here-again}
    \min_{w \in \Delta^{\cX \times \cA}} {\normInline{\phi_\pi - \phi_\pi^\star}_{\Sigma_w^{-1}}}^2 \leq \frac{8}{d} \cdot \min_{w \in \cF} {\normInline{\phi_\pi - \phi_\pi^\star}_{\Sigma_w^{-1}}}^2. 
\end{align}
The remainder of the proof is devoted to proving \eqref{eq:comparison-equation-here-again}.

Now, consider the right-hand-side of \eqref{eq:comparison-equation-here-again}. We need to reason about the optimal choice of $w$. Note that all actions for $a \neq 1$ reveal the same feature vector---which corresponds to 0 reward; so, without loss of generality we can restrict the minimizing sampling distribution's support to $a = 1$. Moreover, because $w \in \cF$ constrains the marginal of $w$ with respect to $x$ to be equal to $p$, this indicates that one minimizing sampling distribution for the right-hand-side of \eqref{eq:comparison-equation-here-again} is given by 
\begin{align*}
    w(x,a) = \begin{cases}
        1 - (d-1)/d^2, & x = 1, a = 1\\
        1/d^2 & x \neq 1, a = 1 \\
        0 & \text{otherwise}
    \end{cases}. 
\end{align*}

In this case, $\Sigma_{w}^{-1}$ is a diagonal matrix with $\frac{d^2 - d + 1}{d^2}$ in the first entry, and $d^2$ on the remainder of the diagonal. Hence, we have
\begin{align*}
    \min_{w \in \cF} {\normInline{\phi_\pi - \phi_\pi^\star}^2_{\Sigma_w^{-1}}} &= \frac{d^2-d+1}{d^2} \cdot \frac{(d-1)^2}{d^4} + d^2 \sum_{i=2}^d \frac{1}{d^4} \\
    &= \frac{d^2-d+1}{d^2} \cdot \frac{(d-1)^2}{d^4} + \frac{d-1}{d^2} \\
    &\geq \frac{1}{2d}, 
\end{align*}
where the last inequality uses that $d > 1$. 

Next, we turn to the left-hand-side of \eqref{eq:comparison-equation-here-again}. Consider the (active) sampling distribution 
\begin{align*}
    w'(x,a) = \begin{cases}
        \frac{1}2, & x = 1, a = 1\\
        \frac{1}{2(d-1)}& x \neq 1, a = 1 \\
        0 & \text{otherwise}
    \end{cases}.
\end{align*}
In this case, $\Sigma_{w'}^{-1}$ is a diagonal matrix with $2$ in the first entry, and $2(d-1)$ on the remainder of the diagonal.
This distribution is actively sampling contexts in the sense that its marginal with respect to $x$ is \emph{not} $p$. A similar calculation as above shows that 
\begin{align*}
    \min_{w \in \Delta^{\cX \times \cA}} {\normInline{\phi_\pi - \phi_\pi^\star}^2_{\Sigma_w^{-1}}} &\leq {\normInline{\phi_\pi - \phi_\pi^\star}^2_{\Sigma_{w'}^{-1}}} = 2 \frac{(d-1)^2}{d^4} + 2(d-1) \sum_{i = 2}^d \frac{1}{d^4} \\
    &= 2 \frac{(d-1)^2}{d^4} + \frac{2(d-1)^2}{d^4} \leq \frac{4}{d^2}. 
\end{align*}

Thus, we conclude that 
\begin{align*}
    \min_{w \in \Delta^{\cX \times \cA}} {\normInline{\phi_\pi - \phi_\pi^\star}_{\Sigma_w^{-1}}}^2 \leq \frac{4}{d} = 8 \cdot \frac{1}{2d} \leq 8 \min_{w \in \cF} {\normInline{\phi_\pi - \phi_\pi^\star}_{\Sigma_w^{-1}}}^2, 
\end{align*}
as desired. 
\end{proof}

This result shows that from the PAC-learning perspective, Active-ContextualRAGE (Algorithm~\ref{alg:rage}) once again improves the dimension dependence from \cite{li2022instance} on an instance-dependent basis. However, as neither is known to be implementable in polynomial time, it remains a future work to explore possible implementations or heuristics to approximately simulate these algorithms. We hope our work provides insight on how to incorporate active context sampling, if a polynomial-time implementation or approximation of Contextual-RAGE is eventually developed in future research. 

%% file: appendix/additional_experiments.tex
\newpage
\section{Additional experimental details}\label{apx:additional-experiments}

This appendix contains additional experimental details to aid in reproducing our empirical results.

\subsection{Overview of passive context sampling baselines}\label{sec:passive-learning-baseline}

Here, we briefly describe the two baselines against which we compare our empirical results. 

\paragraph{Reward-free LinUCB (RFLinUCB).} RFLinUCB is a slight modification (proposed in \cite{zanette2021design}) of the LinUCB algorithms (proposed by \cite{abbasi2011improved}) in order to adapt it for the pure exploration setting. The RFLinUCB algorithm implements the standard LinUCB algorithm of \citet{abbasi2011improved} with the reward function set to 0. The original LinUCB algorithm contains the reward function in its exploration policy because it is designed to minimize the online regret (rather than the simple regret, which is of interest in the pure exploration setting), and this is an unncessary in the pure exploration setting \citep{zanette2021design}.  

\paragraph{Planner-Sampler} \citet{zanette2021design} point out that one limitation of the RFLinUCB algorithm is that the algorithm is \emph{adaptive} to the observed context-action pairs, in the sense that the $t$-th observed context depends on the history $\{x_1, a_1\},..., \{x_t, a_t\}$. This can be challenging to implement in certain real-world scenarios, since it requires \emph{sequentially} querying context-action pairs. To address this, the planner-sampler algorithm of \citet{zanette2021design} is designed to match the regret bound of RFLinUCB with a \emph{fixed, non-adaptive} policy. The Planner-Sampler algorithm proceeds in two stages. 

First, a ``Planner'' algorithm observes the features (but not the rewards) of $T_0$ contexts $\{\phi(x_t, a_t)\}_{t \in [T_0], a\in\cA}$ for an offline set of contexts drawn independently $x_t \sim p$. The Planner uses these observed features to design a stochastic exploration policy $\pi': \cX \to \cA$. 

In the second stage, a ``Sampler'' algorithm samples the rewards of $T$ context-action pairs drawn according to the Planner's policy $\pi'$. That is, the Sampler observes $r(x'_1, a'_1), ..., r(x'_{T}, a'_{T})$ where each $x'_t \sim p$ and $a'_t \sim \pi'(x'_t)$ independently. \citet{zanette2021design} study various tradeoffs of $T_0$ and $T$. 

In our experiments, we generously set $T_0 = |\cX|$ in order to ensure fair comparison with RF-LinUCB and our active context sampling approach, which use knowledge of the full feature mapping. We then measure the regret as a function of the number of reward observations, $T$ required by the ``Sampler''. 

\paragraph{Active-SCLB and Passive-SCLB.} Recall that in our experiments, we run Algorithm~\ref{alg:active-lcb} with $\alpha \gets 0$ and solve the SDP in Line~\ref{line:approximate} of Algorithm~\ref{alg:active-lcb} to convergence. In this setting, there is a natural passive context sampling analogue of our Algorithm: instead of (approximately) solving 
\begin{align*}
    \hat{w} \approx \argmin_{w \in \Delta^{\cS \times \cA}} \normInline{\phi(x,a)}_{\Sigma_{\hat{w}}^{-1}}^2, 
\end{align*}
we can solve the same problem with the additional constraint that the marginal distribution of $w$ with respect to $x$ must match $p$. That is, we can define 
\begin{align*}
    \cF = \left\{w \in \Delta^{\cX \times \cA} : \forall x \in \cX, \sum_{a \in \cA} w(x,a) = p(x) \right\} 
\end{align*}
and solve
\begin{align*}
    \hat{w} \approx \argmin_{w \in \cF} \normInline{\phi(x,a)}_{\Sigma_{\hat{w}}^{-1}}^2. 
\end{align*}
Note that restricting the feasible space to $\cF$ only requires some additional linear constraints, and consequently it is easy to see that the problem can still be formulated as an SDP. For completeness, we include a full pseudocode of the versions used in our numerical experiments in Algorithm~\ref{alg:active-sclb-experiments} and Algorithm~\ref{alg:passive-sclb-experiments}.

\begin{algorithm2e}[htbp]
\DontPrintSemicolon            
\SetAlgoLined                  
\KwIn{SCLB $\cB = (\cX, \cA, \phi, p, \nu, \theta^\star)$, regularization parameter $\lambda > 0$, sample budget $T \in Z_{>0}$.}
\KwOut{A policy $\hat{\pi}: \cX \to \cA$}
\BlankLine
\caption{Active-SCLB}\label{alg:active-sclb-experiments}
\tcp{Approximately solve the following using an SDP solver.}
 Compute $\hat{w} = \argmin_{w \in \Delta^{\cX \times \cA}}\E_{x \sim p} \max_{a \in \cA} \normInline{\phi(x,a)}_{\SigmaOInv{w}}^2$\; 
  \tcp{Draw samples and apply ridge regression to compute a policy $\hat{\pi}$. Recall in synthetic experiments, we sampled with replacement, while in experiments on real-world datasets, we sampled without replacement using rejection sampling.}
 $\cS \gets \smash{\{(x_1, a_1), ..., (x_T, a_T) \}}$ where each $(x_t,a_t) {\sim} \hat{w}$ \iid\; 
 $\thetahat \gets \SigmaLambdaSInv{\cS}{\lambda} \sum_{(x_t, a_t) \in \calS} \phi(x_t, a_t) r_{t} $\; 
 \Return{$\hat{\pi}(x) \gets x \mapsto \argmax_{a \in \cA} \phi(x,a)^\top \thetahat$}
\end{algorithm2e}

\begin{algorithm2e}[htbp]
\DontPrintSemicolon            
\SetAlgoLined                  
\KwIn{SCLB $\cB = (\cX, \cA, \phi, p, \nu, \theta^\star)$, regularization parameter $\lambda > 0$, sample budget $T \in Z_{>0}$.}
\KwOut{A policy $\hat{\pi}: \cX \to \cA$}
\BlankLine
\caption{Passive-SCLB}\label{alg:passive-sclb-experiments}
\tcp{Approximately solve the following using an SDP solver.}
 Compute $\hat{w} = \argmin_{w \in \cF}\E_{x \sim p} \max_{a \in \cA} \normInline{\phi(x,a)}_{\SigmaOInv{{w}}}^2$\; 
  \tcp{Draw samples and apply ridge regression to compute a policy $\hat{\pi}$. Recall in synthetic experiments, we sampled with replacement, while in experiments on real-world datasets, we sampled without replacement using rejection sampling.}
 $\cS \gets \smash{\{(x_1, a_1), ..., (x_T, a_T) \}}$ where each $(x_t,a_t) {\sim} \hat{w}$ \iid\; 
 $\thetahat \gets \SigmaLambdaSInv{\cS}{\lambda} \sum_{(x_t, a_t) \in \calS} \phi(x_t, a_t) r_{t} $\; 
 \Return{$\hat{\pi}(x) \gets x \mapsto \argmax_{a \in \cA} \phi(x,a)^\top \thetahat$}
\end{algorithm2e}

\subsection{Implementation details for SDP solving}\label{sec:implementational-details} In all of our experiments, we used CVXPY (an open source Python-embedded modeling language for convex optimization problems) to model the SDP variables, objectives, and constraints. Within CVXPY, we used the Mosek (Version 10) SDP solver for solving the SDPs (described further in the following paragraph.) All solver hyperparameters were fixed to their default values in CVXPY.

MOSEK is a software package for solving structured optimizations such as SDPs. Although it is not open-sourced, MOSEK is available \href{https://www.mosek.com/products/academic-licenses/}{to academic users for free} Additionally, for non-academic users, MOSEK offers a \href{https://www.mosek.com/products/trial/}{free trial period}.

To ensure numerical stability, when applying the SDP solver, we (1) included a small amount of numerical regularization for the SDP constraints, on the order of $1\text{e-}6$, to avoid numerical stability errors and (2) normalized all features such that $\normInline{\phi(x,a)}_2 \leq 1$. For consistency, the feature normalization was applied to \emph{all} baselines as well as Active-SCLB in our experiments.  

\subsection{Runtime.}\label{sec:runtime}

Active-SCLB and Passive-SCLB have a runtime overhead to initialize a sampling distribution, but after this initial runtime overhead, sampling $T$ values from this distribution is fast even if $T$ is large and in particular, the reward sampling can be done fully in parallel \citep{zanette2021design}. 

On the other hand, RFLinUCB has no initialization overhead because it sequentially selects the next action, depending on the past. Each sample must be implemented sequentially, so its complexity scales like $T \times R$ where $R$ is the time to collect a single reward observation. 

Thus, while RFLinUCB has low overhead and is computationally cheap, its practical application can be slow in settings where each reward observation takes time (e.g., clinical trials, educational interventions, world simulations etc.), i.e., when $R$ is large. As \citep{zanette2021design} explain, this is a key consideration which is not captured in pure numerical simulations, where $R$ is artifically very small. 

In Table~\ref{tab:warfarin_results_time}, we show runtime (minutes) on Warfarin in the same setup as Section 6. An entry of 0 indicates a value $<1$ minute. We show the initialization time as well as the additional time required to draw $T$ samples for a few values of $T$. Initialization overhead of Active-SCLB is slightly higher than Passive-SCLB and Planner-Sampler due to its added complexity. RFLinUCB has no initialization overhead and is the fastest; however, Active-SCLB consistently achieves lower regret (see Figure 1 of our paper.) As discussed above, RFLinUCB sequentially samples rewards which can be slow in some practical applications.

\begin{table}[h!]
\centering
\begin{tabular}{llccc}
\toprule
\textbf{Method} & \textbf{Initialization} & \textbf{T=1,000} & \textbf{T=5,000} & \textbf{T=10,000} \\
\midrule
RFLinUCB        & 0  & 1 & 2 & 4 \\
Planner-Sampler & 22 & 0 & 0 & 0 \\
Passive-SCLB    & 17 & 0 & 0 & 0 \\
Active-SCLB     & 24 & 0 & 0 & 0 \\
\bottomrule
\end{tabular}
\caption{Runtime comparison on the Warfarin dataset.}
\label{tab:warfarin_results_time}
\end{table}

\subsection{Experiments with approximate $p$.}\label{sec:approximate_p}

We modified Active-SCLB to replace $p$ with $p’$, where $p'$ is the empirical distribution constructed from $|\mathcal{X}|/4$ random contexts (denoted Active-SCLB-Empirical) in Figure~\ref{fig:subsampled-p}. The setup is otherwise identical to Section~\ref{sec:experiments}. 

Active-SCLB-Empirical is sometimes slightly worse than Active-SCLB but performs similarly overall and consistently outperforms baselines. Baselines require significantly more samples for the average regret to match that of Active-SCLB-Empirical. 

\begin{figure}
    \centering
    \includegraphics[width=0.23\linewidth]{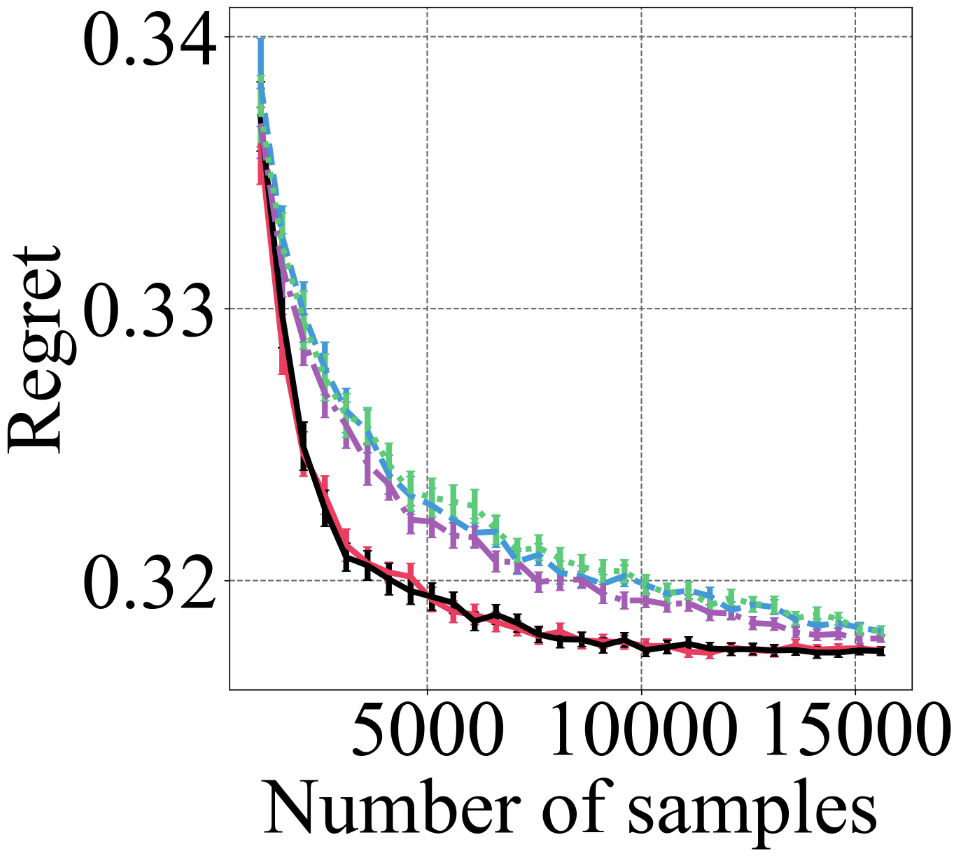}
    \includegraphics[width=0.23\linewidth] {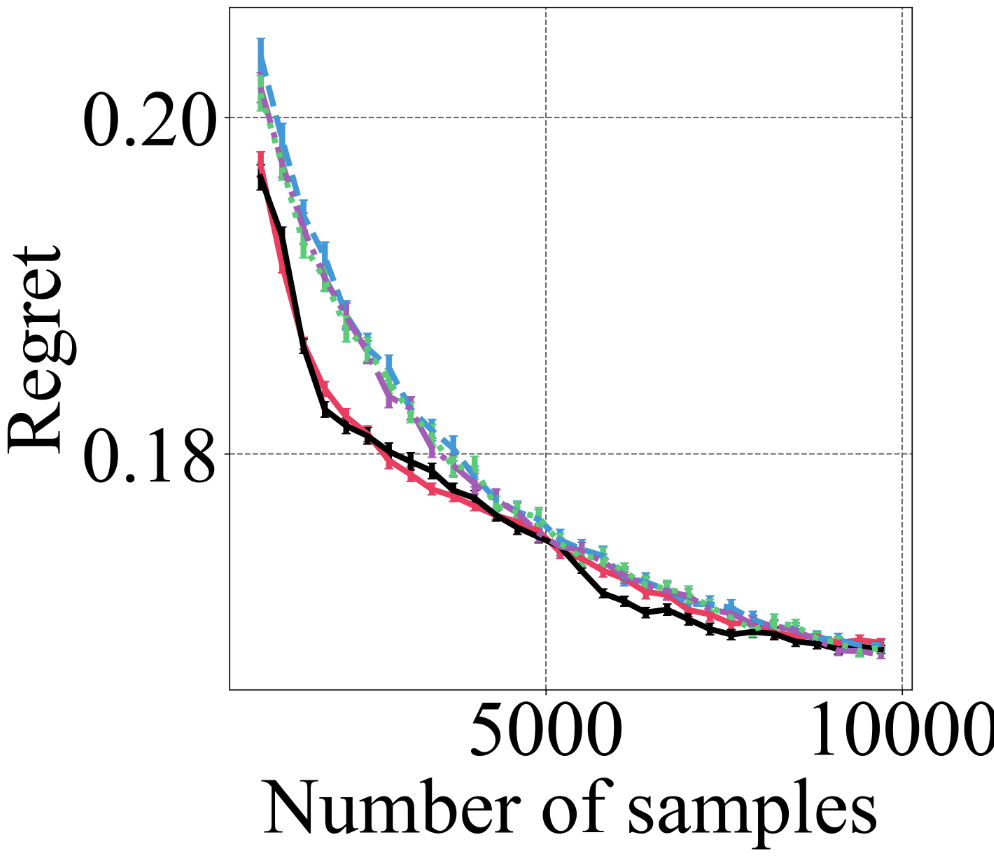} \\
    \vspace{1em}
    \includegraphics[width=0.35\linewidth]{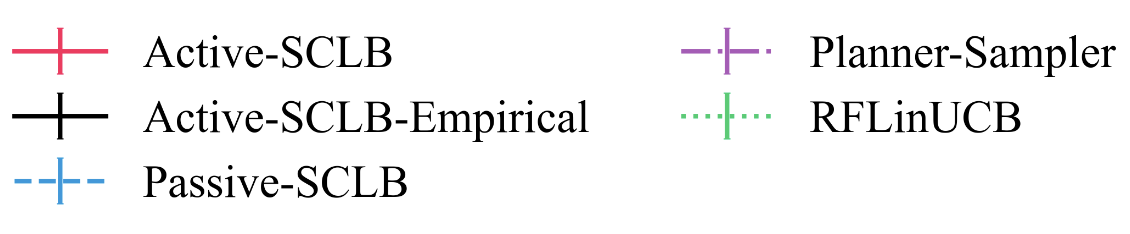}
    \caption{Experiments on Jester and Warfarin with subsampled $p$. Overall, we see that Active-SCLB-Empirical performs competitively wuth Active-SCLB, sometimes even outperforming it (possibly due to some regularization affects associated with using an empirical $p$).}
    \label{fig:subsampled-p}
\end{figure}

\subsection{Experiments with different regularization levels.}\label{sec:regularizer}

In this section, we include results on our real-world datasets for differing values of the regularizer $\lambda$. The choice of $\lambda$ is typically a design choice, which practitioners use to balance bias-variance trade-offs (see also, the discussion in \citep{zanette2021design}). Smaller values of $\lambda$ reduce bias at the risk of increased variance; larger values of $\lambda$ increase bias (often at the benefit of reduced variance). 

Figures~\ref{fig:warfarin} and~\ref{fig:joke} show that our improvements remain largely consistent as we vary $\lambda \in \{1\text{e-}6, 1\text{e-}4, 1\text{e-}2\}$. These results were omitted in the main body for brevity, because the results are largely consistent across different values of $\lambda$. 

\begin{figure}[tb]
  \centering
  \captionsetup[subfigure]{font=small,justification=centering}

  \begin{tabular}{ccc}
    \subcaptionbox{$\lambda=1\mathrm{e-}6$}[0.28\linewidth]{%
      \begin{minipage}{\linewidth}
        \centering
        \includegraphics[width=\linewidth]{figures/real-world/warfarin/warfarin_1e-6_regret.png}\\[1ex]
        \includegraphics[width=\linewidth]{figures/real-world/warfarin/warfarin_1e-6_match.png}
      \end{minipage}
    } &
    \subcaptionbox{$\lambda=1\mathrm{e-}4$}[0.28\linewidth]{%
      \begin{minipage}{\linewidth}
        \centering
        \includegraphics[width=\linewidth]{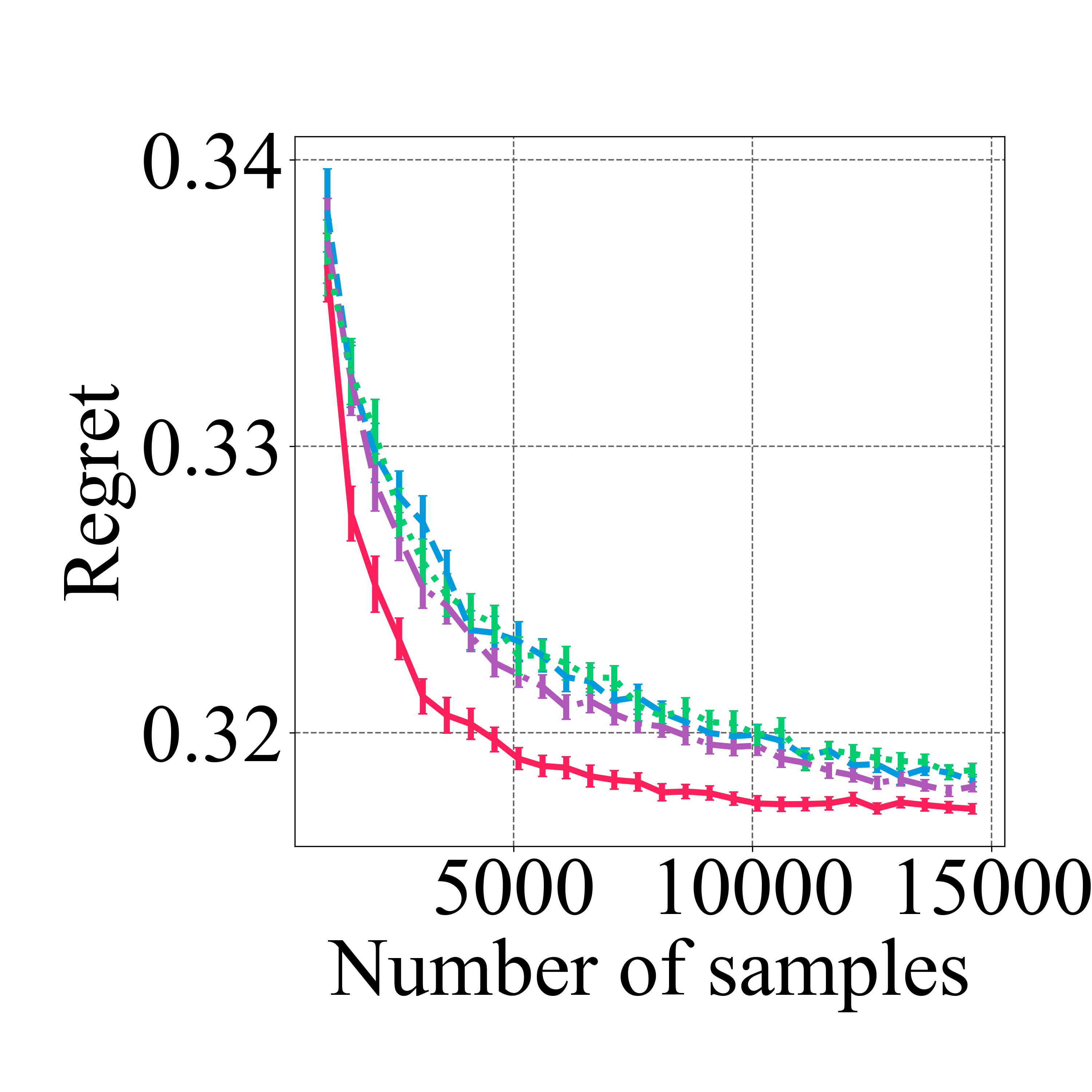}\\[1ex]
        \includegraphics[width=\linewidth]{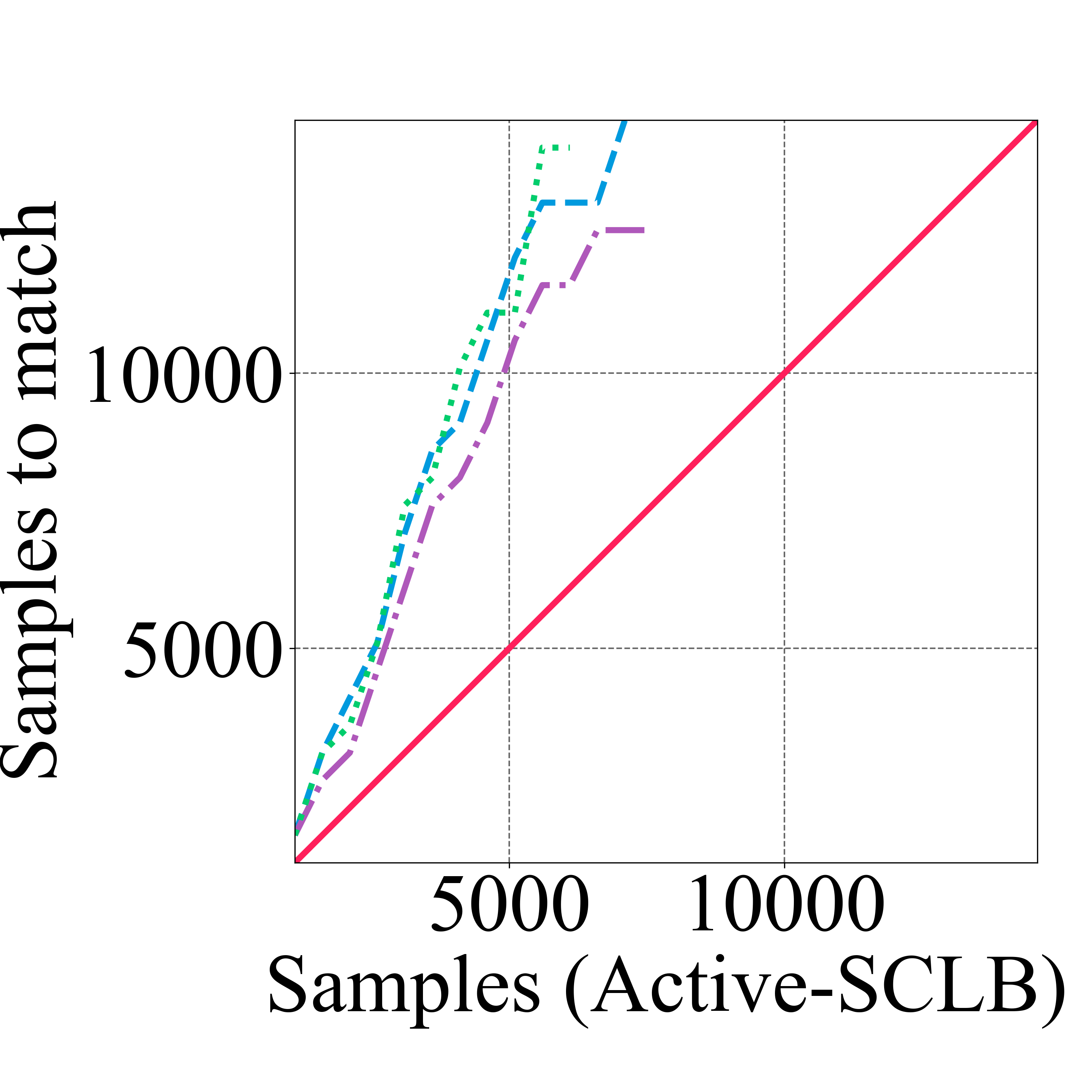}
      \end{minipage}
    } &
    \subcaptionbox{$\lambda=1\mathrm{e-}2$}[0.28\linewidth]{%
      \begin{minipage}{\linewidth}
        \centering
        \includegraphics[width=\linewidth]{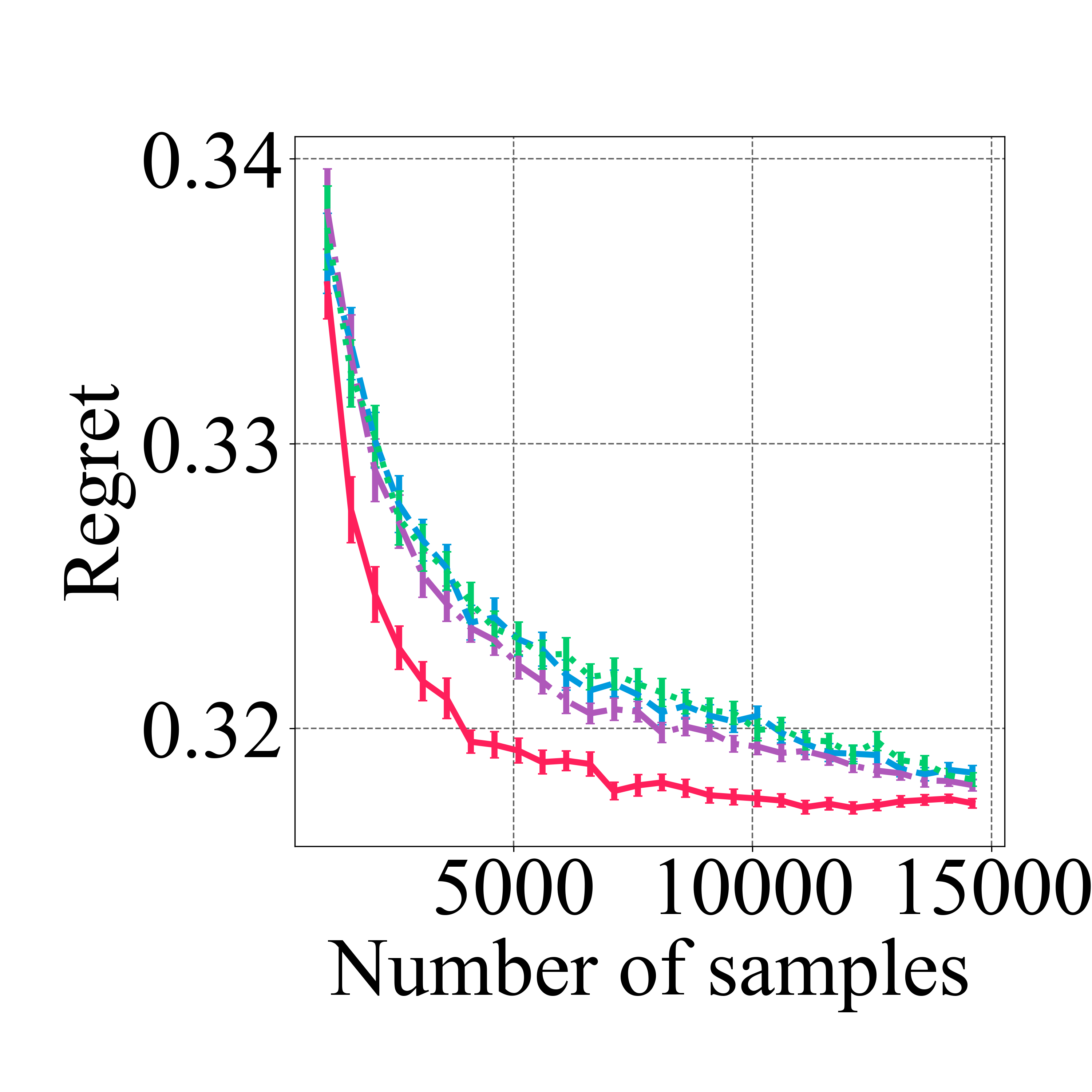}\\[1ex]
        \includegraphics[width=\linewidth]{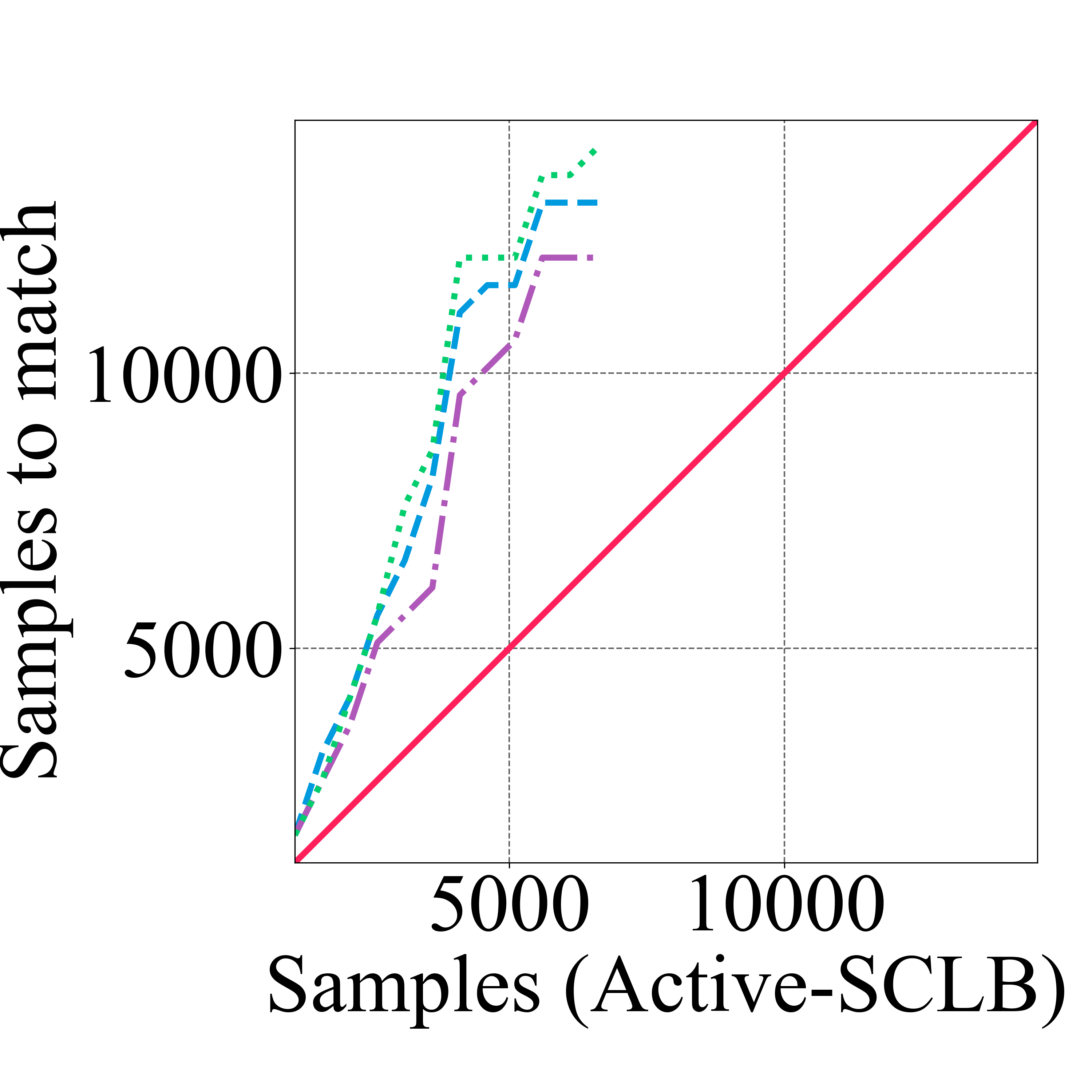}
      \end{minipage}
    }
  \end{tabular}

  \centering
  \begin{minipage}[c]{0.45\linewidth}
    \centering
    \includegraphics[width=\linewidth]{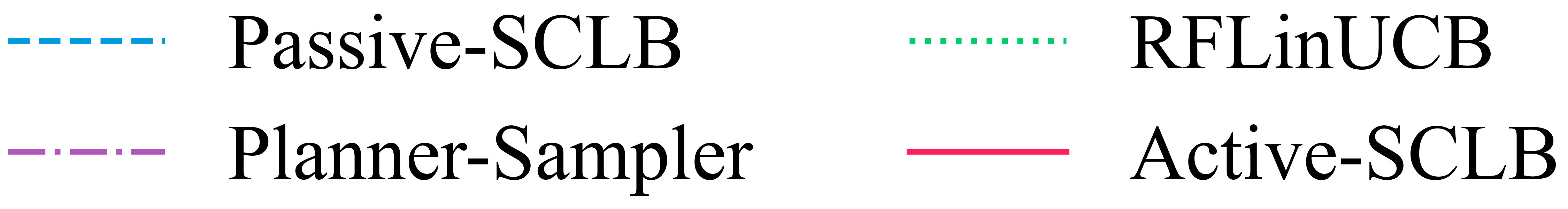}
  \end{minipage}%
  \hfill
  \begin{minipage}[c]{0.45\linewidth}
    \centering
    {\small \textbf{Regret of naive baseline:} 0.382}
  \end{minipage}

  \caption{%
    Warfarin dataset. \emph{Top:} Regret vs.\ number of samples (mean $\pm$ 2 standard errors over 100 trials). \emph{Bottom:}  Minimum number of samples required for baselines' mean regret (over 100 trials) to match Active-SCLB's mean regret for a given sample budget.
  }
  \label{fig:warfarin}
\end{figure}

\begin{figure}[tb]
  \centering
  \captionsetup[subfigure]{font=small,justification=centering}

  \begin{tabular}{ccc}
    \subcaptionbox{$\lambda=1\mathrm{e-}6$}[0.29\linewidth]{%
      \begin{minipage}{\linewidth}
        \centering
        \includegraphics[width=\linewidth]{figures/real-world/joke/joke_1e-6_regret.png}\\[1ex]
        \includegraphics[width=\linewidth]{figures/real-world/joke/joke_1e-6_match.png}
      \end{minipage}
    } &
    \subcaptionbox{$\lambda=1\mathrm{e-}4$}[0.29\linewidth]{%
      \begin{minipage}{\linewidth}
        \centering
        \includegraphics[width=\linewidth]{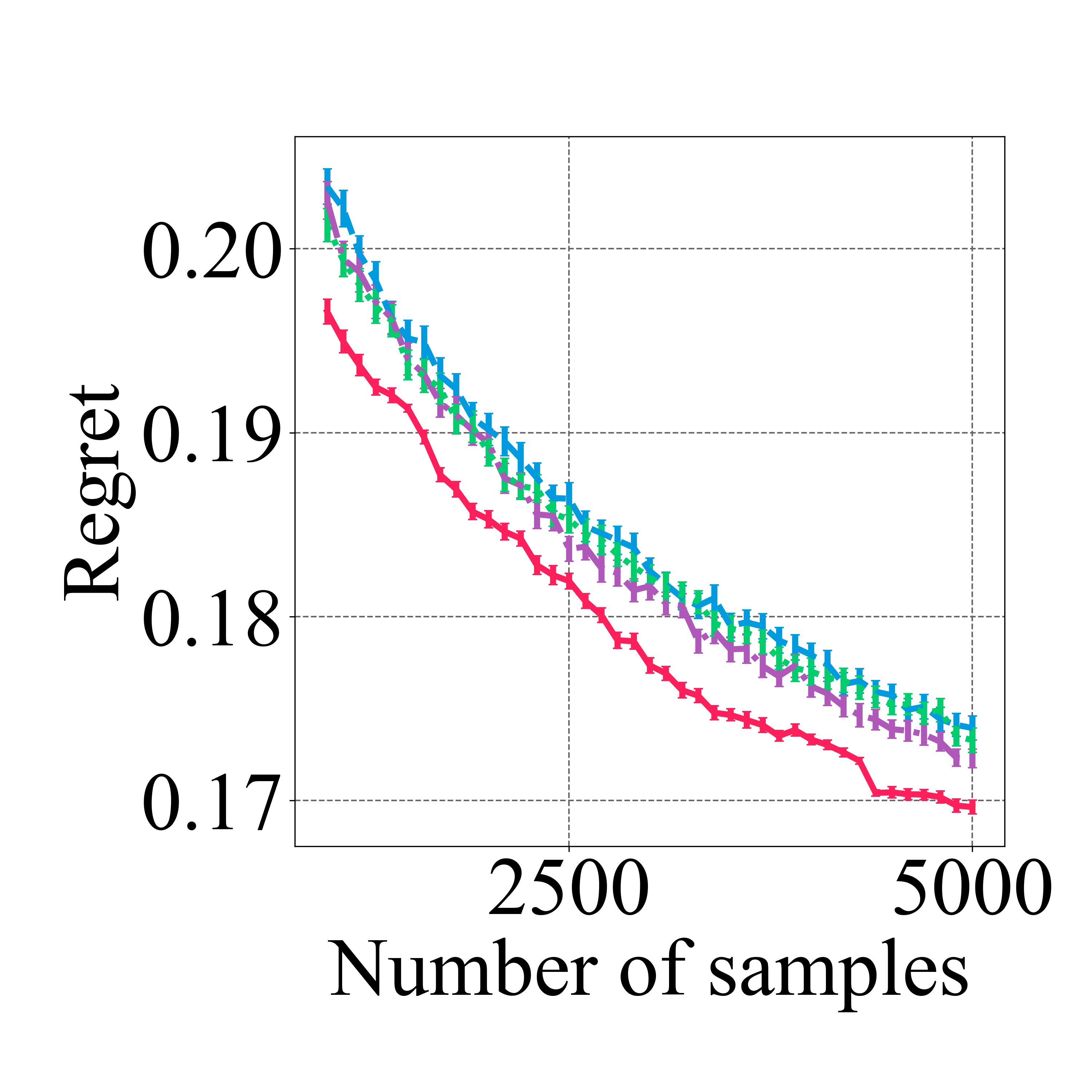}\\[1ex]
        \includegraphics[width=\linewidth]{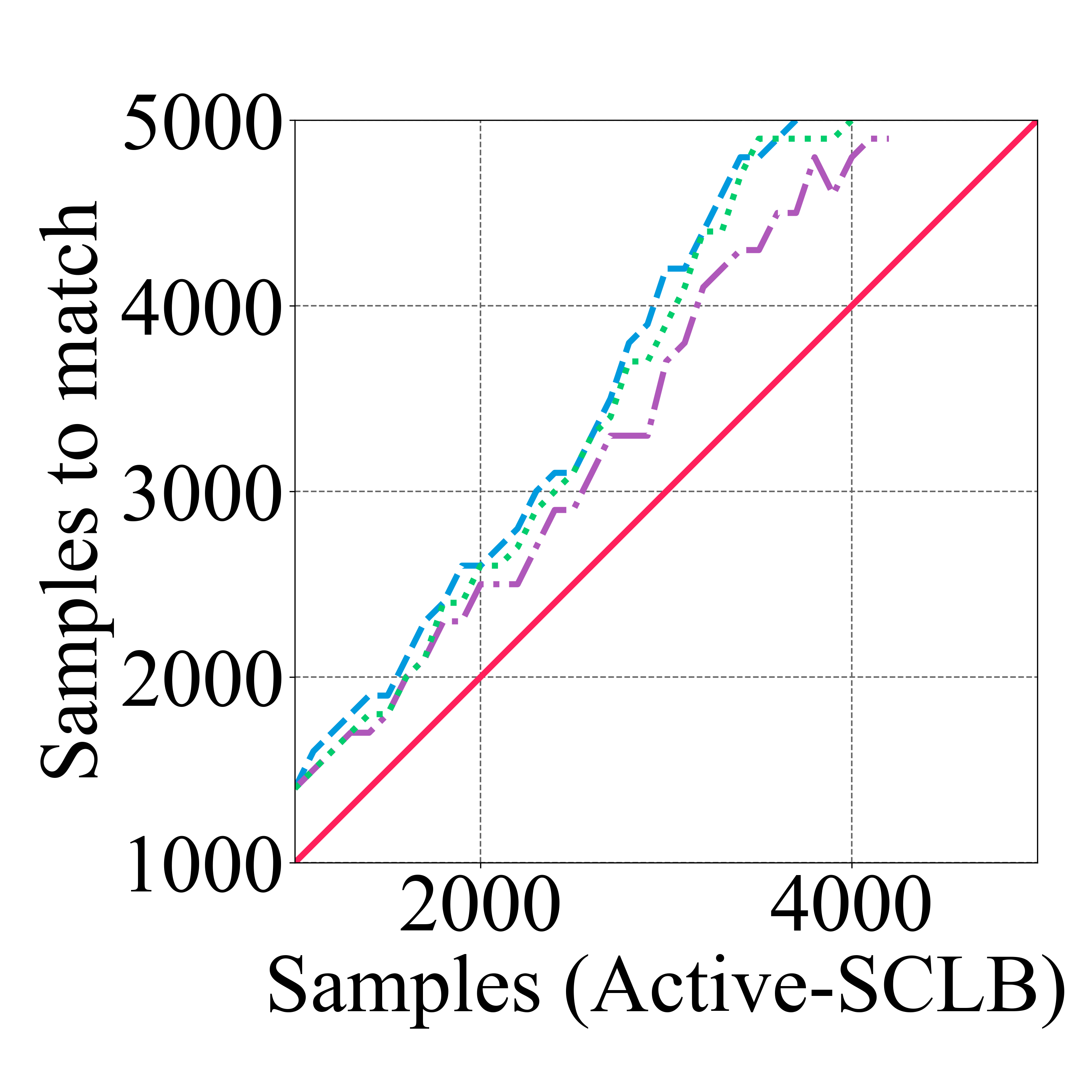}
      \end{minipage}
    } &
    \subcaptionbox{$\lambda=1\mathrm{e-}2$}[0.29\linewidth]{%
      \begin{minipage}{\linewidth}
        \centering
        \includegraphics[width=\linewidth]{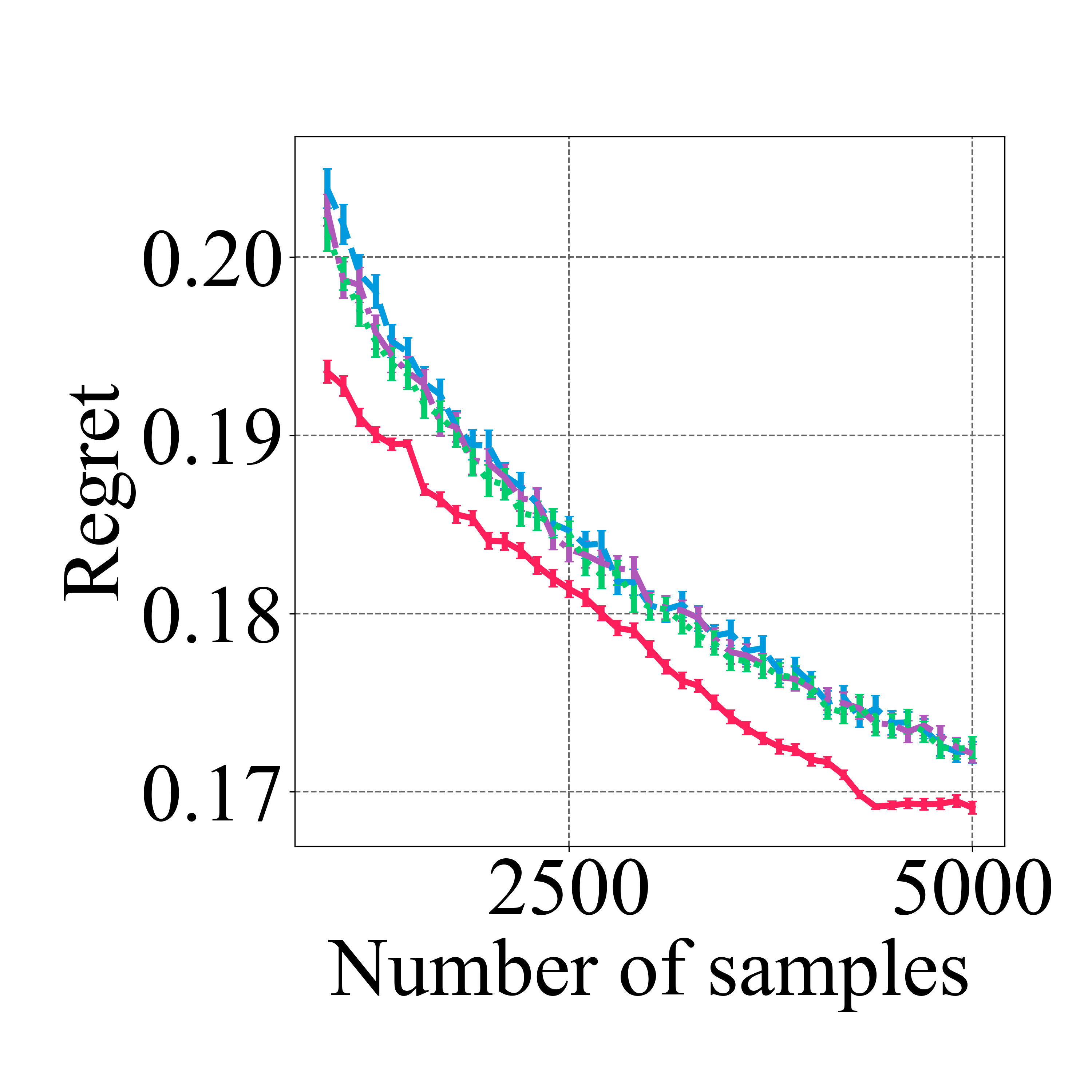}\\[1ex]
        \includegraphics[width=\linewidth]{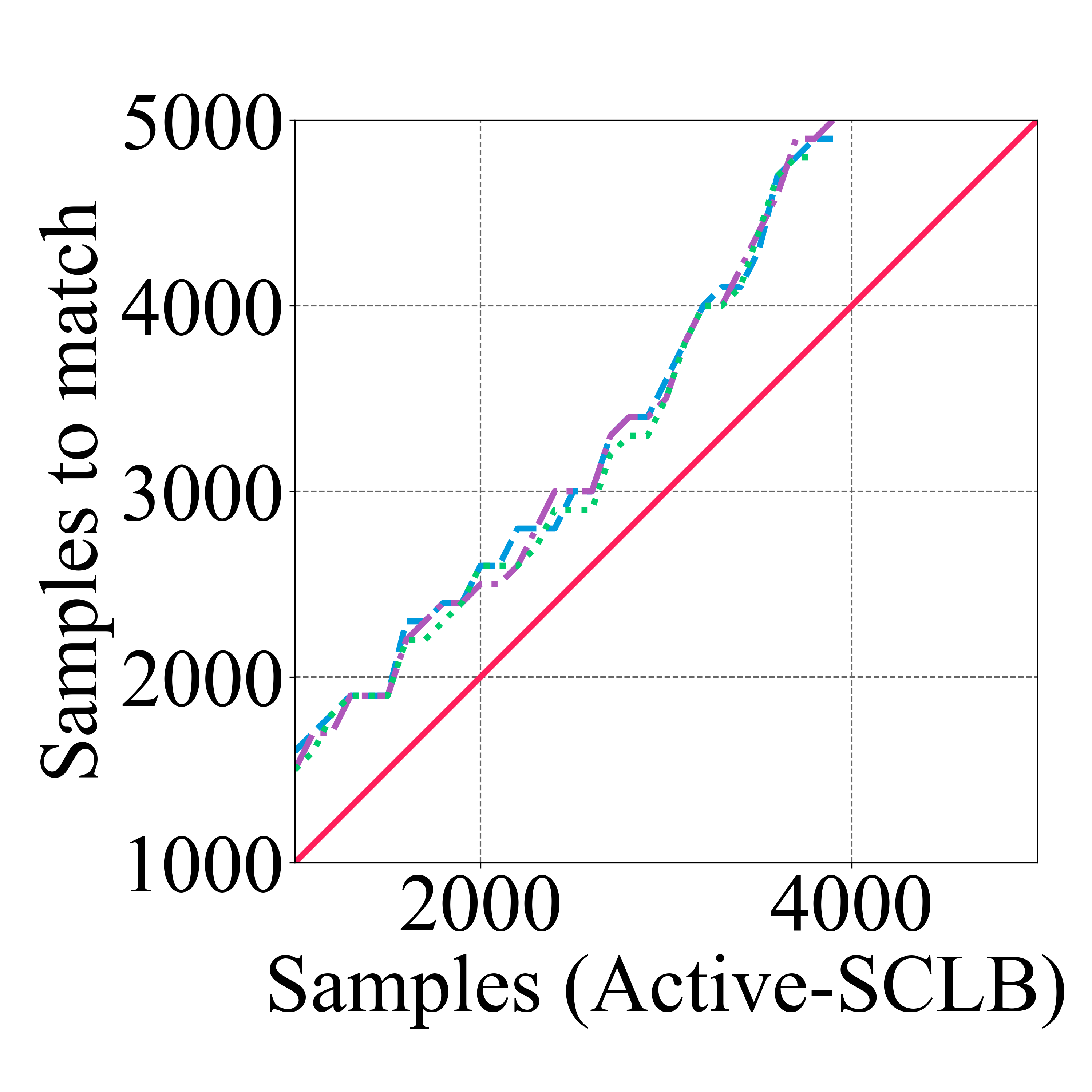}
      \end{minipage}
    }
  \end{tabular}

  \centering
  \begin{minipage}[c]{0.45\linewidth}
    \centering
    \includegraphics[width=\linewidth]{figures/legend.png}
  \end{minipage}%
  \hfill
  \begin{minipage}[c]{0.45\linewidth}
    \centering
    {\small \textbf{Regret of naive baseline:} 0.219}
  \end{minipage}

  \caption{%
    Jester dataset. \emph{Top:} Regret vs.\ number of samples (mean $\pm$ 2 standard errors over 100 trials). \emph{Bottom:}  Minimum number of samples required for baselines' mean regret (over 100 trials) to match Active-SCLB's mean regret for a given sample budget.
  }
  \label{fig:joke}
  \vspace{-.5em}
\end{figure}

%% file: appendix/omitted_proofs.tex
\newpage
\section{Omitted proofs}\label{apx:omitted-proofs}

This section contains proofs which are omitted in the main body. Recall that when applied to matrices, $\normInline{\cdot}$ always denotes the spectral norm. When applied to vectors, $\normInline{\cdot}$ denotes the Euclidean norm. 

In Section~\ref{apx:established} we collect some established results in SCLBs and PSD matrix theory, which we use in our analysis. In Section~\ref{apx:proof-main} we prove Theorem~\ref{theorem:main}. In Section~\ref{apx:comparison} we include omitted proofs from Section~\ref{sec:comparison}. 

\subsection{Omitted proofs of standard results in SCLBs and PSD matrix theory.}\label{apx:established}

We first prove some standard facts about the Lowener order. We believe the facts in the following Lemma~\ref{lemma:lowener} are well-known; however, we collect them here for the sake of completeness.

\begin{lemma}[Lowener order facts]\label{lemma:lowener} The following facts hold whenever $A, B, C \in \R^{d \times d}$ are positive semidefinite. 
\begin{enumerate}[label=(\roman*)]
    \item $A \succeq B$ if and only if for all $x \in \R^d$, $x^\top A x \geq x^\top B x$.    
    \item If $A \succ 0$, $A^{-1} \succ 0$. 
    \item If $A \succ (\succeq) B$, then for any symmetric $F \in \R^{d \times d}$, we have $FAF \succ (\succeq) FBF$.  
    \item If $A \succeq B \succ 0$, then $A^{-1} \preceq B^{-1}$.
    \item If $A = B +C$ then $A \succeq B$. 
    \item If $\normInline{ B^{-1/2} (A-B) B^{-1/2}} \leq \epsilon$ for some $\epsilon \in (0, 1)$, then $(1-\epsilon) B \preceq A \preceq (1+\epsilon) B$. 
\end{enumerate}
\end{lemma}
\begin{proof} We prove the facts one-by-one. 
\begin{enumerate}[label=(\roman*)]

    \item If $A \succeq B$, then $A - B \succeq 0$. For any $x \in \R^d$, we have
    \[
    x^\top (A - B) x \geq 0 \quad \text{ implies } \quad x^\top A x \geq x^\top B x.
    \]
    Conversely, if $x^\top A x \geq x^\top B x$ for all $x \in \R^d$, then for any $x \in \R^d$
    \begin{align*}
        x^\top (A-B) x = x^\top A x - x^\top B x \geq 0, 
    \end{align*}
    and hence $A - B \succeq 0$---thus, we have $A \succeq B$. 
    
    \item Since $A \succ 0$ (i.e., $A$ is positive definite), it is invertible and its inverse is also positive definite. This follows because for any $x \neq 0$, 
    \[
    x^\top A^{-1} x = (A^{-1} x)^\top A (A^{-1} x) > 0.
    \]
    Thus, $A^{-1} \succ 0$. 

    \item  Suppose $A \succeq B$. Then for any $x \in \R^d$, 
    \begin{align*}
        x^\top A x \geq x^\top B x. 
    \end{align*}
    Now, for any $y \in \R^d$ let $x = Fy$. We have that 
    \begin{align*}
        x^\top A x \geq x^\top B x,   
    \end{align*}
    and consequently using that $F$ is symmetric, we have
    \begin{align*}
        y^\top F A F y \geq y^\top F B F y,   
    \end{align*}
    The proof if $A \succ B$ is identical, replacing the $\geq$ with $>$. 

    \item Suppose $A \succeq B \succ 0$. Since $B \succ 0$, $B$ is invertible. For any nonzero vector $x \in \R^d$, define $y = B^{-1/2} x$. Then, note that
    \[
    x^\top A^{-1} x \leq x^\top B^{-1} x \quad \text{if and only if} \quad y^\top B^{1/2} A^{-1} B^{1/2} y \leq \|y\|^2.
    \]
    Define the matrix \( M = B^{1/2} A^{-1} B^{1/2} \). We aim to show that \( M \preceq I \), which, by the above argument, implies \( A^{-1} \preceq B^{-1} \).
    
    Since \( A \succeq B \succ 0\), then \( B^{-1/2} A B^{-1/2} \succeq I \), by (iii). This ensures that the smallest eigenvalue of \( B^{-1/2} A B^{-1/2} \succeq I \) is at least 1. Consequently, $(B^{-1/2} A B^{-1/2})^{-1} \preceq I$. Now, 
    \[
    (B^{-1/2} A B^{-1/2})^{-1} \preceq I \quad \text{implies} \quad B^{1/2} A^{-1} B^{1/2} \preceq I.
    \]
    Therefore, by (i), we have that for all \( x \in \R^d \),
    \[
    x^\top A^{-1} x = y^\top M y \leq y^\top y = x^\top B^{-1} x,
    \]
    and hence \( A^{-1} \preceq B^{-1} \), as desired.

    \item If $A = B + C$ and $C \succeq 0$, then $A - B = C \succeq 0$, so $A \succeq B$.

    \item If $\normInline{ B^{-1/2} (A-B) B^{-1/2}} \leq \epsilon$, then for any $x \in \R^d$, 
    \begin{align*}
        -\epsilon \cdot x^\top x \leq x^\top B^{-1/2} (A-B) B^{-1/2} x \leq \epsilon \cdot x^\top x
    \end{align*}
    Consequently, by (i), we have that 
    \begin{align*}
       -\epsilon I \preceq B^{-1/2} (A-B) B^{-1/2} \preceq \epsilon I
    \end{align*}
    Now, applying (iii) we have 
    \begin{align*}
         -\epsilon B \preceq (A-B) \preceq \epsilon B 
    \end{align*}
    Rearranging, we obtain
    \begin{align*}
         (1-\epsilon) B \preceq A \preceq (1+\epsilon) B. 
    \end{align*}
\end{enumerate}
\end{proof}

Next, we provide a proof of Lemma~\ref{lemma:simple-regret-bound-general}. The proof follows the exposition in Section 3 of \cite{zanette2021design} closely, however, we re-prove the result here for the sake of completeness. 

\ridgebound*

\begin{proof} Recall that by \eqref{eq:simple-regret-sclb}, 
\begin{align*}
    R(\hat{\pi}) &= \E_{x \sim p} [\max_{a \in \cA} \phi(x,a)^\top \thetastar - \phi(x,\hat{\pi}(x))^\top \thetastar]. 
\end{align*}
Let $\pi^\star : x \to \argmax_{a \in \cA} \phi(x,a)^\top \thetastar$. Now, for any $x \in \cX$, 
\begin{align*}
    \phi(x,\pi^\star(x))^\top \thetastar - \phi(x,\hat{\pi}(x))^\top \thetastar &= \phi(x,\pi^\star(x))^\top \thetahat - \phi(x,\hat{\pi}(x))^\top \thetastar + \phi(x,\pi^\star(x))^\top \thetastar - \phi(x,{\pi}^\star(x))^\top \thetahat \\
    &\leq \phi(x,\hat{\pi}(x))^\top \thetahat - \phi(x,\hat{\pi}(x))^\top \thetastar + \phi(x,\pi^\star(x))^\top \thetastar - \phi(x,{\pi}^\star(x))^\top \thetahat \\
    &= \phi(x, \hat{\pi}(x))^\top (\thetahat - \thetastar) + \phi(x, \pi^\star(x))^\top (\theta^\star - \thetahat),
\end{align*}
where the inequality follows because by construction of $\hat{\pi}$, we know that 
\begin{align*}
    \phi(x,\hat{\pi}(x))^\top \thetahat \geq \phi(x,{\pi}^\star(x))^\top \thetahat. 
\end{align*}
By Proposition 1 and 2 of \cite{soare2014best} (or Theorem 19.2 of \cite{lattimore2020bandit}, derived originally from \cite{abbasi2011improved}), we have that with probability $1-\delta/2$, for all $(x,a) \in \cX \times \cA$, 
\begin{align*}
    |\phi(x, a)^\top (\thetahat-\thetastar)| \leq \sqrt{\betamind}/2 \cdot \max_{a \in \cA} \normInline{\phi(x,a)}_{{\SigmaLambdaS{\cS}{\lambda}}^{-1}}. 
\end{align*}
Condition on this event in the remainder of the proof. 

Now, combining the preceding two displays and applying a triangle inequality, we can conclude 
\begin{align*}
    R(\hat{\pi}) &\leq \E_{x \sim p} \phi(x, \hat{\pi}(x))^\top (\thetahat - \thetastar) + \phi(x, \pi^\star(x))^\top (\theta^\star - \thetahat) \\
    &\leq 2 \E_{x \sim p} [\max_{a \in \cA} |\phi(x, a)^\top (\thetahat-\thetastar)|] \leq \sqrt{\betamind} \cdot \E_{x \sim p} \max_{a \in \cA} \normInline{\phi(x,a)}_{{\SigmaLambdaSInv{\cS}{\lambda}}}. 
\end{align*}
Now, note that Jensen's inequality ensures 
\begin{align*}
    \paren{\E_{x \sim p} \max_{a \in \cA} \normInline{\phi(x,a)}_{{\SigmaLambdaSInv{\cS}{\lambda}}}}^2 \leq \E_{x \sim p} \max_{a \in \cA} \normInline{\phi(x,a)}_{{\SigmaLambdaSInv{\cS}{\lambda}}}^2, 
\end{align*}
or equivalently, 
\begin{align*}
    \E_{x \sim p} \max_{a \in \cA} \normInline{\phi(x,a)}_{{\SigmaLambdaSInv{\cS}{\lambda}}} \leq \sqrt{\E_{x \sim p} \max_{a \in \cA} \normInline{\phi(x,a)}_{{\SigmaLambdaSInv{\cS}{\lambda}}}^2}. 
\end{align*}
Thus, 
\begin{align}\label{eq:intermediate-with-lambda}
    R(\hat{\pi}) \leq \sqrt{\beta} \cdot \sqrt{\E_{x \sim p} \max_{a \in \cA} \normInline{\phi(x,a)}_{{\SigmaLambdaSInv{\cS}{\lambda}}}^2}.
\end{align}
\end{proof}

The following lemma about \emph{Schur complements} will later be helpful for justifying the SDP formulation of Line~\ref{line:approximate} in Algorithm~\ref{alg:active-lcb}. We believe the following Lemma~\ref{lemma:schur-complement} is well-known, but we include the proof for completeness.

\begin{lemma}[Schur complements]\label{lemma:schur-complement} Let $M \in \R^{(d+1) \times (d+1)}$ be a symmetric matrix of the form
\begin{align*}
    M = \begin{pmatrix}
        a & b^\top \\
        b & C
    \end{pmatrix}
\end{align*}
where $a \in \R, b \in \R^{d}, C \in \R^{d \times d}$ and $C \succ 0$. Then, $M \succeq 0$ if and only if $a - b^\top C^{-1} b \geq 0$. The quantity $a - b^\top C^{-1} b$ is called the \emph{Schur complement} of block $C$. 
\end{lemma}
\begin{proof} First, note that $C \succ 0$ is invertible, and so let 
\begin{align*}
    G = \begin{pmatrix}
        1 & -b^\top C^{-1} \\
        0 & I 
    \end{pmatrix}. 
\end{align*}
Note that $G$ is invertible, and in particular, it is easy to verify that
\begin{align*}
    G^{-1} = \begin{pmatrix}
        1 & b^\top C^{-1} \\
        0 & I 
    \end{pmatrix}. 
\end{align*}
Next, observe that 
\begin{align*}
    G M G^\top = \begin{pmatrix}
        a - b^\top C^{-1} b & 0 \\
        0 & C
    \end{pmatrix}. 
\end{align*}
Now, to prove the claim, first, suppose that $M \succeq 0$. Then, consider any $z \in \R^d$ and note that 
\begin{align*}
    z^\top G M G^\top z =  (G^\top z)^\top M (G^\top z) \geq 0, 
\end{align*}
and hence $G M G^\top \succeq 0$. However, $G M G^\top$ is also block-diagonal, so we can conclude that each block must be positive semi-definite. Thus, $a - b^\top C^{-1} b \geq 0$. 

On the other hand, suppose that $a - b^\top C^{-1} b \geq 0$. Then, because $G M G^\top$ is block-diagonal and $C \succ 0$, we have that $G M G^\top \succeq 0$. Then, consider any $z \in \R^{d+1}$ and note that 
\begin{align*}
    0 \leq ((G^{-1})^\top z)^\top G M G^\top ((G^{-1})^\top z) = z^\top M z. 
\end{align*}
Thus, we conclude that $M \succeq 0$. 
\end{proof}

\subsection{Proof of Theorem~\ref{theorem:main}}\label{apx:proof-main}

In this section, we prove Theorem~\ref{theorem:main}. Throughout this section, we fix $\cB = (\cX, \cA, \phi, p, \nu, \thetastar)$ to be any SCLB instance. 

We divide our proof of Theorem~\ref{theorem:main} into two sub-sections. First, in Section~\ref{apx:implement} we explain how to formulate \eqref{eq:optimal-sampling-distribution} as an SDP in order to implement Line~\ref{line:approximate} of Algorithm~\ref{alg:active-lcb}. This, along with the discussion of Theorem~\ref{thm:kiefer-wolofitz} in Section~\ref{sec:theoretical-analysis} ensures that the algorithm is implementable in polynomial time. 

Second, in Section~\ref{apx:regret} we formally carry out the regret analysis, which was sketched in Section~\ref{sec:theoretical-analysis}. 

\subsubsection{Expressing \eqref{eq:optimal-sampling-distribution} as an SDP to implement Line~\ref{line:approximate}}\label{apx:implement}

In this section, we show how to express the optimization problem \eqref{eq:optimal-sampling-distribution} as an SDP. 

\begin{theorem}[SDP formulation]\label{theorem:SDP-formulation-hat} Let $h \in \R_{\geq 0}^{\cX \times \cA}$ be such that $\sum_{(x,a) \in \cX \times \cA} h(x,a) \leq 1$. Then, the optimization problem 
\begin{equation}\label{eq:target}
\begin{split}
  \textnormal{minimize:}\quad   & \E_{x\sim p} \max_{a \in \cA} \normInline{\phi(x,a)}_{\Sigma_w^{-1}}^2, \\[3pt]
  \textnormal{subject to:}\quad & w \in \Delta^{\cX \times \cA},
  \\[3pt] \quad & w(x,a) \geq h(x,a) \quad\forall\,(x,a)\in\cX\times\cA. 
\end{split}
\end{equation}
is equivalent to the following semi-definite program (SDP):  
\begin{equation}\label{eq:sdp-formulation}
\begin{split}
  \textnormal{minimize:}\quad   & \sum_{x \in \cX} t(x)\,p(x), \\[3pt]
  \textnormal{subject to:}\quad & w \in \Delta^{\cX \times \cA}, \\[3pt]
                          & 
  \begin{pmatrix}
    t(x)            & \phi(x,a)^\top \\[3pt]
    \phi(x,a)       & \Sigma_w
  \end{pmatrix}
  \succeq 0
  \quad\forall\,(x,a)\in\cX\times\cA,
  \\[3pt] \quad & w(x,a) \geq h(x,a) \quad\forall\,(x,a)\in\cX\times\cA. 
\end{split}
\end{equation}
\end{theorem}
\begin{proof} By Lemma~\ref{lemma:schur-complement} the optimization problem \eqref{eq:sdp-formulation} is equivalent to 
\begin{equation*}
\begin{split}
  \text{minimize:}\quad   & \sum_{x \in \cX} t(x)\,p(x), \\[3pt]
  \text{subject to:}\quad & w \in \Delta^{\cX \times \cA}, \\[3pt]
                          & 
  t(x) - \phi(x,a)^\top \Sigma_w^{-1} \phi(x,a) \geq 0
  \quad\forall\,(x,a)\in\cX\times\cA,
  \\[3pt] \quad & w(x,a) \geq h(x,a) \quad\forall\,(x,a)\in\cX\times\cA.  
\end{split}
\end{equation*}

Next, note that $t(x) - \phi(x,a)^\top \Sigma_w^{-1} \phi(x,a) \geq 0$ if and only if $t(x) \geq \normInline{\phi(x,a)}_{\Sigma_w^{-1}}^2.$ Thus, \eqref{eq:sdp-formulation} is further equivalent to 
\begin{equation*}
\begin{split}
  \text{minimize:}\quad   & \sum_{x \in \cX} t(x)\,p(x), \\[3pt]
  \text{subject to:}\quad & w \in \Delta^{\cX \times \cA}, \\[3pt]
                          & 
  t(x) \geq \max_{a \in \cA} \normInline{\phi(x,a)}_{\Sigma_w^{-1}}^2 
  \quad\forall\,x \in\cX,
\\[3pt] \quad & w(x,a) \geq h(x,a) \quad\forall\,(x,a)\in\cX\times\cA. 
\end{split}
\end{equation*}

From the above display we conclude that \eqref{eq:sdp-formulation} is equivalent to 
\begin{equation*}
\begin{split}
  \text{minimize:}\quad   & \sum_{x \in \cX} p(x) \max_{a \in \cA} \normInline{\phi(x,a)}_{\Sigma_w^{-1}}^2, \\[3pt]
  \text{subject to:}\quad & w \in \Delta^{\cX \times \cA},   \\[3pt] \quad & w(x,a) \geq h(x,a) \quad\forall\,(x,a)\in\cX\times\cA. 
\end{split}
\end{equation*}
which is equivalent to \eqref{eq:optimal-sampling-distribution}. 
\end{proof}

\begin{remark} In the special case where $q \equiv 0$, \eqref{eq:target} is equivalent to \eqref{eq:optimal-sampling-distribution}. Meanwhile, it is also easy to see that \eqref{eq:target} generalizes \eqref{eq:smooth-opt}.
\end{remark}
\subsubsection{Regret analysis}\label{apx:regret}

In this section~\ref{apx:regret} we carry out the regret analysis outlined in Section~\ref{sec:theoretical-analysis}. 

\paragraph{Matrix concentration.} Our regret analysis relies on the following standard matrix version of Hoeffding's inequality. 

\begin{theorem}[Matrix Hoeffding (Theorem 1.3 of \cite{tropp2012user}, restated)]\label{thm:matrix-hoeffding} Consider a finite sequence of independent, symmetric, random matrices $X_1, ..., X_T \in \R^{d \times d}$ and $\gamma > 0$ such that 
\begin{align*}
    \E[X_t] = 0, \text{ and } X_t^2 \preceq \gamma^2 I \text{ almost surely. }
\end{align*}
Then, for all $\eta \geq 0$, 
\begin{align*}
    \prob{\lambda_{\max} \paren{\sum_{t \in [T]} X_t} \geq \eta} \leq d \exp\paren{-\frac{\eta^2}{8\gamma^2 T}}. 
\end{align*}
\end{theorem}

In fact, we will only need to use the following (simple) corollary of Theorem~\ref{thm:matrix-hoeffding}. 

\begin{corollary}[Matrix Hoeffding Corollary ]\label{corr:matrix-hoeffding-two-sided} Consider a finite sequence of symmetric matrices $X_1, ..., X_T \in \R^{d \times d}$ and $\gamma > 0$ such that 
\begin{align*}
    \E[X_t] = 0, \text{ and } \normInline{X_t} \leq \gamma \text{ almost surely. }
\end{align*}
Then, for all $\eta \geq 0$, 
\begin{align*}
    \prob{ \norm{\frac{1}{T}{\sum_{t \in [T]} X_t}} \geq \eta} \leq 2d \exp\paren{-\frac{\eta^2 T}{8\gamma^2}}. 
\end{align*}
\end{corollary}
\begin{proof} Note that $\normInline{X_t} \leq \gamma$ implies that $X_t^2 \preceq \gamma^2 I.$ Thus, for all $\eta \geq 0$
\begin{align*}
    \prob{\lambda_{\max} \paren{\frac{1}{T}\sum_{t \in [T]} X_t} \geq \eta} 
 &= \prob{\lambda_{\max} \paren{\sum_{t \in [T]} X_t} \geq T\eta} \\
 &\leq d \exp\paren{-\frac{\eta^2 T^2}{8\gamma^2 T}} = d \exp\paren{-\frac{\eta^2T}{8\gamma^2 }}, 
\end{align*}
where the inequality holds by Theorem~\ref{thm:matrix-hoeffding}. Applying the same analysis to $-X_t$, we have that
\begin{align*}
     \prob{\lambda_{\max} \paren{\frac{1}{T}\sum_{t \in [T]} -X_t} \geq \eta} = \prob{\lambda_{\min} \paren{\frac{1}{T}\sum_{t \in [T]} X_t} \leq -\eta} \leq d \exp\paren{-\frac{\eta^2T}{8\gamma^2 }}. 
\end{align*}
The corollary now follows by a union bound over the events in the two preceding displays. 
\end{proof}

With these concentration inequalities, we are prepared to analyze the concentration of $\SigmOHatM{w}$ to $\SigmaOf{w}$. 

\paragraph{Concentration of $\SigmOHatM{w}$ to $\SigmaOf{w}$.}

To aid the analysis, we define some additional notation. Let $w^\star, q^\star$ be as in \eqref{eq:optimal-sampling-distribution}, Theorem~\ref{thm:kiefer-wolofitz} respectively; $w^*$ be as in \eqref{eq:smooth-opt}, $w, q$ be as in Algorithm~\ref{alg:active-lcb} and
\begin{align*}
    M \defeq \SigmaOf{{q}}^{-1/2}\SigmaOf{w} \SigmaOf{{q}}^{-1/2}. 
\end{align*}

We collect some useful properties in the following lemma. 
\begin{lemma}[Properties to aid concentration analysis]\label{lemma:concentration-properties} Let $w, q$ be as in Algorithm~\ref{alg:active-lcb}. Then, the following hold. 
\begin{enumerate}[label=(\roman*)]
    \item $\SigmaOf{w} \succeq \alpha \SigmaOf{q}$. 
    \item $M \succeq \alpha I$. 
    \item For each $t \in [T]$, let $Z_t$ be a random matrix defined as follows. Draw $(x_t, a_t) \sim w$ (\iid for each $t$) and let
    \begin{align*}
        Z_t = M^{-1/2} \Brac{\SigmaOf{{q}}^{-1/2} \paren{\frac{\lambda}{T} I + \phi(x_t, a_t) \phi(x_t, a_t)^\top} \SigmaOf{{q}}^{-1/2} - M } M^{-1/2}. 
    \end{align*}
    Then, $\E[Z_t] = 0$. 
    \item $\normInline{Z_t} \leq \frac{3d}{\alpha}$.
\end{enumerate}
\end{lemma}
\begin{proof} We prove the claims one-by-one. Recall that $w, q \in \Delta^{\cX \times \cA}$ and consequently, we can write $w = (1-\alpha)\bar{w} + \alpha q$ for some $\bar{w} \in \Delta^{\cX \times \cA}$.
\begin{enumerate}[label=(\roman*)]
    \item By Lemma~\ref{lemma:lowener} (v), we have that 
    \begin{align*}
        \SigmaOf{w} = (1-\alpha) \SigmaOf{\bar{w}} + \alpha \SigmaOf{{q}} \succeq \alpha \SigmaOf{{q}}. 
    \end{align*}
    \item Expanding out $M$, by Lemma~\ref{lemma:lowener} (v), we have
    \begin{align*}
        M &= \SigmaOf{{q}}^{-1/2} ((1-\alpha) \SigmaOf{\bar{w}} + \alpha \SigmaOf{{q}})\SigmaOf{{q}}^{-1/2} \\
        &= (1-\alpha) \SigmaOf{{q}}^{-1/2} \SigmaOf{\bar{w}}  \SigmaOf{{q}}^{-1/2} + \alpha I \succeq \alpha I, 
    \end{align*}
    where the last inequality used Lemma~\ref{lemma:lowener} (v). 
    \item To see that $\E[Z_t] = 0$ note that by linearity,
    \begin{align*}
        \E\left[\frac{\lambda}{T} I + \phi(x_t, a_t) \phi(x_t, a_t)^\top\right] &= \SigmaOf{w}. 
    \end{align*}
    So, by linearity of expectation and the definition of $M$, we have 
    \begin{align*}
        \E\Brac{ {\SigmaOf{{q}}^{-1/2} \paren{\frac{\lambda}{T} I + \phi(x_t, a_t) \phi(x_t, a_t)^\top} \SigmaOf{{q}}^{-1/2} - M } } = 0.  
    \end{align*}
    Applying linearity of expectation once more, 
    \begin{align*}
        \E\Brac{M^{-1/2} \Brac{\SigmaOf{{q}}^{-1/2} \paren{\frac{\lambda}{T} I + \phi(x_t, a_t) \phi(x_t, a_t)^\top} \SigmaOf{{q}}^{-1/2} - M } M^{-1/2}} = 0. 
    \end{align*}
    \item Notice that we can expand $Z_t$ into three terms as follows. 
    \begin{align*}
        Z_t = \frac{\lambda}{T} M^{-1/2} \Sigma_{{q}}^{-1} M^{-1/2} + [M^{-1/2} \SigmaOf{{q}}^{-1/2} \phi(x_t,a_t)][M^{-1/2} \SigmaOf{{q}}^{-1/2} \phi(x_t,a_t)]^\top - I.    
    \end{align*}
    We analyze this term by term and apply triangle inequality. 
    \begin{itemize}
        \item The spectral norm of the first term is bounded using submultiplicativity:
        \begin{align*}
            \norm{\frac{\lambda}{T} M^{-1/2} \Sigma_{{q}}^{-1} M^{-1/2}} \leq \frac{\lambda}{T}\normInline{M^{-1}} \normInline{\Sigma_{{q}}^{-1}} \leq \frac{\lambda}{T} \cdot \frac{1}{\alpha} \cdot \frac{T}{\lambda} = \frac{1}{\alpha}. 
        \end{align*}
        The last inequality in the display above used the fact that $\Sigma_{{q}} \succeq \frac{\lambda}{T} I$ and the fact from part  (ii) that $M \succeq \alpha I$ to deduce (using Lemma~\ref{lemma:lowener} (iv)) that  $\Sigma_{{q}}^{-1} \preceq \frac{T}{\lambda} I$ and $M^{-1} \preceq \frac{1}{\alpha} I$. 

        \item Meanwhile, the second term is a rank-one matrix, and hence, 
        \begin{align*}
        \norm{[M^{-1/2} \SigmaOf{{q}}^{-1/2} \phi(x_t,a_t)][M^{-1/2} \SigmaOf{{q}}^{-1/2} \phi(x_t,a_t)]^\top]} &=
        \normInline{M^{-1/2} \SigmaOf{{q}}^{-1/2} \phi(x_t,a_t)}^2 \\
        &= \phi(x_t,a_t)^\top \SigmaOf{{q}}^{-1/2} M^{-1} \SigmaOf{{q}}^{-1/2} \phi(x_t,a_t) \\
        &= \phi(x_t,a_t)^\top \SigmaOInv{w} \phi(x_t,a_t). 
    \end{align*}
        However, note that by (i) and Lemma~\ref{lemma:lowener} (iv), we know $\SigmaOInv{w} \preceq \frac{1}{\alpha} \SigmaOInv{{q}}$. Thus, by Lemma~\ref{lemma:lowener} (i) we have that 
        \begin{align*}
           \phi(x_t,a_t)^\top \SigmaOInv{w} \phi(x_t,a_t) 
            &\leq \frac{1}{\alpha} \phi(x_t,a_t)^\top \SigmaOInv{{q}} \phi(x_t,a_t) \\
            &= \frac{1}{\alpha} \normInline{\phi(x_t,a_t)}_{\SigmaOInv{{q}}} \leq \frac{2d}{\alpha}, 
        \end{align*}
        where the last inequality holds by construction of ${q}$ in Line~\ref{line:smooth} of Algorithm~\ref{alg:active-lcb}.
        \item The last term is $-I$, whose spectral norm is 1. 
    \end{itemize}
    Consequently, by triangle inequality,
    \begin{align*}
        \normInline{Z_t} &\leq \frac{1}{\alpha} + 1 + \frac{2d}{\alpha} \leq \frac{4d}{\alpha}, 
    \end{align*}
    where the second inequality used that $d \geq 1, \alpha < 1$ implies $1 + 1/\alpha < 2d/\alpha$. 
\end{enumerate}
\end{proof}

\begin{lemma}[Application of matrix Hoeffding]\label{lemma:application-of-hoeffding} Let $w, {q}, \alpha$ be as in Algorithm~\ref{alg:active-lcb}. Define $T_0 = 512\cdot d^2/\alpha^2 \log(4d/\delta)$. Then, for any $T \geq T_0$, with probability $1-\delta/2$, 
\begin{align*}
    \normInline{M^{-1/2} [\SigmaOf{{q}}^{-1/2} \SigmOHatM{w} \SigmaOf{{q}}^{-1/2} - M] M^{-1/2} } \leq \frac{1}{2}. 
\end{align*}
\end{lemma}
\begin{proof} Let $Z_t$ be as in Lemma~\ref{lemma:concentration-properties}. We will apply Corollary~\ref{corr:matrix-hoeffding-two-sided} with $X_t \gets Z_t$, $\gamma \gets 4d/\alpha, \eta \gets 1/2$. Note that 
\begin{align*}
    M^{-1/2} [\SigmaOf{{q}}^{-1/2} \SigmOHatM{w} \SigmaOf{{q}}^{-1/2} - M] M^{-1/2} = \frac{1}{T} \sum_{t \in [T]} Z_t, 
\end{align*}
and hence 
Lemma~\ref{lemma:concentration-properties} along with Corollary~\ref{corr:matrix-hoeffding-two-sided} ensures that 
\begin{align*}
    \prob{ \normInline{M^{-1/2} [\SigmaOf{{q}}^{-1/2} \SigmOHatM{w} \SigmaOf{{q}}^{-1/2} - M] M^{-1/2} } \geq 1/2} \leq 2d \exp\paren{-\frac{T}{32\gamma^2}}. 
\end{align*}
Whenever $T \geq T_0 = 32 \gamma^2 \log(4d/\delta)$, the right-hand side is at most $\delta/2$. 
\end{proof}

Next, we fill out the proof of Lemma~\ref{lemma:no-loss}. 

\noloss*
\begin{proof} The first inequality holds immediately by Line~\ref{line:approximate}. For the second inequality, let $w' = (1-\alpha) w^\star + \alpha q$. Note that $w'$ is feasible for \eqref{eq:smooth-opt}. Since $\Sigma_{w'} \succeq (1-\alpha) \cdot \Sigma_{w^\star}$ implies $\Sigma_{w'}^{-1} \preceq 1/(1-\alpha) \Sigma_{w^\star}^{-1}$,
\begin{align*}
   \E_{x \sim p} \max_{a \in \cA} \normInline{\phi(x,a)}_{\Sigma_{w^*}^{-1}}^2 \leq \E_{x \sim p} \max_{a \in \cA} \normInline{\phi(x,a)}_{\Sigma_{w'}^{-1}}^2 \leq 1/(1-\alpha) \E_{x \sim p} \max_{a \in \cA} \normInline{\phi(x,a)}_{\Sigma_{w^\star}^{-1}}^2
\end{align*}
where the first inequality is by the optimality of $w^*$ and the second inequality is by the Loewner ordering. The lemma now follows by definition of $\cC_\cB$. 
\end{proof}

Finally, we are prepared to prove Theorem~\ref{theorem:main}. 
\mainresult*
\begin{proof} Take $\alpha = 1/2$ and $T_0$ as in Lemma~\ref{lemma:application-of-hoeffding}. Then, by Lemma~\ref{lemma:application-of-hoeffding}, with probability $1-\delta/2$, 
\begin{align*}
    { \normInline{M^{-1/2} [\SigmaOf{{q}}^{-1/2} \SigmOHatM{w} \SigmaOf{{q}}^{-1/2} - M] M^{-1/2} } \leq 1/2}. 
\end{align*}
Condition on this event in the remainder of the proof. Now, by Lemma~\ref{lemma:lowener} (vi), 
\begin{align*}
    1/2 \cdot M \preceq \SigmaOf{{q}}^{-1/2} \SigmOHatM{w} \SigmaOf{{q}}^{-1/2} \preceq 3/2 \cdot M. 
\end{align*}
Multiplying through by $\SigmaOf{{q}}^{1/2}$ on the left and right and applying Lemma~\ref{lemma:lowener} (iii), we have
\begin{align*}
    1/2 \cdot \Sigma_{{q}}^{1/2} M \Sigma_{{q}}^{1/2} \preceq \SigmOHatM{w} \preceq 3/2 \cdot \Sigma_{{q}}^{1/2} M \Sigma_{\hat{q}}^{1/2}. 
\end{align*}
Substituting in the definition $M = \SigmaOf{{q}}^{-1/2}\SigmaOf{w} \SigmaOf{{q}}^{-1/2}$, we can simplify the above display to obtain 
\begin{align*}
    1/2 \cdot \SigmaOf{w} \preceq\SigmOHatM{w} \preceq 3/2 \cdot \SigmaOf{w}. 
\end{align*}

Consequently, by Lemma~\ref{lemma:lowener} (iv), it follows that 
\begin{align*}
    \SigmaOHatMInv{w} \preceq 2 \cdot \SigmaOInv{w}. 
\end{align*}
Now, note that for $\cS$ as defined in Line~\ref{line:cS}, $\SigmaOS{\cS} = T \SigmOHatM{w} $ and hence, $T \SigmaOSInv{\cS} = \SigmaOHatMInv{w}$. Thus, 
\begin{align}\label{eq:upper-bound-here}
    \SigmaOSInv{\cS} \preceq \frac{2}{T} \SigmaOInv{{w}}. 
\end{align}

Consequently, Lemma~\ref{lemma:lowener} (i), \eqref{eq:upper-bound-here} ensures that for any $(x,a) \in \cX \times \cA$, we have 
\begin{align*}
    \normInline{\phi(x,a)}_{\SigmaOSInv{\cS}}^2 = \phi(x,a)^\top \SigmaOSInv{\cS} \phi(x,a) \leq \frac{2}{T} \phi(x,a)^\top \SigmaOInv{{w}} \phi(x,a) = \frac{2}{T} \normInline{\phi(x,a)}_{\SigmaOInv{{w}}}^2. 
\end{align*}
Taking expectation and max and applying the above display point-wise, we have
\begin{align*}
    \E_{x \sim p} \max_{a \in \cA} \normInline{\phi(x,a)}_{\SigmaOSInv{\cS}}^2 \leq \frac{2}{T} \E_{x \sim p} \max_{a \in \cA} \normInline{\phi(x,a)}_{\SigmaOInv{{w}}}^2 \leq \frac{8}{T} \cC_{\cB}, 
\end{align*}
where the last inequality holds due to Lemma~\ref{lemma:no-loss}. 

Now, by applying a union bound with the guarantee of Lemma~\ref{lemma:simple-regret-bound-general}, we see that with probability $1-\delta$, $R(\hat{\pi}) \leq \Tilde{O}(\sqrt{\cC_{\cB}\beta/T})$. 

Finally, to show that $\cC_{\cB} \leq d$, we let $q^\star$ be the G-optimal design as in the Kiefer-Wolfowitz Theorem (Theorem~\ref{thm:kiefer-wolofitz}). Then, by the definition of $\cC_{\cB}$, (recall \eqref{eq:our-bound}) we have
\begin{align*}
    \cC_{\cB} \leq \E_{x \sim p} \max_{a \in \cA} \normInline{\phi(x,a)}_{\SigmaOf{q^\star}}^2 \leq \max_{(x,a) \in \cX \times \cA} \normInline{\phi(x,a)}_{\SigmaOf{q^\star}}^2 \leq d. 
\end{align*}
\end{proof}

\subsection{Miscellaneous omitted proofs}\label{apx:comparison}

\activesdphardinstance*
\begin{proof} Consider $w(x,a)$ defined as follows
\begin{align*}
    w(x,a) = \begin{cases}
        1-(d-1)/(2d), & x = 1, a = 1 \\
        1/(2d), & x \neq 1, a = 1, \\
        0, & \text{otherwise}
    \end{cases}. 
\end{align*}
Note that because $\cX = [d]$, we have 
\begin{align*}
    \sum_{(x,a) \in \cX \times \cA} w(x,a) = \sum_{x \in [d]} w(x, 1) = 1 - \frac{(d-1)}{2d} + (d-1) \cdot \frac{1}{2d} = 1. 
\end{align*}
Then, we can see that
\begin{align*}
    C \defeq \sum_{(x,a) \in \cX \times \cA} w(x,a) \phi(x,a)\phi(x,a)^\top,
\end{align*}
is a diagonal matrix with 
\begin{align*}
    C_{ii} = \begin{cases}
        1-(d-1)/(2d), & i = 1 \\
        1/(2d), & i \neq 1
    \end{cases}. 
\end{align*}
Thus, 
\begin{align*}
    \E_{x\sim p} \max_{a \in \cA} \phi(x,a)^\top C^{-1} \phi(x,a) &\leq \frac{2d}{d+1} + \sum_{i = 2}^d \frac{1}{d^2}  \max\paren{ 2d, \frac{2d}{d+1} } \\
    &\leq \frac{2d}{d+1} + 2(d-1) \max\paren{\frac{1}{d}, \frac{1}{d(d+1)}} \\
    &\leq 2 + \frac{2(d-1)}{d} \leq 4, 
\end{align*}
where the second-to-last inequality holds because $d(d+1) \geq d$.
\end{proof}


\barrier*
\begin{proof} 
Note that 
\begin{align*}
    \E[\Sigma_{\cS}] = \lambda I + \sum_{t \in [T]} \sum_{x \in \cX} p(x) \sum_{a \in \cA} [\pi_t(x)]_a \phi(x,a) \phi(x, a)^\top, 
\end{align*}
and consequently, 
\begin{align*}
    \frac{1}{T}\E[\Sigma_{\cS}] &= \frac{\lambda}{T} I +  \sum_{x \in \cX} p(x) \frac{1}{T}\sum_{t \in [T]} \sum_{a \in \cA} [\pi_t(x)]_a \phi(x, a) \phi(x, a)^\top \\
    &= \frac{\lambda}{T} I + \sum_{(x,a) \in \cX \times \cA} p(x) \paren{\frac{1}{T} \sum_{t \in [T]} [\pi_t(x)]_a} \phi(x,a) \phi(x,a)^\top. 
\end{align*}
So, for each $(x,a) \in \cX \times \cA$, let
\begin{align*}
    [\bar{\pi}(x)]_a &\defeq \frac{1}{T} \sum_{t \in [T]} [\pi_t(x)]_a, \\
    w(x,a) &\defeq p(x) [\bar{\pi}(x)]_a. 
\end{align*}
We then observe that 
\begin{align}\label{eq:reduce-to-w}
    \frac{1}{T}\E[\Sigma_{\cS}] = \Sigma_{w}, \quad T \paren{\E[\Sigma_{\cS}]}^{-1} = \Sigma_w^{-1}. 
\end{align}
Thus, we can observe that 
\begin{align*}
    T \E_{x \sim p}\max_{a \in \cA} \normInline{\phi(x,a)}_{\paren{\E[\Sigma_{\cS}]}^{-1}}^2 
    &= \E_{x \sim p}\max_{a \in \cA} \normInline{\phi(x,a)}_{\SigmaOInv{w}}^2 \\
    &\geq  \E_{x \sim p}\E_{a \sim \bar{\pi}(x)} \normInline{\phi(x,a)}_{\SigmaOInv{w}}^2 \\
    &= \E_{(x,a) \sim w}  \normInline{\phi(x,a)}_{\SigmaOInv{w}}^2 \\
    &= \E_{(x,a) \sim w}  \phi(x,a)^\top \SigmaOInv{w} \phi(x,a) \\
    &= \sum_{(x,a) \in \cX \times \cA} \tr\paren{w(x,a) \phi(x,a)\phi(x,a)^\top \SigmaOInv{w}} \\
    &= \tr\paren{\sum_{(x,a) \in \cX \times \cA} w(x,a) \phi(x,a)\phi(x,a)^\top \SigmaOInv{w}}. 
\end{align*}
where the inequality follows from the simple fact that the expectation is always at most the maximum; the second-to-last equality uses the cyclic property of the trace; and the last equality uses linearity of the trace. Substituting in the formula for $\Sigma_w$ and dividing both sides of the above display by $T$, 
\begin{align*}
    &\E_{x \sim p}\max_{a \in \cA} \normInline{\phi(x,a)}_{\paren{\E[\Sigma_{\cS}]}^{-1}}^2 \geq \frac{1}{T} \tr\paren{\sum_{(x,a) \in \cX \times \cA} w(x,a) \phi(x,a)\phi(x,a)^\top \SigmaOInv{w}} \\
    &=\frac{1}{T}\tr\paren{\sum_{(x,a) \in \cX \times \cA} w(x,a) \phi(x,a)\phi(x,a)^\top \Brac{\frac{\lambda}{T} I + \sum_{(x,a) \in \cX \times \cA} w(x,a) \phi(x,a)\phi(x,a)^\top  }^{-1}} \\
    &\underset{\lambda \to 0}{\to} \frac{d}{T}. 
\end{align*}
Finally, using that $\E[\Sigma_\cS^{-1}] \succeq \paren{\E[\Sigma_{\cS}]}^{-1}$ and Lemma~\ref{lemma:lowener} (i), we have 
\begin{align*}
    \E_{x \sim p}\max_{a \in \cA} \normInline{\phi(x,a)}_{\paren{\E[\Sigma_{\cS}]}^{-1}}^2 \leq \E_{x \sim p}\max_{a \in \cA} \normInline{\phi(x,a)}_{\E[\Sigma_{\cS}^{-1}]}^2. 
\end{align*}
Combining the two above displays, we conclude that 
\begin{align*}
    \lim_{\lambda \to 0} \E_{x \sim p}\max_{a \in \cA} \normInline{\phi(x,a)}_{\E[\Sigma_{\cS}^{-1}]}^2 \geq \frac{d}{T}, 
\end{align*}
as desired. 
\end{proof}


%% file: ref.bib
@inproceedings{chaudhuri2015convergence,
  title={Convergence rates of active learning for maximum likelihood estimation},
  author={Chaudhuri, Kamalika and Kakade, Sham M and Netrapalli, Praneeth and Sanghavi, Sujay},
  booktitle={\cNIPS{2015}},
  year={2015}
}

@inproceedings{tropp2012user,
  title={User-friendly tail bounds for sums of random matrices},
  author={Tropp, Joel A},
  booktitle={Foundations of computational mathematics},
  year={2012},
}

@inproceedings{zanette2021design,
  title={Design of experiments for stochastic contextual linear bandits},
  author={Zanette, Andrea and Dong, Kefan and Lee, Jonathan N and Brunskill, Emma},
  booktitle={Advances in Neural Information Processing Systems},
  year={2021}
}

@inproceedings{abbasi2011improved,
  title={Improved algorithms for linear stochastic bandits},
  author={Abbasi-Yadkori, Yasin and P{\'a}l, D{\'a}vid and Szepesv{\'a}ri, Csaba},
  booktitle={\cNIPS{2011}},
  year={2011}
}

@inproceedings{chu2011contextual,
  title={Contextual bandits with linear payoff functions},
  author={Chu, Wei and Li, Lihong and Reyzin, Lev and Schapire, Robert},
  booktitle={\cAISTATS{2011}},
  year={2011},
}

@inproceedings{pukelsheim2006optimal,
  title={Optimal design of experiments},
  author={Pukelsheim, Friedrich},
  year={2006},
  booktitle={SIAM}
}

@inproceedings{deshmukh2018simple,
  title={Simple regret minimization for contextual bandits},
  author={Deshmukh, Aniket Anand and Sharma, Srinagesh and Cutler, James W and Moldwin, Mark and Scott, Clayton},
  booktitle={arXiv preprint arXiv:1810.07371},
  year={2018}
}

@inproceedings{krishnamurthy2023proportional,
  title={Proportional response: Contextual bandits for simple and cumulative regret minimization},
  author={Krishnamurthy, Sanath Kumar and Zhan, Ruohan and Athey, Susan and Brunskill, Emma},
  booktitle={Advances in Neural Information Processing Systems},
  year={2023}
}

@inproceedings{varatharajah2022contextual,
  title={A contextual-bandit-based approach for informed decision-making in clinical trials},
  author={Varatharajah, Yogatheesan and Berry, Brent},
  booktitle={Life},
  year={2022},
}

@inproceedings{wang2025optimal,
  title={Optimal dose selection in phase I/II dose finding trial with contextual bandits: a case study and practical recommendations},
  author={Wang, Jixian and Tiwari, Ram},
  booktitle={Journal of Biopharmaceutical Statistics},
  year={2025},
}

@inproceedings{bouneffouf2012contextual,
  title={A contextual-bandit algorithm for mobile context-aware recommender system},
  author={Bouneffouf, Djallel and Bouzeghoub, Amel and Gan{\c{c}}arski, Alda Lopes},
  booktitle={Neural Information Processing: 19th International Conference, ICONIP 2012, Doha, Qatar, November 12-15, 2012, Proceedings, Part III 19},
  year={2012},
}

@inproceedings{tang2014ensemble,
  title={Ensemble contextual bandits for personalized recommendation},
  author={Tang, Liang and Jiang, Yexi and Li, Lei and Li, Tao},
  booktitle={Proceedings of the 8th ACM Conference on Recommender Systems},
  year={2014}
}

@inproceedings{liu2024sample,
  title={Sample-efficient alignment for llms},
  author={Liu, Zichen and Chen, Changyu and Du, Chao and Lee, Wee Sun and Lin, Min},
  booktitle={arXiv preprint arXiv:2411.01493},
  year={2024}
}

@inproceedings{chen2024online,
  title={Online personalizing white-box llms generation with neural bandits},
  author={Chen, Zekai and Chen, Po-Yu and Buet-Golfouse, Francois},
  booktitle={Proceedings of the 5th ACM International Conference on AI in Finance},
  year={2024}
}

@inproceedings{soare2014best,
  title={Best-arm identification in linear bandits},
  author={Soare, Marta and Lazaric, Alessandro and Munos, R{\'e}mi},
  booktitle={Advances in neural information processing systems},
  volume={27},
  year={2014}
}

@inproceedings{ruan2021linear,
  title={Linear bandits with limited adaptivity and learning distributional optimal design},
  author={Ruan, Yufei and Yang, Jiaqi and Zhou, Yuan},
  booktitle={Proceedings of the 53rd Annual ACM SIGACT Symposium on Theory of Computing},
  year={2021}
}

@inproceedings{settles2009active,
  title={Active learning literature survey},
  author={Settles, Burr},
  year={2009},
  booktitle={University of Wisconsin-Madison Department of Computer Sciences}
}

@inproceedings{char2019offline,
  title={Offline contextual bayesian optimization},
  author={Char, Ian and Chung, Youngseog and Neiswanger, Willie and Kandasamy, Kirthevasan and Nelson, Andrew O and Boyer, Mark and Kolemen, Egemen and Schneider, Jeff},
  booktitle={Advances in Neural Information Processing Systems},
  year={2019}
}

@inproceedings{das2024active,
  title={Active preference optimization for sample efficient RLHF},
  author={Das, Nirjhar and Chakraborty, Souradip and Pacchiano, Aldo and Chowdhury, Sayak Ray},
  booktitle={arXiv preprint arXiv:2402.10500},
  year={2024}
}

@inproceedings{li2022near,
  title={Near-optimal policy identification in active reinforcement learning},
  author={Li, Xiang and Mehta, Viraj and Kirschner, Johannes and Char, Ian and Neiswanger, Willie and Schneider, Jeff and Krause, Andreas and Bogunovic, Ilija},
  booktitle={arXiv preprint arXiv:2212.09510},
  year={2022}
}

@inproceedings{kong2020sublinear,
  title={Sublinear optimal policy value estimation in contextual bandits},
  author={Kong, Weihao and Brunskill, Emma and Valiant, Gregory},
  booktitle={International conference on artificial intelligence and statistics},
  year={2020},
}

@inproceedings{lattimore2020bandit,
  title={Bandit algorithms},
  author={Lattimore, Tor and Szepesv{\'a}ri, Csaba},
  year={2020},
  booktitle={Cambridge University Press}
}

@inproceedings{jiang2020faster,
  title={A faster interior point method for semidefinite programming},
  author={Jiang, Haotian and Kathuria, Tarun and Lee, Yin Tat and Padmanabhan, Swati and Song, Zhao},
  booktitle={2020 IEEE 61st annual symposium on foundations of computer science (FOCS)},
  year={2020},
}

@inproceedings{goulart2024clarabel,
  title={Clarabel: An interior-point solver for conic programs with quadratic objectives},
  author={Goulart, Paul J and Chen, Yuwen},
  booktitle={arXiv preprint arXiv:2405.12762},
  year={2024}
}

@inproceedings{goldberg2001eigentaste,
  title={Eigentaste: A constant time collaborative filtering algorithm},
  author={Goldberg, Ken and Roeder, Theresa and Gupta, Dhruv and Perkins, Chris},
  booktitle={Information Retrieval},
  year={2001},
}

@inproceedings{international2009estimation,
  title={Estimation of the warfarin dose with clinical and pharmacogenetic data},
  author={International Warfarin Pharmacogenetics Consortium},
  booktitle={New England Journal of Medicine},
  year={2009},
}

@inproceedings{truda2021evaluating,
  title={Evaluating warfarin dosing models on multiple datasets with a novel software framework and evolutionary optimisation},
  author={Truda, Gianluca and Marais, Patrick},
  booktitle={Journal of Biomedical Informatics},
  year={2021},
}

@inproceedings{kirschner2020distributionally,
  title={Distributionally robust Bayesian optimization},
  author={Kirschner, Johannes and Bogunovic, Ilija and Jegelka, Stefanie and Krause, Andreas},
  booktitle={International Conference on Artificial Intelligence and Statistics},
  year={2020},
}

@inproceedings{gentile2024fast,
  title={Fast rates in pool-based batch active learning},
  author={Gentile, Claudio and Wang, Zhilei and Zhang, Tong},
  booktitle={Journal of Machine Learning Research},
  year={2024}
}

@inproceedings{civril2009selecting,
  title={On selecting a maximum volume sub-matrix of a matrix and related problems},
  author={Civril, Ali and Magdon-Ismail, Malik},
  booktitle={Theoretical Computer Science},
  year={2009},
}

@inproceedings{welch1982algorithmic,
  title={Algorithmic complexity: three NP-hard problems in computational statistics},
  author={Welch, William J},
  booktitle={Journal of Statistical Computation and Simulation},
  year={1982}
}

@inproceedings{li2022instance,
  title={Instance-optimal pac algorithms for contextual bandits},
  author={Li, Zhaoqi and Ratliff, Lillian and Jamieson, Kevin G and Jain, Lalit and others},
  booktitle={Advances in Neural Information Processing Systems},
  year={2022}
}

@inproceedings{fujisawa2012high,
  title={High-performance general solver for extremely large-scale semidefinite programming problems},
  author={Fujisawa, Katsuki and Sato, Hitoshi and Matsuoka, Satoshi and Endo, Toshio and Yamashita, Makoto and Nakata, Maho},
  booktitle={SC'12: Proceedings of the International Conference on High Performance Computing, Networking, Storage and Analysis},
  year={2012},
}

@inproceedings{lin2025pdcs,
  title={PDCS: A Primal-Dual Large-Scale Conic Programming Solver with GPU Enhancements},
  author={Lin, Zhenwei and Xiong, Zikai and Ge, Dongdong and Ye, Yinyu},
  booktitle={arXiv preprint arXiv:2505.00311},
  year={2025}
}

@inproceedings{frostig2015competing,
  title={Competing with the empirical risk minimizer in a single pass},
  author={Frostig, Roy and Ge, Rong and Kakade, Sham M and Sidford, Aaron},
  booktitle={Conference on learning theory},
  year={2015},
}

@inproceedings{fowler2023clinical,
  title={Clinical trial active learning},
  author={Fowler, Zoe and Kokilepersaud, Kiran Premdat and Prabhushankar, Mohit and AlRegib, Ghassan},
  booktitle={Proceedings of the 14th ACM international conference on bioinformatics, computational biology, and health informatics},
  year={2023}
}

@inproceedings{deng2011active,
  title={Active learning for personalizing treatment},
  author={Deng, Kun and Pineau, Joelle and Murphy, Susan},
  booktitle={2011 IEEE Symposium on Adaptive Dynamic Programming and Reinforcement Learning (ADPRL)},
  year={2011},
}

@inproceedings{labrecque2021addressing,
  title={Addressing online behavioral advertising and privacy implications: A comparison of passive versus active learning approaches},
  author={Labrecque, Lauren I and Markos, Ereni and Darmody, Aron},
  booktitle={Journal of Marketing Education},
  year={2021},
}

@inproceedings{wagenmaker2022instance,
  title={Instance-dependent near-optimal policy identification in linear mdps via online experiment design},
  author={Wagenmaker, Andrew and Jamieson, Kevin G},
  booktitle={Advances in Neural Information Processing Systems},
  year={2022}
}

@inproceedings{wagenmaker2022beyond,
  title={Beyond no regret: Instance-dependent pac reinforcement learning},
  author={Wagenmaker, Andrew J and Simchowitz, Max and Jamieson, Kevin},
  booktitle={Conference on Learning Theory},
  year={2022},
}

@inproceedings{gheshlaghi2013minimax,
  title={Minimax PAC bounds on the sample complexity of reinforcement learning with a generative model},
  author={Gheshlaghi Azar, Mohammad and Munos, R{\'e}mi and Kappen, Hilbert J},
  booktitle={Machine learning},
  year={2013},
}

@inproceedings{wang2021sample,
  title={Sample-efficient reinforcement learning for linearly-parameterized mdps with a generative model},
  author={Wang, Bingyan and Yan, Yuling and Fan, Jianqing},
  booktitle={Advances in neural information processing systems},
  year={2021}
}

@inproceedings{rodemann2024reciprocal,
  title={Reciprocal learning},
  author={Rodemann, Julian and Jansen, Christoph and Schollmeyer, Georg},
  booktitle={Advances in Neural Information Processing Systems},
  year={2024}
}

@inproceedings{de2023active,
  title={Active learning in contextual bandits: handling the uncertainty about the user's preferences in interactive recommendation systems},
  author={de Campos Silva, Nicollas and others},
  year={2023},
  booktitle={Universidade Federal de Minas Gerais}
}

@inproceedings{canonne2020short,
  title={A short note on learning discrete distributions},
  author={Canonne, Cl{\'e}ment L},
  booktitle={arXiv preprint arXiv:2002.11457},
  year={2020}
}

@inproceedings{minsker2016active,
  title={Active clinical trials for personalized medicine},
  author={Minsker, Stanislav and Zhao, Ying-Qi and Cheng, Guang},
  booktitle={Journal of the American Statistical Association},
  year={2016},
}
